\documentclass[twocolumn]{article}
\usepackage[letterpaper,margin=0.75in,columnsep=0.25in]{geometry}

\usepackage[round]{natbib}

\usepackage{tikz}
\usetikzlibrary{positioning, arrows.meta, backgrounds, fit}
\usetikzlibrary{calc}
\usepackage{amsmath, amssymb, amsthm, mathtools}
\usepackage{bbm, bm}
\usepackage{marvosym, wasysym}
\usepackage{accents}

\usepackage{graphicx}
\usepackage{subcaption} 
\usepackage{placeins}
\usepackage{multicol, multirow}
\usepackage{array}
\usepackage{physics}
\usepackage{paralist}
\usepackage{comment}
\usepackage[normalem]{ulem}
\usepackage{fancyhdr}
\usepackage[dvipsnames]{xcolor}

\usepackage{algorithm}
\usepackage{algpseudocode}
\usepackage{listings}

\usepackage{hyperref}
\hypersetup{
    colorlinks=true,
    linkcolor=NavyBlue, 
    citecolor=NavyBlue,
    urlcolor=NavyBlue
}

\theoremstyle{plain} %
\newtheorem{theorem}{Theorem}[section]  
\newtheorem{lemma}[theorem]{Lemma}
\newtheorem{proposition}[theorem]{Proposition}

\theoremstyle{definition} %

\newtheorem{assumption}[theorem]{Assumption}

\theoremstyle{remark}

\usepackage[utf8]{inputenc} %
\usepackage[T1]{fontenc}    %
\usepackage{hyperref}       %
\usepackage{url}            %
\usepackage{booktabs}       %
\usepackage{arydshln}       %
\usepackage{amsfonts}       %
\usepackage{nicefrac}       %
\usepackage{microtype}      %
\usepackage{xcolor}         %

\begin{document}

\title{Online Conformal Prediction for Non-Exchangeable Panel Data}
\author{  Daohong Tu\\
  Department of Management Science and Engineering\\
  Stanford University\\
  \texttt{daohongt@stanford.edu}
  \and
  Kay Giesecke\\
  Department of Management Science and Engineering\\
  Stanford University\\
  \texttt{giesecke@stanford.edu}
}
\maketitle

\begin{abstract}
We study online conformal prediction in a partially observed panel: a new cross-section of peer outcomes is observed before each target outcome, target feedback may be intermittent or absent, and neither units nor rounds need be exchangeable. We propose Weighted Temporal Quantile Adjustment (W-TQA), which combines similarity weights learned from unit histories with an adaptive target-specific miscoverage level. We prove that neither target--peer mismatch nor coverage on unrevealed rounds is identifiable, so assumptions on cross-unit similarity and on the feedback mechanism cannot be avoided. We bound the past-conditional miscoverage in terms of this mismatch, quantify the cost of learning the weights under a profile-similarity assumption, and show that the implemented procedure, which falls back on the largest peer score, attains long-run average coverage under missing-completely-at-random feedback and a feasibility condition on the calibration panel. On synthetic panels and four real panels, including a genuinely asynchronous U.S.-to-Tokyo equity panel, W-TQA attains the highest coverage on the worst-covered target units on every panel; similarity weighting contributes most when target feedback is scarce, and temporal adaptation more as feedback accumulates.
\end{abstract}

\section{Introduction}\label{sec:introduction}
Panel data combine repeated observations over time with heterogeneous units.
Temporal dependence, distribution drift, and cross-unit heterogeneity make the
exchangeability assumptions underlying classical conformal prediction
\citep{vovk2005algorithmic,papadopoulos2002inductive,lei2018distribution,angelopoulos2023conformal}
generally implausible.

In many panels, when a forecast is required for one unit, outcomes of related
units are already observed: returns of completed U.S.\ trading sessions, for
example, are available when the Tokyo market opens, and
Appendix~\ref{app:motivating_examples} describes further examples in finance,
traffic monitoring, and retail. We study online conformal prediction in this
partially observed panel setting.

We propose \emph{Weighted Temporal Quantile Adjustment} (W-TQA), which
maintains two online states. The \emph{spatial branch} maintains similarity
weights that place larger calibration weight on peer units whose feature
history resembles the target's; the \emph{temporal branch} maintains an
adaptive miscoverage level that starts at the target level \(\alpha\) and is
updated only when target feedback is revealed. The two states address
different information regimes: the weights are available from the outset and
do not require online target outcomes, whereas the adaptive level learns
target-specific calibration errors as feedback arrives. At each round, W-TQA
computes a weighted conformal threshold from the round's peer cross-section
using both states; within this calibration layer, past information enters
only through the two states, so memory does not grow over time and each round
requires a single weighted quantile. Adaptive conformal inference
\citep{gibbs2021adaptive} returns the whole real line once its level falls low
enough, and its analysis relies on this: a trivial set cannot miss, which keeps
the level bounded below. Because that set is uninformative, W-TQA falls back on
the largest peer score whenever its own level is that low.

The analysis faces three obstacles. On the spatial side, the peers need not be
exchangeable with the target, so weighted calibration pays a mismatch price
that the observed data cannot reveal. On the temporal side, the fallback set
can miss, so the level has no automatic lower bound. With intermittent
feedback, misses on unrevealed rounds leave no trace, so calibration on revealed
rounds need not transfer to them.

In summary, our contributions are threefold:

\noindent\textbf{Framework.} We formulate online conformal prediction for
partially observed panels as a problem with two information streams of unequal
availability: a peer cross-section refreshed every round before the target
outcome, and target feedback arriving intermittently, if at all. Single-series
online formulations, intermittent feedback included, rely on the target's own
stream alone; the closest longitudinal-panel formulations assume exchangeable
calibration units and uninterrupted target feedback. The separation lets
calibration proceed on non-exchangeable peers without online target labels:
W-TQA calibrates on the first stream at every round and adapts its level on the
second.

\noindent\textbf{Theory.} We bound W-TQA's past-conditional miscoverage in
terms of target--peer mismatch and, under a profile-based similarity
assumption, quantify the cost of learning weights relative to a latent-profile
oracle. For the temporal branch, we establish long-run coverage guarantees for
the implemented procedure, fallback included, under a feasibility condition on
the calibration panel and missing-completely-at-random (MCAR) feedback. We also prove that target--peer
mismatch and coverage on unrevealed rounds are not identifiable from observed
data alone, clarifying why assumptions on cross-unit similarity and the reveal
mechanism are needed.

\noindent\textbf{Experiments.} On synthetic panels and four real panels from
finance, retail, and electricity, the two branches operate in complementary
feedback regimes: similarity weighting contributes most when target feedback
is scarce and temporal adaptation increasingly as feedback accumulates, so
W-TQA matches or exceeds the better of its two ablations at nearly every
feedback probability. It also attains the highest tail coverage, the average
coverage on the worst-covered target units, on every panel, with average
coverage at or above \(1-\alpha\).
On the U.S.-to-Tokyo panel, a real-world instance of this information
structure, the same pattern holds.

\section{Related Work}

\paragraph{Online and sequential conformal prediction.}
For a single non-exchangeable series, online adaptive methods update the
calibration level or score threshold over time
\citep{gibbs2021adaptive,zaffran2022adaptive,gibbs2024online,bhatnagar2023improved,angelopoulos2023conformalpid,angelopoulos2024decaying,areces2025online}.
Recent work allows incomplete target feedback under different reveal
mechanisms: labels may arrive intermittently \citep{wang2025mirror}, through
expert interventions in interactive imitation learning
\citep{zhao2025conformalized}, or only when they fall in the issued set, a
semi-bandit formulation \citep{ge2025stochastic,yang2026adversarial}. Other
sequential methods, such as EnbPI and
SPCI, calibrate on the residuals of a fitted predictor
\citep{stankeviciute2021conformal,xu2021conformal,xu2023sequential}, and
OLCP~\citep{lai2026online} localizes a rolling window of the series' own past
scores in covariate space before applying an online
update.\footnote{OLCP is concurrent and independent work.}

\paragraph{Weighted and localized conformal prediction.}
A separate line of work weights the calibration scores before forming the
conformal threshold, either to address distributional mismatch
\citep{tibshirani2019conformal,podkopaev2021distribution,barber2023conformal,farinhas2024nonexchangeable}
or to localize calibration toward similar observations
\citep{guan2023localized,hore2023conformal,auer2023conformal,lee2025kernel};
\citet{barber2026unifying} unify these weighted methods.

The single-series methods above calibrate on the target series' own past
observations and therefore depend on the arrival of its feedback; W-TQA
instead calibrates on the round's peer cross-section, which is refreshed
whether or not target feedback arrives, and weights the peers by similarity
learned from unit histories.

\paragraph{Panel and longitudinal conformal prediction.}
The closest line of work targets panel and longitudinal data under assumptions
our setting does not impose. \citet{lin2022conformal}
introduce Temporal Quantile Adjustment (TQA), with error-feedback (TQA-E) and
budgeted (TQA-B) variants, and \citet{batra2023conformal} propose LPCI, which extends
SPCI with group-specific fixed effects. Both assume cross-sectional exchangeability among calibration units, and both
update the target-side state at every round, which requires uninterrupted target
feedback. \citet{dunn2023distribution} study hierarchical prediction sets but assume
i.i.d.\ observations within each group, and thus abstract away serial dependence
entirely. Taken together, prior work addresses localization, intermittent feedback, and
longitudinal structure largely in separate settings; our problem combines a
per-round peer cross-section, heterogeneous non-exchangeable units, and
intermittent target feedback.

\section{Problem Formulation}\label{sec:problem_formulation}

We study online conformal prediction in a panel data setting. Consider
feature-response pairs \((X_{i,t},Y_{i,t})\in\mathbb R^d\times\mathbb R\),
where \(i=1,\ldots,N+1\) indexes units and \(t=1,2,\ldots\) indexes time. The
first \(N\) units serve as calibration units, and unit \(N+1\) is the target
unit. Given \(X_{N+1,t}\), we seek to construct a prediction set
\(C_t(X_{N+1,t})\subseteq\mathbb R\) for the target response \(Y_{N+1,t}\),
using the calibration-unit observations available at \(t\) together with
any past target feedback. The key challenge is that
the data may not be exchangeable, which violates the central assumption of standard
conformal prediction.

\paragraph{Observations arriving at round \(t\).}
At each round \(t=1,2,\ldots\), a single batch of new observations
\(\mathcal O_t\) arrives, given by
\[
\begin{aligned}
\mathcal O_t
:=
\bigl(
&X_{1,t},Y_{1,t},\ldots,X_{N,t},Y_{N,t},\\
&X_{N+1,t},\;
R_{t-1},\;
R_{t-1}Y_{N+1,t-1}
\bigr),
\end{aligned}
\]
with \(R_0=Y_{N+1,0}=0\). The calibration units are promptly observed, so
the full round-\(t\) calibration slice \(\{(X_{i,t},Y_{i,t})\}_{i\le N}\) is in
\(\mathcal O_t\). The target covariate \(X_{N+1,t}\) is also in \(\mathcal O_t\), but the
target response \(Y_{N+1,t}\) is not: in applications, the response may be
revealed with a delay or never reported. We model this through the binary
process \((R_t)_{t\ge1}\), with \(R_t=1\) if \(Y_{N+1,t}\) is revealed by
round \(t+1\) and \(R_t=0\) otherwise. If \(R_t=1\), then \(Y_{N+1,t}\) enters the filtration through
\(\mathcal O_{t+1}\), which contains \(R_t\) and \(R_tY_{N+1,t}\); if
\(R_t=0\), it never does.
We take the delay to be one round to keep the notation simple; essentially the
same algorithm handles longer bounded delays, with similar guarantees
(Appendix~\ref{app:delayed-feedback}).
We include \(R_{t-1}\)
in \(\mathcal O_t\), in addition to \(R_{t-1}Y_{N+1,t-1}\), so that
missing feedback (\(R_{t-1}=0\)) is distinguished from an observed response
equal to zero.\looseness=-1

\paragraph{Filtration.}
The information available through round \(t\) is
\(\mathcal F_t:=\sigma(\mathcal O_1,\ldots,\mathcal O_t)\), with
\(\mathcal F_0:=\{\emptyset,\Omega\}\), and the prediction set
\(C_t(X_{N+1,t})\subseteq\mathbb R\) must be \(\mathcal F_t\)-measurable.
We also use the refinement
\(\mathcal F_t^{+}:=\mathcal F_t\vee\sigma(R_t,R_tY_{N+1,t})\), which adds the
round-\(t\) target feedback to \(\mathcal F_t\).
Table~\ref{tab:info_structure} summarizes the round-\(t\) information
structure.

\begin{center}
  \footnotesize
  \captionsetup{skip=5pt}
  \setlength{\tabcolsep}{4pt}
  \captionof{table}{Information available when the round-$t$ prediction set is formed.}\label{tab:info_structure}
  \begin{tabular}{@{}lcc@{}}
    \toprule
    & Rounds $s<t$ & Round $t$\\
    \midrule
    Calibration $i\le N$ & $(X_{i,s},Y_{i,s})$ & $(X_{i,t},Y_{i,t})$\\
    Target $N+1$ & $(X_{N+1,s},R_s,R_sY_{N+1,s})$ & $X_{N+1,t}$ only\\
    \bottomrule
  \end{tabular}
\end{center}

\paragraph{Coverage goal.}
For a target level \(\alpha\in(0,1)\), we aim for the
deployed prediction sets to cover the target response with
probability close to \(1-\alpha\). Section~\ref{sec:wtqa_guarantees} makes this precise in two
forms: a per-round conditional bound on miscoverage given the past
(Theorems~\ref{thm:spatial_stepwise} and~\ref{thm:spatial_oracle}), and a
bound on the time-averaged miscoverage probability
(Theorem~\ref{thm:temporal_mcar}).

\section{The W-TQA Algorithm}\label{sec:proposed_method}

We now describe W-TQA for the online panel setting of
Section~\ref{sec:problem_formulation}. W-TQA maintains two online states: the
spatial weights \(W^{(t)}\), formed at the end of round \(t-1\), and the
adaptive miscoverage level \(\alpha_t\), refreshed at the start of round \(t\)
from the lagged feedback in \(\mathcal O_t\). Each round then computes a
weighted conformal threshold from the current calibration panel using
\(W^{(t)}\) and \(\alpha_t\).
\paragraph{Spatial branch: history-based similarity weights.}\label{subsec:online_protocol}

Write $\mathcal W_{N+1}:=\{v\in[0,1]^{N+1}:\sum_{k=1}^{N+1}v_k=1\}$ for the probability
simplex. At round $t$, the spatial weight vector 
\[
W^{(t)}:=\bigl(w_{N+1,1}^{(t)},\dots,w_{N+1,N+1}^{(t)}\bigr)\in\mathcal W_{N+1}
\]
assigns a similarity weight to each calibration unit $k\in\{1,\dots,N\}$ relative to
the target, while the final coordinate $w_{N+1,N+1}^{(t)}$ is reserved for the
target and is used in the weighted calibration step below. To compute these weights, we maintain a running
feature mean $\widehat\mu_{k,t}:=\frac{1}{t}\sum_{s=1}^{t}X_{k,s}\in\mathbb R^d$ for each unit
$k$, which we call the \emph{profile} of unit $k$.
To compare profiles on a common scale, let $s_{t,j}$, for $j\in\{1,\dots,d\}$, be the
cross-sectional standard deviation of the $j$-th profile coordinates $\{\widehat\mu_{i,t,j}\}_{i\le N+1}$ (set
$s_{t,j}=1$ when this standard deviation is zero), let $\Lambda_t:=\operatorname{diag}(s_{t,1},\dots,s_{t,d})$, and define the
round-$t$ profile metric $\rho_t(u,v):=\|\Lambda_t^{-1}(u-v)\|_2/\sqrt d$. For a bandwidth
constant $h>0$, set the unnormalized round-$t$ weight to
\[
\kappa_{k,t-1}
=
\exp\!\Bigl(-\tfrac{\rho_{t-1}(\widehat\mu_{k,t-1},\widehat\mu_{N+1,t-1})^2}{2h^2}\Bigr),
\quad k\le N,
\]
with $\kappa_{N+1,t-1}=1$.
The factors $s_{t,j}$ and $\sqrt d$ normalize $\rho_t$ for the measurement units of the features and for
the dimension $d$, so $h$ is on the same scale on every panel.
This Gaussian kernel places higher weight on calibration units whose
long-run feature means are closer to the target's in the metric $\rho_{t-1}$, giving a soft analogue
of nearest-neighbor calibration. The normalized weights are
$w^{(t)}_{N+1,k}=\kappa_{k,t-1}/\sum_{j\le N+1}\kappa_{j,t-1}$. Since \(W^{(t)}\) uses only the feature histories
\(\{X_{k,s}\}_{s\le t-1}\), it is \(\mathcal F_{t-1}\)-measurable and available
whether or not any target response has been revealed.

\paragraph{Temporal branch: adaptive miscoverage level.}\label{subsec:temporal_branch}

The level \(\alpha_t\) is initialized at the target level, \(\alpha_0=\alpha\), and updated using
lagged target feedback. At the start of round \(t\), if \(R_{t-1}=1\), the
algorithm uses the revealed \(Y_{N+1,t-1}\) in \(\mathcal O_t\) to compute the
lagged miscoverage indicator
\(\ell_{t-1}=\mathbf 1\{Y_{N+1,t-1}\notin C_{t-1}(X_{N+1,t-1})\}\) and takes a
gradient step with step size \(\gamma>0\), as in adaptive conformal
inference \citep{gibbs2021adaptive}:
\begin{equation}\label{eq:base-level-update}
\alpha_t=\alpha_{t-1}+\gamma(\alpha-\ell_{t-1}).
\end{equation}
A miss (\(\ell_{t-1}=1\)) thus lowers the level and enlarges the prediction
set, and a cover raises the level and shrinks it. If \(R_{t-1}=0\), the level
is carried forward unchanged: \(\alpha_t=\alpha_{t-1}\). By construction,
\(\alpha_t\) is \(\mathcal F_{t-1}^{+}\)-measurable.

\paragraph{Weighted conformal calibration.}\label{subsec:weighted_conformal_calibration}

Once \(W^{(t)}\) and \(\alpha_t\) are fixed for round \(t\), the conformal step
uses only the current
calibration slice. It is based on a nonconformity score
\(\hat s:\mathbb R^d\times\mathbb R\to\mathbb R\), where larger values
indicate worse agreement between a covariate--response pair and the fitted
prediction rule. For regression, a common choice is the absolute residual score
\(\hat s(x,y)=|y-\hat f(x)|\), where
\(\hat f:\mathbb R^d\to\mathbb R\) is a predictor trained on historical data
and held fixed.
Our coverage guarantees hold for any nonconformity score.
Write
\(\hat s_{i,t}:=\hat s(X_{i,t},Y_{i,t})\), including the latent target score
\(\hat s_{N+1,t}\), which is not observed at prediction time. Following the
weighted conformal construction of \citet{barber2023conformal}, we represent
this unobserved target score by a \(+\infty\) sentinel in the augmented score
vector \(\widetilde s_t=(\hat s_{1,t},\ldots,\hat s_{N,t},+\infty)\).
With
$\overline{\mathbb R}:=\mathbb R\cup\{-\infty,+\infty\}$, define for every level
$\beta\in\mathbb R$ the round-$t$ weighted quantile
\begin{equation}\label{eq:weighted-quantile}
\begin{aligned}
&Q_{1-\beta}(\widetilde s_t;W^{(t)})
:=
\inf\Bigl\{
q\in\overline{\mathbb R}:\\
&\qquad\sum_{k=1}^{N+1} w_{N+1,k}^{(t)} \mathbf 1\{(\widetilde s_t)_k\le q\}\ge 1-\beta
\Bigr\},
\end{aligned}
\end{equation}
where $(\widetilde s_t)_k$ denotes the $k$-th coordinate of $\widetilde s_t$
and the infimum of the empty set is $+\infty$. The weighted conformal threshold
and prediction set are
\begin{equation}\label{eq:base-prediction-set}
\begin{aligned}
\hat q_t&=Q_{1-\alpha_t}(\widetilde s_t;W^{(t)}),\\
C_t(X_{N+1,t})&=\{y\in\mathbb R:\hat s(X_{N+1,t},y)\le\hat q_t\}.
\end{aligned}
\end{equation}
Whenever $\alpha_t<w_{N+1,N+1}^{(t)}$, the threshold $\hat q_t$ is $+\infty$
by~\eqref{eq:weighted-quantile}, so the set $C_t$ in~\eqref{eq:base-prediction-set}
is the whole real line.

\paragraph{Finitization.}\label{subsec:finitization}
The whole real line is an uninformative prediction set, so W-TQA
replaces $\alpha_t$ in~\eqref{eq:base-prediction-set} by the
\emph{effective level}
\[
\tilde\alpha_t:=\max\bigl\{\alpha_t,\,w_{N+1,N+1}^{(t)}\bigr\}.
\]
When $\alpha_t<w_{N+1,N+1}^{(t)}$, this raises the level to
$w_{N+1,N+1}^{(t)}$, the smallest level at which the quantile is finite, and the
threshold becomes the largest calibration score $\max_{k\le N}\hat s_{k,t}$
(Lemma~\ref{lem:effective-level}), so for scores such as $|y-\hat f(x)|$ every
prediction set is bounded. We call such a round a \emph{fallback round}
and every other round a \emph{regular round}; on a fallback round the
prediction set is the \emph{fallback set}
$\{y\in\mathbb R:\hat s(X_{N+1,t},y)\le\max_{k\le N}\hat s_{k,t}\}$.
Finitization does not modify the level $\alpha_t$ itself: it
follows~\eqref{eq:base-level-update} on every revealed round, including
fallback rounds, and may become negative.
Algorithm~\ref{alg:wtqa} is the resulting procedure.

\begin{algorithm*}[t]
\caption{Online Weighted Temporal Quantile Adjustment (W-TQA)}\label{alg:wtqa}
\small
\begin{algorithmic}[1]
\Require score $\hat s$, target level $\alpha\in(0,1)$, bandwidth $h>0$, step size $\gamma>0$
\State Initialize $\alpha_0=\alpha$, $\widehat\mu_{k,0}=0$ for $k\le N+1$,
$W^{(1)}=(1/(N+1),\ldots,1/(N+1))$, and $R_0=0$.
\For{$t=1,2,\ldots$}
    \State Observe $\mathcal O_t=\bigl((X_{k,t},Y_{k,t})_{k\le N},\,X_{N+1,t},\,R_{t-1},\,R_{t-1}Y_{N+1,t-1}\bigr)$.
    \State Set $\alpha_t=\alpha_{t-1}+\gamma R_{t-1}(\alpha-\ell_{t-1})$, with $\ell_{t-1}=\mathbf 1\{Y_{N+1,t-1}\notin C_{t-1}\}$ computed only if $R_{t-1}=1$.
    \State Form $\widetilde s_t=\bigl(\hat s(X_{1,t},Y_{1,t}),\ldots,\hat s(X_{N,t},Y_{N,t}),+\infty\bigr)$.
    \State Set $\tilde\alpha_t=\max\{\alpha_t,w_{N+1,N+1}^{(t)}\}$ and
    $\hat q_t=Q_{1-\tilde\alpha_t}(\widetilde s_t;W^{(t)})$.
    \State Output $C_t=\{y\in\mathbb R:\hat s(X_{N+1,t},y)\le \hat q_t\}$.
    \State Update $\widehat\mu_{k,t}=(1-1/t)\widehat\mu_{k,t-1}+X_{k,t}/t$ for $k\le N+1$.
    \State Set $s_{t,j}$ to the cross-sectional standard deviation of $\{\widehat\mu_{i,t,j}\}_{i\le N+1}$ ($s_{t,j}=1$ if it is zero) and $\Lambda_t=\operatorname{diag}(s_{t,1},\dots,s_{t,d})$.
    \State Set $\kappa_{k,t}=\exp\{-\|\Lambda_t^{-1}(\widehat\mu_{k,t}-\widehat\mu_{N+1,t})\|_2^2/(2dh^2)\}$ for $k\le N$, $\kappa_{N+1,t}=1$, and $w^{(t+1)}_{N+1,k}=\kappa_{k,t}/\sum_{j\le N+1}\kappa_{j,t}$.
\EndFor
\end{algorithmic}
\end{algorithm*}

Past observations enter Algorithm~\ref{alg:wtqa} only through $W^{(t)}$ and
$\alpha_t$, and calibration uses only the current cross-section. W-TQA
therefore stores only the two states and the running feature means, so its
memory does not grow with $t$, and its per-round calibration cost is one
weighted split-conformal quantile over the current calibration panel of
size~\(N+1\).

\section{Theoretical Guarantees}\label{sec:wtqa_guarantees}

We analyze W-TQA through two guarantees: past-conditional
coverage for the weighted cross-section and long-run average coverage for the
adaptive temporal update.
All proofs, with the probability conventions they use, are in
Appendix~\ref{app:proofs}.

\subsection{Past-conditional coverage}\label{subsec:spatial_guarantees}

Fix a round $t$. Let
\[
P_t:=\mathcal L\!\left((\hat s_{1,t},\dots,\hat s_{N,t},\hat s_{N+1,t})\,\middle|\,
\mathcal F_{t-1}^{+}\right)
\]
be the conditional law of the round-$t$ score vector, and let
$P_{t,k}^{\mathrm{swap}}$ denote the corresponding law after swapping the $k$-th
calibration score $\hat s_{k,t}$ and the target score $\hat s_{N+1,t}$.
By Algorithm~\ref{alg:wtqa}, \(W^{(t)}\) and \(\alpha_t\), hence
\(\tilde\alpha_t=\max\{\alpha_t,w_{N+1,N+1}^{(t)}\}\), are
\(\mathcal F_{t-1}^{+}\)-measurable. Conditioning on \(\mathcal F_{t-1}^{+}\)
therefore fixes the weights and the effective level while leaving the round-\(t\)
scores random, and the weighted conformal bound of
\citet{barber2023conformal}, applied at level \(\tilde\alpha_t\), bounds the
round-\(t\) miscoverage probability given the past.
\begin{theorem}\label{thm:spatial_stepwise}
For every round $t$, the past-conditional miscoverage probability of
Algorithm~\ref{alg:wtqa} satisfies, almost surely,
\[
\mathbb P\!\left(
Y_{N+1,t}\notin C_t(X_{N+1,t})\,\middle|\,\mathcal F_{t-1}^{+}
\right)
\le
\tilde\alpha_t+\Delta_t ,
\]
where
\[
\Delta_t:=
\sum_{k=1}^N
w_{N+1,k}^{(t)}
d_{\mathrm{TV}}\!\left(P_t,P_{t,k}^{\mathrm{swap}}\right).
\]
\end{theorem}
Here and below, $d_{\mathrm{TV}}$ denotes total variation distance.
The term $\Delta_t$ in Theorem~\ref{thm:spatial_stepwise} is the price of
cross-sectional mismatch. We call
$d_{\mathrm{TV}}(P_t,P_{t,k}^{\mathrm{swap}})$ the \emph{swap discrepancy} of
unit $k$: it measures how much exchanging unit $k$'s score with the target's
changes the joint score law. The price $\Delta_t$ is the average swap
discrepancy under the deployed weights.
It is small when the weights concentrate on units whose scores are
nearly exchangeable with the target's, but the observed data cannot distinguish $\Delta_t=0$ from
its largest possible value.
\begin{proposition}\label{prop:spatial_impossibility}
Fix a round $t$ with $N \ge 2$ calibration units and a score $\hat s$ that is not constant in the response, and make no assumption linking units to each other or rounds to each other. There are two panel distributions under which everything observed before the target response has the same distribution, even if every past response was revealed, yet $\Delta_t = 0$ under one and $\Delta_t = 1 - w^{(t)}_{N+1,N+1}$, its largest possible value, under the other. When $\tilde\alpha_t<1$, Algorithm~\ref{alg:wtqa} misses with past-conditional probability $0$ under the first and $1/2$ under the second.
\end{proposition}

Any spatial guarantee must therefore rest on an assumption linking the target's score law to its peers'. We state ours in two parts: a latent profile that the running means learn, and an assumed link between profile distance and swap discrepancy.

\begin{assumption}[Mean-profile learnability and TV-profile control]\label{ass:spatial_profiles}
For each unit $k$, there exists a latent profile $\mu_k\in\mathbb R^d$
such that for every $\delta\in(0,1)$ and $t\ge2$, with
probability at least $1-\delta$,
\[
\max_{1\le k\le N+1}\rho_{t-1}(\widehat\mu_{k,t-1},\mu_k)
\le \varepsilon_{t,\delta}.
\]
Here $\varepsilon_{t,\delta}\in[0,\infty]$ is a deterministic tolerance.
In addition, for every $t\ge1$ there exists a nondecreasing function
$\phi_t:[0,\infty)\to[0,1]$, which may depend on $\mathcal F_{t-1}^{+}$, such
that, for every $k=1,\dots,N$, almost surely,
\begin{equation}\label{eq:tv-profile}
d_{\mathrm{TV}}\!\left(P_t,P_{t,k}^{\mathrm{swap}}\right)
\le
\phi_t\!\left(\rho_{t-1}(\mu_k,\mu_{N+1})\right).
\end{equation}
\end{assumption}

This assumption asks that the historical summaries used by W-TQA learn stable unit
profiles, and that units closer in this profile space have more similar score laws.
Because $\rho_{t-1}$ divides each coordinate by the cross-sectional standard
deviation $s_{t-1,j}$, $\varepsilon_{t,\delta}$ scales with the ratio of the
feature noise to the spread of the latent profiles across units. If, in each
coordinate, either the latent profiles differ across units or the feature takes
the same value for all units, then
$\varepsilon_{t,\delta}=O(\sqrt{\log(Nd/\delta)/t})$ under sub-Gaussian or
$m$-dependent conditions (Appendix~\ref{app:eps-concrete}). %
Appendix~\ref{app:homosc_factor} verifies the TV-profile condition~\eqref{eq:tv-profile} in a factor
model, in which an observed common factor drives every unit and the
unit-specific features share one covariance; on real panels, though the
condition itself cannot be tested (Proposition~\ref{prop:spatial_impossibility}),
Appendix~\ref{app:profile-distance-diagnostic} finds that lagged profile
distance predicts subsequent residual-law mismatch and that the deployed
weights reduce it relative to uniform weights.
Theorem~\ref{thm:spatial_stepwise} holds without Assumption~\ref{ass:spatial_profiles}.
The assumption is used only in the next result, which bounds $\Delta_t$ by the profile gap incurred by oracle weights computed from the latent profiles, plus the cost of estimating the profiles.
\begin{theorem}\label{thm:spatial_oracle}
Fix any $t\ge2$. Suppose Assumption~\ref{ass:spatial_profiles} holds. Define
\[
\begin{aligned}
\omega_k^\circ
&:=
\frac{\exp(-\rho_{t-1}(\mu_k,\mu_{N+1})^2/(2h^2))}
{1+\sum_{j=1}^N \exp(-\rho_{t-1}(\mu_j,\mu_{N+1})^2/(2h^2))},\\
D_\mu&:=\max_{1\le k\le N}\rho_{t-1}(\mu_k,\mu_{N+1}) .
\end{aligned}
\]
Then, for every $\delta\in(0,1)$, with probability at least $1-\delta$,
\[
\begin{aligned}
&\mathbb P\!\left(
Y_{N+1,t}\notin C_t(X_{N+1,t})\,\middle|\,\mathcal F_{t-1}^{+}
\right)\\
&\quad\le
\tilde\alpha_t
+
\underbrace{
\sum_{k=1}^N
\omega_k^\circ\,\phi_t\!\left(\rho_{t-1}(\mu_k,\mu_{N+1})\right)
}_{\text{oracle profile gap}}\\
&\qquad+
\underbrace{
\frac{D_\mu\,\varepsilon_{t,\delta}+\varepsilon_{t,\delta}^2}{h^2}
}_{\text{estimation remainder}} .
\end{aligned}
\]
\end{theorem}

The Gaussian-kernel weights in W-TQA implement a soft neighborhood rule in the
latent-profile space. Units whose profiles are close to the target's in the metric $\rho_{t-1}$ receive larger weight, while
distant units are downweighted smoothly rather than discarded by a hard nearest-neighbor
cutoff. The bandwidth \(h\) controls this localization, interpolating between
peer weights concentrated on the nearest profiles for small \(h\) and more
diffuse weights for large \(h\). In Theorem~\ref{thm:spatial_oracle}, a smaller $h$ concentrates
$\omega^\circ$ on the nearest profiles and shrinks the oracle profile gap,
but enlarges the $1/h^2$ factor in the estimation remainder.

The running-mean profile and the Gaussian kernel are the instantiation used
in Algorithm~\ref{alg:wtqa}, not requirements of the analysis.
Appendix~\ref{app:general-spatial} extends Theorem~\ref{thm:spatial_oracle}
to any predictable profile in a metric space and any $(0,1]$-valued kernel whose logarithm
is locally Lipschitz; the estimation remainder becomes the tolerance
$\varepsilon_{t,\delta}$ times that Lipschitz constant, which is $1/h$ for
the Laplace kernel and $1/(2h)$ for the inverse-multiquadric kernel with
scale $h$. Rolling-window or exponentially weighted profiles, and profiles
that include feature covariances (Appendix~\ref{app:second-order-extension}),
are covered by the same statement.

\subsection{Long-run average coverage}\label{subsec:temporal_guarantees}

Recall from Algorithm~\ref{alg:wtqa} that, with
\(\ell_t:=\mathbf 1\{Y_{N+1,t}\notin C_t(X_{N+1,t})\}\), a revealed round applies
the gradient step~\eqref{eq:base-level-update} and an unrevealed round leaves
the level unchanged, so that for every $t\ge1$
\begin{equation}\label{eq:level-recursion}
\alpha_{t+1}=\alpha_t+\gamma R_t(\alpha-\ell_t),
\end{equation}
and $\alpha_1=\alpha_0=\alpha$ because $R_0=0$. Write
$\mathcal R_T:=\{t\le T:R_t=1\}$ for the set of revealed rounds up to round $T$
and $|\mathcal R_T|$ for its cardinality.
The next result expresses the miscoverage on the revealed rounds through the
final level~$\alpha_{T+1}$ and bounds it from below.
\begin{proposition}\label{prop:temporal-observed}
Run Algorithm~\ref{alg:wtqa} with step size $\gamma>0$. Then
$\alpha_t\le1+\gamma\alpha$ for every $t\ge0$, and, almost surely, for every
$T\ge1$, on $\{\mathcal R_T\ne\varnothing\}$,
\begin{align}
\frac1{|\mathcal R_T|}\sum_{t\in\mathcal R_T}\ell_t-\alpha
&=\frac{\alpha-\alpha_{T+1}}{\gamma |\mathcal R_T|}
\label{eq:revealed-frequency-identity}\\
&\ge-\frac{1-\alpha+\gamma\alpha}{\gamma |\mathcal R_T|}.
\label{eq:revealed-lower-bound}
\end{align}
\end{proposition}

The level $\alpha_{T+1}$ is observed during the run,
so~\eqref{eq:revealed-frequency-identity} gives the realized deviation exactly
after deployment. An a priori guarantee needs two further steps:
\textbf{a lower bound on $\alpha_{T+1}$}, as in \citet{gibbs2021adaptive}, and
\textbf{a transfer from revealed to unrevealed rounds}.

\paragraph{Lower bound on $\alpha_{T+1}$.}
After finitization, $\alpha_{T+1}$ has no a priori lower bound: on a fallback round
W-TQA issues the fallback set, which can miss the target, and each revealed
miss lowers the level $\alpha_t$ further, even when it is already negative. We therefore
make the following assumption on the fallback set.

\begin{assumption}[Feasibility of the fallback set]\label{ass:feasibility}
There exists $\zeta\in(0,\alpha]$ such that, almost surely for every $t$,
\begin{equation}\label{eq:fallback-feasibility}
\mathbb P\Bigl(\hat s_{N+1,t}>\max_{k\le N}\hat s_{k,t}\,\Big|\,\mathcal F_{t-1}^{+}\Bigr)
\le\alpha-\zeta .
\end{equation}
\end{assumption}

The event in~\eqref{eq:fallback-feasibility} involves neither the weights
$W^{(t)}$ nor the level $\alpha_t$, so Assumption~\ref{ass:feasibility} constrains the panel rather
than the trajectory of the level. It is reasonable: the fallback set is the widest
set that W-TQA issues, and it misses only when the target's score
exceeds all $N$ calibration scores. When the target's score behaves like one
more calibration score, this has probability close to $1/(N+1)$, about $0.002$
for $N$ of order $500$ and far below $\alpha$.
More generally, for any fixed set $S$ of calibration units, this probability
is at most $1/(|S|+1)$ plus the average swap discrepancy of the units in~$S$
(Lemma~\ref{lem:feasibility-from-swap}).

\paragraph{Transfer from revealed to unrevealed rounds.}
The next result, the selection bound of \citet{manski1989} applied round by
round, shows that the observed data restrict $\mathbb P(\ell_t=1)$ to an
interval but cannot single out a value within it.
\begin{proposition}\label{prop:temporal_impossibility}
Fix a round $t$ and any online procedure whose prediction set $C_t$ is neither
empty nor all of $\mathbb R$, in the sense that there are $\mathcal F_t$-measurable
points $y_t^{\mathrm{in}}\in C_t$ and $y_t^{\mathrm{out}}\notin C_t$ almost surely. With
$\pi_t:=\mathbb P(R_t=1)>0$ and
$m_t^{\mathrm{obs}}:=\mathbb P(\ell_t=1\mid R_t=1)$,
\begin{equation}\label{eq:identification-interval}
\mathbb P(\ell_t=1)\in
\big[\,\pi_t m_t^{\mathrm{obs}},\;\pi_t m_t^{\mathrm{obs}}+(1-\pi_t)\,\big].
\end{equation}
Conversely, for any $\pi_t\in(0,1]$ and $m_t^{\mathrm{obs}}\in[0,1]$, there is
a family of panel distributions under which everything observed, including
$R_t$ and $R_tY_{N+1,t}$, has the same distribution, while $\mathbb P(\ell_t=1)$
ranges over all of~\eqref{eq:identification-interval}.
\end{proposition}

Misses on unrevealed rounds leave no trace in the observed data, so the data
cannot narrow~\eqref{eq:identification-interval}, at a single round or on average
over rounds (Appendix~\ref{app:selection_bias}). This transfer is therefore not
possible without an assumption on the reveal mechanism.

\begin{assumption}[MCAR target feedback]\label{ass:mcar}
There exists $\pi\in(0,1]$ such that, with
$\mathcal H_t:=\mathcal F_t\vee\sigma(Y_{N+1,t})$,
\[
\mathbb P(R_t=1\mid\mathcal H_t)=\pi
\qquad\text{a.s.\ for all }t.
\]
\end{assumption}

Since $\ell_t$ is $\mathcal H_t$-measurable, revelation does not depend on
whether the current set covers, so $\mathbb P(\ell_t=1)=m_t^{\mathrm{obs}}$.
Combining Proposition~\ref{prop:temporal-observed} with the lower bound on
$\alpha_{T+1}$ supplied by Assumption~\ref{ass:feasibility} and the transfer
supplied by Assumption~\ref{ass:mcar}, the next result gives long-run average
coverage.
\begin{theorem}\label{thm:temporal_mcar}
Run Algorithm~\ref{alg:wtqa} with step size $\gamma>0$ and suppose
Assumptions~\ref{ass:feasibility} and~\ref{ass:mcar} hold. Then, for every $T\ge1$
and $\delta\in(0,1)$, with probability at least $1-\delta$, if
$\mathcal R_T\ne\varnothing$, the miscoverage on the revealed rounds satisfies
\begin{equation}\label{eq:revealed-two-sided}
\begin{aligned}
-\frac{1-\alpha+\gamma\alpha}{\gamma|\mathcal R_T|}
&\le\frac1{|\mathcal R_T|}\sum_{t\in\mathcal R_T}\ell_t-\alpha\\
&\le\frac{\alpha+\gamma(1-\alpha)+\gamma\log(1/\delta)/(8\zeta)}{\gamma|\mathcal R_T|}.
\end{aligned}
\end{equation}
For every $T\ge1$, the average miscoverage probability over all rounds satisfies
\begin{equation}\label{eq:all-round-bound}
\begin{aligned}
-\frac{1-\alpha+\gamma\alpha}{\pi\gamma T}
&\le\frac1T\sum_{t=1}^T
\mathbb P\!\left(Y_{N+1,t}\notin C_t(X_{N+1,t})\right)-\alpha\\
&\le\frac{\alpha+\gamma(1-\alpha)+\gamma/(8\zeta)}{\pi\gamma T}.
\end{aligned}
\end{equation}
\end{theorem}
The bound~\eqref{eq:revealed-two-sided} controls, at rate $1/(\gamma|\mathcal R_T|)$,
the realized miscoverage on the revealed rounds, a quantity observed after
deployment, and~\eqref{eq:all-round-bound} controls the expected miscoverage
over all rounds at rate $1/(\pi\gamma T)$. The term $\gamma/(8\zeta)$
in~\eqref{eq:all-round-bound} is the cost of finitization: below zero, the level
$\alpha_t$ rises on average by at least $\gamma\zeta$ per revealed round, so its
excursions below $-\gamma(1-\alpha)$ have an exponential tail
(Lemma~\ref{lem:level-stability}); for $\gamma=0.01$ and $\zeta$ near $\alpha=0.1$ it
is about $0.013$. When Assumption~\ref{ass:mcar} fails,
the all-round deviation contains the additional selection-bias term
$-(\pi T)^{-1}\sum_{t=1}^T\mathrm{Cov}(p_t,\ell_t)$, where
$p_t:=\mathbb P(R_t=1\mid\mathcal H_t)$ has mean $\pi$; it can reach
$(1-\pi)(1-\alpha)$ in absolute value (Appendix~\ref{app:selection_bias}).

\section{Experiments}\label{sec:experiments}

\subsection{Experimental setup}\label{subsec:exp-setup}

\paragraph{Data sets.}
We evaluate W-TQA on one synthetic panel with three difficulty levels and four
real panels (Table~\ref{tab:experiment-panel-splits}). In the synthetic panel
and in HF, M5 and SGSC the target outcome is realized together with the
calibration outcomes; we withhold it until after prediction and make feedback
intermittent by masking it at random, which lets us vary the feedback
probability $p$. The cross-market panel has this timing natively: at each Tokyo
open the U.S.\ calibration stocks' latest session has closed, and the Tokyo
target's open-to-close return is realized before the next round, so every round
is a revealed round ($p=1$).

\begin{center}
  \scriptsize
  \captionsetup{skip=5pt}
  \setlength{\tabcolsep}{2pt}
  \captionof{table}{Panels used in the experiments; $T$ is the number of conformal
  rounds.}\label{tab:experiment-panel-splits}
  \begin{tabular}{@{}llccc@{}}
    \toprule
    Data set & Panel / target & Calib. & Test & $T$\\
    \midrule
    Synthetic & factor model, Easy/Med./Hard & 470 & 30 & 60\\
    HF hourly & next-hour U.S.\ stock return & 395 & 100 & 500\\
    M5 & next-day item log sales & 454 & 194 & 600\\
    SGSC & next-half-hour household log load & 280 & 120 & 600\\
    Cross-market & Tokyo return, U.S.\ calib. & 386 & 100 & 100\\
    \bottomrule
  \end{tabular}
\end{center}

\begin{figure*}[t]
  \centering
  \begin{minipage}[t]{0.48\linewidth}
    \centering
    \includegraphics[width=\linewidth]{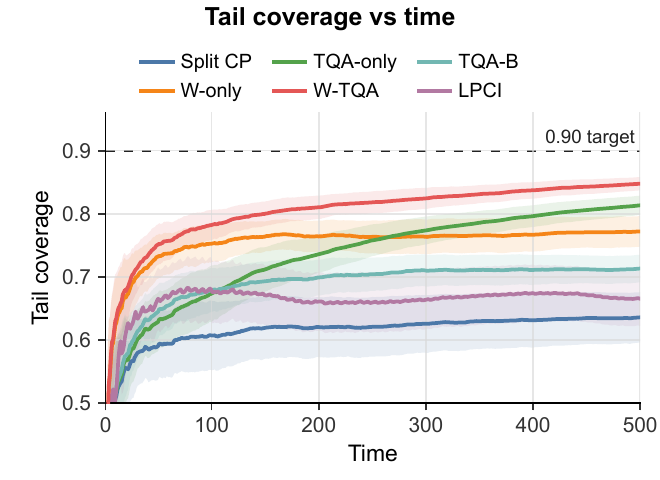}
  \end{minipage}\hfill
  \begin{minipage}[t]{0.48\linewidth}
    \centering
    \includegraphics[width=\linewidth]{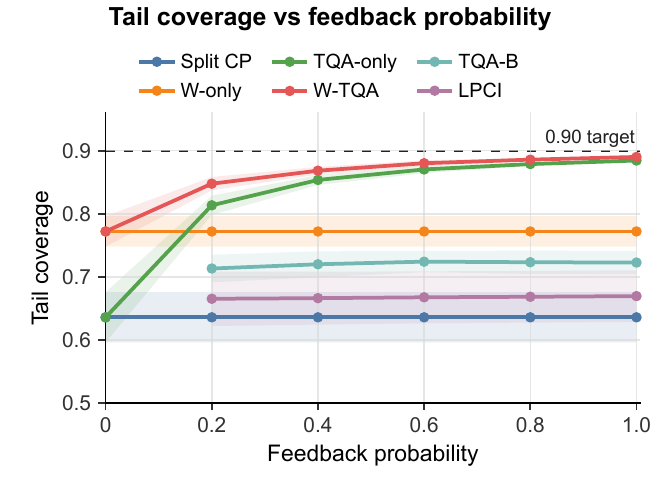}
  \end{minipage}
  \caption{Mechanism plots on the HF hourly panel, means over the $30$ splits
  with $\pm1$ sd bands. Left: cumulative tail coverage at target-feedback
  probability $p=0.2$. Right: tail coverage versus target-feedback observation
  probability \(p\).}\label{fig:hf_mechanisms}
\end{figure*}

\paragraph{Experimental pipeline.}
Each replication splits the units into a calibration panel and a held-out test
panel and treats each test unit as a separate target with its own
\(\alpha_t\) and \(W^{(t)}\). On every panel with masked feedback, the covariates are
built only from exogenous variables, burn-in statistics and calibration outcomes
(Table~\ref{tab:real-features}), so no method sees a withheld target outcome. A single
standardized ridge predictor
\(\hat f\) is fit only on burn-in observations of the calibration units and
held fixed, and all methods other than LPCI (which regresses quantiles of the
signed residual \(Y_{i,t}-\hat f(X_{i,t})\) on its lagged values) use the score
\(\hat s_{i,t}=|Y_{i,t}-\hat f(X_{i,t})|\), so their prediction sets are
intervals centered at \(\hat f(X_{i,t})\). At each round, the contemporaneous
calibration-unit outcomes form the round-$t$ calibration scores, and the
methods issue intervals for the test units before their outcomes are revealed.
The reveal indicator \(R_t\sim\mathrm{Bernoulli}(p)\) is drawn once per round
and shared by all test units (\(p=1\) is full feedback); if \(R_t=0\), the
round's test outcomes are withheld from online updates but still used for
evaluation.
This reveal mechanism is MCAR (Assumption~\ref{ass:mcar});
Appendix~\ref{app:nonmcar-feedback-exp} stress-tests W-TQA under
outcome-informative reveal mechanisms that violate it, and W-TQA is robust to
these violations.

\paragraph{Methods.}
We compare W-TQA with \textbf{Split CP} (non-adaptive split conformal); its
ablations \textbf{W-only} (Algorithm~\ref{alg:wtqa} with $\alpha_t\equiv\alpha$) and
\textbf{TQA-only} (Algorithm~\ref{alg:wtqa} with $W^{(t)}\equiv W^{(1)}$, the uniform
weights; its level update is the TQA-E rule of \citet{lin2022conformal});
and two adaptive baselines, \textbf{TQA-B} (the budgeted TQA variant of
\citet{lin2022conformal}) and \textbf{LPCI} (the lagged-residual
quantile-regression method of \citet{batra2023conformal}, in a lighter variant
on the longer real panels for tractability;
Appendices~\ref{app:simulation_experiments} and~\ref{app:real-data-panels}).
Hyperparameters are fixed in advance, with $\alpha=0.10$ for every method and
$h=0.6$ and $\gamma=0.01$ for W-TQA and its ablations on every data set; W-TQA is robust to
$(h,\gamma)$ around these defaults (Appendix~\ref{app:parameter-robustness}).
\paragraph{Metrics.}\label{sec:metrics}
Following \citet{batra2023conformal}, we report average coverage, tail
coverage, average width, and the coefficient of variation of interval widths
(Width CoV). Average coverage is a calibration sanity check but misses the
central difficulty of our setting: because we randomly split the units of HF,
M5 and SGSC into calibration and test, the test-unit and calibration-unit
score distributions agree in expectation over the split, so average coverage
near \(1-\alpha\) is attainable without addressing cross-sectional
heterogeneity, and averaging hides the harder units, whose score laws differ
most from the calibration majority. We therefore
emphasize \emph{tail coverage}, defined as the mean per-unit coverage over
the worst-covered \(10\%\) of test units in each replication. For interval
size, average width measures the overall scale, while Width CoV---the
standard deviation divided by the mean of realized interval widths in a
replication---measures how much the widths vary across units and time.
As in \citet{batra2023conformal}, larger values are preferred: Width CoV is
invariant to rescaling all widths by a common factor, so it rewards redistributing
width across units and time rather than widening every interval.

\subsection{Main results}\label{subsec:exp-main}

In Table~\ref{tab:main-results}, HF, M5 and SGSC are reported at $p=0.2$, an
instance of the intermittent feedback that motivates this paper, with about
$100$--$120$ target outcomes revealed over $500$--$600$ rounds; the synthetic
and cross-market panels are reported at $p=1$, with a comparable $60$ and
$100$. Coverage is evaluated
on all rounds; the complete grid over $p$ is in
Appendix~\ref{app:real-data-panels}.

\begin{table}[!t]
  \centering
  \scriptsize
  \captionsetup{skip=5pt}
  \setlength{\tabcolsep}{2.2pt}
  \renewcommand{\arraystretch}{0.95}
  \caption{\textbf{Main results}: synthetic panels at $p=1$, HF, M5 and SGSC at
  $p=0.2$, and the cross-market panel (Cross) at $p=1$. Means over $30$
  replications (standard deviations in Table~\ref{tab:main-results-sd}).
  Relative width is average width over that of Split~CP on the same
  replication. In each column, bold marks the average and tail coverage closest
  to $0.90$ and the largest Width CoV, ties included.}\label{tab:main-results}
  \begin{tabular}{@{}llccccccc@{}}
  \toprule
  & & \multicolumn{3}{c}{Synthetic, $p=1$} & \multicolumn{3}{c}{Real, $p=0.2$} & $p=1$\\
  \cmidrule(lr){3-5}\cmidrule(lr){6-8}\cmidrule(l){9-9}
  & Method & Easy & Med. & Hard & HF & M5 & SGSC & Cross\\
  \midrule
  \multirow{6}{*}{\rotatebox{90}{Avg.\ cov.}} & Split CP & $0.874$ & $0.848$ & $0.746$ & $0.904$ & $\mathbf{0.904}$ & $\mathbf{0.902}$ & $\mathbf{0.904}$ \\
   & TQA-B & $\mathbf{0.883}$ & $0.867$ & $0.819$ & $0.914$ & $0.920$ & $0.916$ & $0.905$ \\
   & LPCI & $0.788$ & $0.730$ & $0.709$ & $\mathbf{0.898}$ & $0.878$ & $0.864$ & $\mathbf{0.904}$ \\
   & TQA-only & $0.882$ & $0.864$ & $0.815$ & $0.922$ & $0.931$ & $0.924$ & $0.905$ \\
   & W-only & $0.936$ & $0.937$ & $0.929$ & $0.927$ & $0.913$ & $0.933$ & $0.915$ \\
   & W-TQA & $0.928$ & $\mathbf{0.929}$ & $\mathbf{0.922}$ & $0.926$ & $0.925$ & $0.935$ & $0.907$ \\
  \midrule
  \multirow{6}{*}{\rotatebox{90}{Tail cov.}} & Split CP & $0.800$ & $0.752$ & $0.502$ & $0.636$ & $0.579$ & $0.646$ & $0.727$ \\
   & TQA-B & $0.825$ & $0.803$ & $0.713$ & $0.713$ & $0.690$ & $0.751$ & $0.763$ \\
   & LPCI & $0.711$ & $0.648$ & $0.616$ & $0.665$ & $0.645$ & $0.575$ & $0.721$ \\
   & TQA-only & $0.832$ & $0.804$ & $0.729$ & $0.813$ & $0.827$ & $0.836$ & $0.824$ \\
   & W-only & $0.880$ & $0.872$ & $0.781$ & $0.772$ & $0.697$ & $0.763$ & $0.826$ \\
   & W-TQA & $\mathbf{0.884}$ & $\mathbf{0.880}$ & $\mathbf{0.828}$ & $\mathbf{0.848}$ & $\mathbf{0.847}$ & $\mathbf{0.858}$ & $\mathbf{0.858}$ \\
  \midrule
  \multirow{5}{*}{\rotatebox{90}{\shortstack{Rel.\\width}}} & TQA-B & $1.042$ & $1.058$ & $1.163$ & $1.030$ & $1.006$ & $1.033$ & $1.001$ \\
   & LPCI & $0.870$ & $0.800$ & $0.881$ & $0.982$ & $0.786$ & $0.830$ & $0.978$ \\
   & TQA-only & $1.024$ & $1.041$ & $1.154$ & $1.048$ & $0.986$ & $1.024$ & $0.995$ \\
   & W-only & $1.251$ & $1.293$ & $1.553$ & $1.113$ & $0.991$ & $1.297$ & $1.083$ \\
   & W-TQA & $1.217$ & $1.256$ & $1.488$ & $1.084$ & $0.963$ & $1.172$ & $1.056$ \\
  \midrule
  \multirow{6}{*}{\rotatebox{90}{Width CoV}} & Split CP & $0.108$ & $0.228$ & $0.217$ & $0.568$ & $0.076$ & $0.346$ & $0.310$ \\
   & TQA-B & $0.172$ & $0.263$ & $0.279$ & $0.688$ & $0.118$ & $0.366$ & $0.340$ \\
   & LPCI & $0.174$ & $0.239$ & $0.269$ & $0.087$ & $0.034$ & $0.209$ & $0.099$ \\
   & TQA-only & $0.118$ & $0.243$ & $\mathbf{0.288}$ & $0.901$ & $0.233$ & $0.489$ & $0.370$ \\
   & W-only & $0.214$ & $\mathbf{0.275}$ & $0.263$ & $1.002$ & $0.332$ & $0.566$ & $0.566$ \\
   & W-TQA & $\mathbf{0.225}$ & $\mathbf{0.275}$ & $0.273$ & $\mathbf{1.068}$ & $\mathbf{0.358}$ & $\mathbf{0.621}$ & $\mathbf{0.607}$ \\
  \bottomrule
  \end{tabular}
\end{table}

Average coverage separates the methods little on the real panels, where every
method except LPCI is at or above $1-\alpha$, as Section~\ref{sec:metrics}
anticipates. On the synthetic panels only the two methods with the spatial
branch reach it. Tail coverage separates the methods: W-TQA is highest in every
column. This is not bought by widening every interval: W-TQA has the largest
Width CoV in every column except Synthetic Hard, and on the real panels its
average width lies within the range of the other methods that reach
$1-\alpha$ (even below it on M5). On the synthetic panels the
narrower methods under-cover even on average; of the two that reach
$1-\alpha$, W-TQA is the narrower and has the higher tail coverage. W-TQA stays
ahead of both ablations with two alternative base predictors
(Appendix~\ref{app:base-predictor-sensitivity}) and over different time periods
of the cross-market panel (Appendix~\ref{app:cross-market-panel}).

\subsection{Complementarity of the two branches}\label{subsec:exp-complementary}

The two branches of W-TQA contribute in complementary feedback regimes: the similarity
weights raise tail coverage from the outset, since they need no target
feedback, while the adaptive level needs revealed target outcomes and its
contribution grows as they accumulate. Figure~\ref{fig:hf_mechanisms} illustrates
this on HF; M5 and SGSC show the same pattern (Appendix~\ref{app:real-data-panels}).

\paragraph{Within one panel over time.}
In Figure~\ref{fig:hf_mechanisms} (left), W-TQA tracks W-only
during the early rounds, when few target
outcomes have been revealed, and moves toward TQA-only as feedback accrues,
its cumulative tail coverage staying above both ablations throughout.

\paragraph{Across feedback probabilities.}
Figure~\ref{fig:hf_mechanisms} (right) sweeps the target-feedback probability
$p$. Split~CP and W-only ignore target feedback and are flat in $p$. TQA-only
is weak when feedback is sparse and approaches W-TQA as feedback accumulates.
W-TQA coincides with W-only at $p=0$ and attains the highest tail coverage at
every $p>0$; its margin over TQA-only narrows as feedback accumulates, but it
remains ahead in nearly all splits (Table~\ref{tab:paired-margins}).

\bibliography{bibliography}

\clearpage
\appendix
\onecolumn

\raggedbottom
\makeatletter\setlength{\@fptop}{0pt}\makeatother %

\section{Additional Problem Context}\label{app:additional_context}\label{app:motivating_examples}

\paragraph{Financial panels.}
Our primary motivation is financial prediction with asynchronous observation times.
One example is cross-market prediction across exchanges with different closing
times: at a given decision point, the realized returns of markets that have already
closed are observed, while the target market's return has not yet been realized and
will be revealed before the next trading round \citep{eun1989international,hamao1990correlations}.
A second example is prediction for an illiquid target asset using a panel of more liquid
related assets. Liquid assets incorporate information quickly, so their returns are
available promptly; the illiquid target's return, by contrast, may arrive with a delay
or may be missing altogether when no trade occurs during the decision window
\citep{chordia2000trading,hou2005market}.
In both cases, \(Y_{i,t}\) may be a realized or excess return over the decision window,
and \(X_{i,t}\) may include available lagged returns, trading volume, volatility, liquidity
proxies, sector characteristics, market factors, and calendar indicators.

\paragraph{Transportation monitoring.}
Another setting is short-horizon traffic prediction on a road network. Units are road
links, and \(Y_{i,t}\) may be the travel time or average speed on link \(i\) during time
window \(t\). Some links are instrumented with loop detectors, cameras, Bluetooth
readers, or dense probe-vehicle data, so their current outcomes are available. Other
links may be unobserved, sparsely observed, or observed only after a delay. The target
unit is such a link, while the observed units form a contemporaneous traffic panel.
Features \(X_{i,t}\) may include available lagged speeds and travel times, upstream and downstream
sensor readings, link length, road class, number of lanes, time-of-day and day-of-week
indicators, weather, incidents, and congestion summaries. This setting matches the
problem of predicting travel times on observed and unobserved road links and the use of
sparse multi-sensor traffic data \citep{kessler2021multi,li2022prediction}.

\paragraph{Retail store--SKU panels.}
A third setting is retail demand prediction at the store--SKU (stock-keeping unit)
or channel--product level. Units are products, stores, or product-store pairs, and
\(Y_{i,t}\) may be daily or weekly sales. In practice, some units report outcomes
promptly while others are delayed or unavailable, for instance because of reporting
latency or absent transactions for a niche product. The features \(X_{i,t}\) can
include available lagged sales, prices, promotions, inventory signals, calendar and holiday
indicators, and store and product attributes. This is a large heterogeneous
panel prediction problem, as in modern retail demand forecasting and the M5 store--SKU
sales panel \citep{fildes2022retail,makridakis2022m5}.

\section{Proofs}\label{app:proofs}

\paragraph{Probability conventions.}
All random quantities are defined on a common probability space
$(\Omega,\mathcal A,\mathbb P)$. The predictor $\hat f$ is trained on
burn-in data that precede the panel stream, and the score $\hat s$ is Borel
measurable. All probabilities and expectations
are conditional on these data and are taken over the panel stream and the
target-feedback process;
conditional probabilities such as $\mathbb P(\cdot\mid\mathcal F_{t-1}^{+})$
and $\mathbb P(\cdot\mid\mathcal H_t)$ are regular conditional probabilities.
By Algorithm~\ref{alg:wtqa}, the prediction set $C_t(X_{N+1,t})$ is
$\mathcal F_t$-measurable, the weight vector $W^{(t)}\in\mathcal W_{N+1}$ is
$\mathcal F_{t-1}$-measurable, and the level $\alpha_t$, the effective level
$\tilde\alpha_t=\max\{\alpha_t,w_{N+1,N+1}^{(t)}\}$ and the fallback event
\[
E_t:=\{w_{N+1,N+1}^{(t)}>\alpha_t\}
\]
on which round $t$ is a fallback round of Section~\ref{sec:proposed_method},
are $\mathcal F_{t-1}^{+}$-measurable.
We set $\kappa_{k,0}:=1$ for $k\le N+1$, which reproduces the uniform $W^{(1)}$
of Algorithm~\ref{alg:wtqa}, so that for every $t\ge1$
\begin{equation}\label{eq:weights-from-kernels}
w^{(t)}_{N+1,k}=\frac{\kappa_{k,t-1}}{\sum_{j\le N+1}\kappa_{j,t-1}},
\qquad
w^{(t)}_{N+1,N+1}=\Bigl(1+\sum_{j\le N}\kappa_{j,t-1}\Bigr)^{-1}.
\end{equation}
Since $0<\kappa_{k,t-1}\le1=\kappa_{N+1,t-1}$ for $k\le N$,
every peer weight satisfies $0<w^{(t)}_{N+1,k}\le w^{(t)}_{N+1,N+1}<1$.
Appendix~\ref{app:oracle-bound} proves the results of
Section~\ref{subsec:spatial_guarantees} and
Appendix~\ref{app:temporal-section} those of
Section~\ref{subsec:temporal_guarantees}.

\subsection{Past-conditional coverage}\label{app:oracle-bound}

\subsubsection{The weighted quantile and the proof of Theorem~\ref{thm:spatial_stepwise}}\label{app:proof-tv-bound}

The first lemma describes the weighted quantile of Section~\ref{sec:proposed_method}
around the sentinel weight; it shows that the threshold of Algorithm~\ref{alg:wtqa}
equals the largest calibration score on every fallback round.
\begin{lemma}\label{lem:effective-level}
Fix a round $t$ and write $w:=w_{N+1,N+1}^{(t)}\in(0,1)$. Then
$Q_{1-\beta}(\widetilde s_t;W^{(t)})=+\infty$ for $\beta<w$;
$\min_{k\le N}\hat s_{k,t}\le Q_{1-\beta}(\widetilde s_t;W^{(t)})\le\max_{k\le N}\hat s_{k,t}$
for $w\le\beta<1$; $Q_{1-w}(\widetilde s_t;W^{(t)})=\max_{k\le N}\hat s_{k,t}$; and
$Q_{1-\beta}(\widetilde s_t;W^{(t)})=-\infty$ for $\beta\ge1$. Consequently
$\hat q_t=Q_{1-\alpha_t}(\widetilde s_t;W^{(t)})$ on $E_t^{c}$ and
$\hat q_t=\max_{k\le N}\hat s_{k,t}$ on $E_t$.
\end{lemma}

\begin{proof}
For $q\in\overline{\mathbb R}$ let
\begin{equation}\label{eq:weighted-cdf}
G_t(q):=\sum_{k=1}^{N+1}w_{N+1,k}^{(t)}\mathbf 1\{(\widetilde s_t)_k\le q\},
\end{equation}
so that $Q_{1-\beta}(\widetilde s_t;W^{(t)})=\inf\{q\in\overline{\mathbb R}:G_t(q)\ge1-\beta\}$.
The sentinel coordinate equals $+\infty$, so for every $q<+\infty$,
\[
G_t(q)=\sum_{k=1}^{N}w_{N+1,k}^{(t)}\mathbf 1\{\hat s_{k,t}\le q\}\le1-w,
\]
with equality if and only if $q\ge\max_{k\le N}\hat s_{k,t}$, because every peer
weight is strictly positive; moreover $G_t(q)=0$ for
$q<\min_{k\le N}\hat s_{k,t}$ and $G_t(+\infty)=1$.
If $\beta<w$, then $1-\beta>1-w\ge G_t(q)$ for every $q<+\infty$, so the
infimum is $+\infty$. If $w\le\beta<1$, then
$G_t(\max_{k\le N}\hat s_{k,t})=1-w\ge1-\beta$, which gives the upper bound, and
$G_t(q)=0<1-\beta$ for $q<\min_{k\le N}\hat s_{k,t}$, which gives the lower
bound. If $\beta=w$, then $G_t(q)\ge1-w$ holds for $q<+\infty$ if and only if
$q\ge\max_{k\le N}\hat s_{k,t}$, so the infimum equals
$\max_{k\le N}\hat s_{k,t}$. If $\beta\ge1$, then $G_t(q)\ge0\ge1-\beta$ for
every $q\in\overline{\mathbb R}$, so the infimum is $-\infty$.
Finally, $\tilde\alpha_t=\alpha_t$ on $E_t^{c}$ and $\tilde\alpha_t=w$ on $E_t$,
and $\hat q_t=Q_{1-\tilde\alpha_t}(\widetilde s_t;W^{(t)})$ by definition.
\end{proof}

\begin{proof}[Proof of Theorem~\ref{thm:spatial_stepwise}]
Fix a round $t$. If $\tilde\alpha_t\ge1$ the claimed bound is at least one, so
suppose $\tilde\alpha_t<1$. Conditionally on $\mathcal F_{t-1}^{+}$, the
weights $W^{(t)}$ and the effective level $\tilde\alpha_t$ are fixed, and by
\eqref{eq:weights-from-kernels} the round-$t$ step of Algorithm~\ref{alg:wtqa}
is the weighted split conformal procedure of \citet{barber2023conformal} with
calibration scores $\hat s_{1,t},\dots,\hat s_{N,t}$, test score
$\hat s_{N+1,t}$, unnormalized weights
$\kappa_{k,t-1}=w_{N+1,k}^{(t)}/w_{N+1,N+1}^{(t)}\in(0,1]$ on the calibration
scores and one on the test score, and level $\tilde\alpha_t$.
Let $Z^{(t)}:=(\hat s_{1,t},\dots,\hat s_{N,t},\hat s_{N+1,t})$ and, for
$k\le N$, let $Z^{(t,k)}$ be the vector obtained from $Z^{(t)}$ by exchanging
coordinates $k$ and $N+1$, so that
\[
\mathcal L(Z^{(t)}\mid\mathcal F_{t-1}^{+})=P_t,
\qquad
\mathcal L(Z^{(t,k)}\mid\mathcal F_{t-1}^{+})=P_{t,k}^{\mathrm{swap}}.
\]
Theorem~2 of \citet{barber2023conformal}, applied to these regular conditional
laws for every realized history in a full-probability set, gives
\[
\mathbb P\bigl(Y_{N+1,t}\in C_t(X_{N+1,t})\,\big|\,\mathcal F_{t-1}^{+}\bigr)
\;\ge\;
1-\tilde\alpha_t
-\sum_{k=1}^N w_{N+1,k}^{(t)}\,d_{\mathrm{TV}}\bigl(P_t,P_{t,k}^{\mathrm{swap}}\bigr)
=1-\tilde\alpha_t-\Delta_t .
\qedhere
\]
\end{proof}

\subsubsection{Proof of Proposition~\ref{prop:spatial_impossibility}}\label{app:spatial-impossibility}

\begin{proof}
Fix $t\ge1$ and $N\ge2$. Since $\hat s$ is not constant in the response, fix
$x_0\in\mathbb R^d$ and $y_0,y_1\in\mathbb R$ with
$\xi_0:=\hat s(x_0,y_0)\ne\xi_1:=\hat s(x_0,y_1)$. Let $V$ be a
Bernoulli$(1/2)$ variable and define two panel laws $Q_0$ and $Q_1$ that share
the covariates $X_{i,s}=x_0$ for all $i\le N+1$ and $s\le t$, the same
deterministic responses before round $t$, and $R_s\equiv1$ for $s<t$. At round
$t$, both laws set $Y_{i,t}=y_V$ for every $i\le N$; the target response is
$Y_{N+1,t}=y_V$ under $Q_0$ and $Y_{N+1,t}=y_{1-V}$ under $Q_1$. The round-$t$
score vectors are therefore
\[
\begin{aligned}
Q_0:\quad (\hat s_{1,t},\ldots,\hat s_{N,t},\hat s_{N+1,t})
&=(\xi_V,\ldots,\xi_V,\xi_V),\\
Q_1:\quad (\hat s_{1,t},\ldots,\hat s_{N,t},\hat s_{N+1,t})
&=(\xi_V,\ldots,\xi_V,\xi_{1-V}).
\end{aligned}
\]
Everything observed before the target response, namely
$(\mathcal O_1,\ldots,\mathcal O_t)$, has the same law under $Q_0$ and $Q_1$:
the histories and covariates coincide, every past response is revealed, and
the round-$t$ calibration responses equal $y_V$ under both laws. In particular $W^{(t)}$ and $\alpha_t$ coincide
under $Q_0$ and $Q_1$.

Under $Q_0$, exchanging $\hat s_{k,t}$ with $\hat s_{N+1,t}$ leaves the score
vector unchanged, so $d_{\mathrm{TV}}(P_t,P_{t,k}^{\mathrm{swap}})=0$ for every
$k$ and $\Delta_t=0$. Under $Q_1$, the score vector takes values in
$\{(\xi_0,\ldots,\xi_0,\xi_1),(\xi_1,\ldots,\xi_1,\xi_0)\}$, whose peer
coordinates are all equal, whereas the vector with coordinates $k$ and $N+1$
exchanged has coordinate $k$ equal to $\xi_{1-V}$ and, because $N\ge2$, another
peer coordinate equal to $\xi_V\ne\xi_{1-V}$; the two supports are disjoint, so
$d_{\mathrm{TV}}(P_t,P_{t,k}^{\mathrm{swap}})=1$ for every $k$ and
\[
\Delta_t=\sum_{k=1}^Nw_{N+1,k}^{(t)}=1-w_{N+1,N+1}^{(t)},
\]
the largest value $\Delta_t$ can take, since each swap discrepancy is at most one.

Suppose $\tilde\alpha_t<1$; this holds at $t=1$, where $\alpha_1=\alpha<1$ and
$w_{N+1,N+1}^{(1)}=1/(N+1)<1$. Since $w_{N+1,N+1}^{(t)}\le\tilde\alpha_t<1$
and every calibration score equals $\xi_V$, Lemma~\ref{lem:effective-level}
gives $\hat q_t=\xi_V$ under both laws. Under $Q_0$ the target score $\xi_V$
does not exceed $\hat q_t$, so the set covers. Under $Q_1$ the target score
$\xi_{1-V}$ exceeds $\hat q_t=\xi_V$ for exactly one of the two values of $V$.
Every generator of $\mathcal F_{t-1}^{+}$ is deterministic under both laws, so
past-conditional probabilities equal unconditional ones, and the miss
probability is $0$ under $Q_0$ and $1/2$ under $Q_1$.
\end{proof}

\subsubsection{General profiles and kernels}\label{app:general-spatial}

The following assumption extends Assumption~\ref{ass:spatial_profiles}
from running-mean profiles with the metric $\rho_{t-1}$ to predictable profiles in a general metric space.

\begin{assumption}[Profile learnability and TV-profile control]\label{ass:spatial_profiles_general}
Let $(\mathcal U,\rho)$ be a metric space and let $\widehat z_{k,t-1} \in\mathcal U$ be the profile formed for unit $k$ before the round-$t$ scores are inspected, so that $\widehat z_{k,t-1}$ is $\mathcal F_{t-1}^{+}$-measurable.
The metric $\rho$ and the function $\phi$ in (ii) may themselves be $\mathcal F_{t-1}^{+}$-measurable ($\rho=\rho_{t-1}$, $\phi=\phi_t$).
\begin{enumerate}
\item[(i)] \emph{Profile learnability.} There are $\mathcal F_{t-1}^{+}$-measurable oracle profiles $z^\circ_{k,t-1} \in\mathcal U$ and, for every $t \ge 2$ and $\delta \in (0,1)$, a finite $\mathcal F_{t-1}^{+}$-measurable radius $\varepsilon_{t,\delta} \ge 0$ such that
$\max_{k \le N+1} \rho(\widehat z_{k,t-1}, z^\circ_{k,t-1}) \le \varepsilon_{t,\delta}$ with probability at least $1-\delta$.
\item[(ii)] \emph{TV-profile control.} There is a nondecreasing $\phi : [0,\infty) \to [0,1]$ such that, for every $t \ge 1$ and $k \le N$, almost surely
$d_{\mathrm{TV}}(P_t, P_{t,k}^{\mathrm{swap}}) \le \phi(r^\circ_{k,t})$, where $r^\circ_{k,t} := \rho(z^\circ_{k,t-1}, z^\circ_{N+1,t-1})$.
\end{enumerate}
\end{assumption}

Theorem~\ref{thm:spatial_oracle} extends to this setting, with the estimation remainder measured in the metric $\rho$.
\begin{theorem}\label{thm:spatial_oracle_general}
Fix $t \ge 2$ and $\delta \in (0,1)$, and suppose Assumption~\ref{ass:spatial_profiles_general} holds. Let $K : [0,\infty) \to (0,1]$ with $K(0) = 1$, and let the round-$t$ weights be $w^{(t)}_{N+1,k} \propto K(\widehat r_{k,t})$ for $k \le N$ with unnormalized sentinel weight one, where $\widehat r_{k,t} := \rho(\widehat z_{k,t-1}, \widehat z_{N+1,t-1})$. Define the oracle weights
\[
\omega^\circ_{k,t} := \frac{K(r^\circ_{k,t})}{1 + \sum_{j \le N} K(r^\circ_{j,t})}, \qquad r_{\max,t}^{\circ} := \max_{k \le N} r^\circ_{k,t},
\]
and let $L_{t,\delta}$ be a finite $\mathcal F_{t-1}^{+}$-measurable Lipschitz constant of $\log K$ on $[0,\, r_{\max,t}^{\circ} + 2\varepsilon_{t,\delta}]$. Then, with probability at least $1-\delta$,
\[
\mathbb P\big(Y_{N+1,t} \notin C_t(X_{N+1,t}) \,\big|\, \mathcal F_{t-1}^{+}\big)
\;\le\; \tilde\alpha_t \;+\;
\underbrace{\sum_{k=1}^{N} \omega^\circ_{k,t}\, \phi(r^\circ_{k,t})}_{\text{oracle profile gap}}
\;+\; \underbrace{L_{t,\delta}\, \varepsilon_{t,\delta}}_{\text{estimation remainder}} .
\]
\end{theorem}

Theorem~\ref{thm:spatial_oracle} is the case of running-mean profiles, the
metric $\rho_{t-1}$ and the Gaussian kernel; its proof is in
Appendix~\ref{app:general-spatial-proof}. The proof of
Theorem~\ref{thm:spatial_oracle_general} uses four lemmas. Let
$\mathcal E_{t,\delta}$ denote the event in
Assumption~\ref{ass:spatial_profiles_general}(i).
On $\mathcal E_{t,\delta}$ the empirical profile distances approximate the oracle ones uniformly over units.
\begin{lemma}\label{lem:distance-approx-abstract}
On $\mathcal E_{t,\delta}$,
\[
\max_{k\le N}|\widehat r_{k,t}-r_{k,t}^{\circ}|
\le2\varepsilon_{t,\delta}.
\]
\end{lemma}

\begin{proof}
For each $k$, the triangle inequality for the metric $\rho$ gives
\[
|\widehat r_{k,t}-r_{k,t}^{\circ}|
\le\rho(\widehat z_{k,t-1},z_{k,t-1}^{\circ})
+\rho(\widehat z_{N+1,t-1},z_{N+1,t-1}^{\circ}).
\]
Both terms are at most $\varepsilon_{t,\delta}$ on $\mathcal E_{t,\delta}$.
\end{proof}

The weight comparison below needs a Lipschitz bound for the softmax map.
\begin{lemma}\label{lem:gibbs-lipschitz-abstract}
For $g\in\mathbb R^m$, define
$\operatorname{softmax}(g)_i=e^{g_i}/\sum_{j=1}^m e^{g_j}$.
For every $g,g'\in\mathbb R^m$,
\[
\|\operatorname{softmax}(g)-\operatorname{softmax}(g')\|_1
\le\|g-g'\|_\infty.
\]
\end{lemma}

\begin{proof}
Write $p=\operatorname{softmax}(g)$. Its Jacobian is
$J(g)=\operatorname{diag}(p)-pp^\top$. For a vector $v$ with
$\|v\|_\infty\le1$, let $V$ equal $v_i$ with probability $p_i$. Then
\[
\|J(g)v\|_1
=\sum_i p_i|v_i-\mathbb E V|
=\mathbb E|V-\mathbb E V|
\le\sqrt{\operatorname{Var}(V)}\le1.
\]
Thus the $\ell_\infty$-to-$\ell_1$ operator norm of the Jacobian is at most
one. Integrating along the segment joining $g$ and $g'$ proves the result.
\end{proof}

Because two weight vectors both sum to one, comparing them against a $[0,1]$-valued score whose sentinel coordinate is zero costs only half their $\ell_1$ distance.
\begin{lemma}\label{lem:centered-discrepancy}
If $u,v\in\mathcal W_{N+1}$ and $c=(c_1,\ldots,c_N,0)$ with
$c_k\in[0,1]$, then
\[
\left|\sum_{k=1}^N(u_k-v_k)c_k\right|
\le\frac12\|u-v\|_1.
\]
\end{lemma}

\begin{proof}
Since $(u-v)^\top\mathbf1=0$,
$(u-v)^\top c=(u-v)^\top(c-\tfrac12\mathbf1)$.
Every coordinate of the centered vector belongs to $[-1/2,1/2]$;
H\"older's inequality gives the stated bound.
\end{proof}

The empirical weight vector is close in $\ell_1$ to its oracle counterpart.
\begin{lemma}\label{lem:oracle-weight-approx-abstract}
Under the conditions of Theorem~\ref{thm:spatial_oracle_general}, define the full
oracle vector
\[
W^{\circ(t)}:=(\omega_{1,t}^{\circ},\ldots,\omega_{N,t}^{\circ},
1-\textstyle\sum_{k=1}^N\omega_{k,t}^{\circ}).
\]
On $\mathcal E_{t,\delta}$,
\[
\|W^{(t)}-W^{\circ(t)}\|_1\le2L_{t,\delta}\varepsilon_{t,\delta}.
\]
\end{lemma}

\begin{proof}
The full logit vectors are
\[
\widehat g_t:=(\log K(\widehat r_{1,t}),\ldots,
\log K(\widehat r_{N,t}),0),
\qquad
g_t^{\circ}:=(\log K(r_{1,t}^{\circ}),\ldots,
\log K(r_{N,t}^{\circ}),0).
\]
Their softmaxes are $W^{(t)}$ and $W^{\circ(t)}$, respectively. On
$\mathcal E_{t,\delta}$ all distances lie in
$[0,r_{\max,t}^{\circ}+2\varepsilon_{t,\delta}]$, and
Lemma~\ref{lem:distance-approx-abstract} and the Lipschitz assumption imply
\[
\|\widehat g_t-g_t^{\circ}\|_\infty
\le2L_{t,\delta}\varepsilon_{t,\delta}.
\]
Apply Lemma~\ref{lem:gibbs-lipschitz-abstract}.
\end{proof}

\begin{proof}[Proof of Theorem~\ref{thm:spatial_oracle_general}]
Fix $t\ge2$ and $\delta\in(0,1)$ and work on $\mathcal E_{t,\delta}$.
The weights are $\mathcal F_{t-1}^{+}$-measurable, and $K\le1$ ensures that
each unnormalized peer weight is at most the unnormalized sentinel weight, so
the proof of Theorem~\ref{thm:spatial_stepwise} applies verbatim and gives
\begin{equation}\label{eq:budget-swap-general}
\mathbb P\!\left(Y_{N+1,t}\notin C_t(X_{N+1,t})
\,\middle|\,\mathcal F_{t-1}^{+}\right)
\le\tilde\alpha_t+
\sum_{k=1}^Nw_{N+1,k}^{(t)}
d_{\mathrm{TV}}(P_t,P_{t,k}^{\mathrm{swap}}).
\end{equation}
Apply the TV-profile condition and set
$c=(\phi(r_{1,t}^{\circ}),\ldots,\phi(r_{N,t}^{\circ}),0)$.
Lemmas~\ref{lem:centered-discrepancy}
and~\ref{lem:oracle-weight-approx-abstract} yield
\begin{align*}
\sum_{k=1}^Nw_{N+1,k}^{(t)}d_{\mathrm{TV}}(P_t,P_{t,k}^{\mathrm{swap}})
&\le\sum_{k=1}^Nw_{N+1,k}^{(t)}\phi(r_{k,t}^{\circ})\\
&\le\sum_{k=1}^N\omega_{k,t}^{\circ}\phi(r_{k,t}^{\circ})
+\tfrac12\|W^{(t)}-W^{\circ(t)}\|_1\\
&\le\sum_{k=1}^N\omega_{k,t}^{\circ}\phi(r_{k,t}^{\circ})
+L_{t,\delta}\varepsilon_{t,\delta}.
\end{align*}
Substitute into~\eqref{eq:budget-swap-general} and use $\mathbb P(\mathcal E_{t,\delta})\ge1-\delta$.
The argument uses $\rho$ only through the triangle inequality in
Lemma~\ref{lem:distance-approx-abstract} and the
$\mathcal F_{t-1}^{+}$-measurability of the profiles, oracle quantities,
$\rho$ and $\phi$, so it applies when $\rho=\rho_{t-1}$ and $\phi=\phi_t$.
\end{proof}

\subsubsection{Proof of Theorem~\ref{thm:spatial_oracle} and kernel constants}\label{app:general-spatial-proof}

\begin{proof}[Proof of Theorem~\ref{thm:spatial_oracle}]
Fix $t\ge2$ and $\delta\in(0,1)$. If $\varepsilon_{t,\delta}=+\infty$, the
estimation remainder is $+\infty$ and the bound holds trivially, so suppose
$\varepsilon_{t,\delta}<\infty$.
Specialize Assumption~\ref{ass:spatial_profiles_general} to $\mathcal U=\mathbb R^d$,
$\rho=\rho_{t-1}$, $\widehat z_{k,t-1}=\widehat\mu_{k,t-1}$,
$z_{k,t-1}^{\circ}=\mu_k$, and $K(r)=e^{-r^2/(2h^2)}$, which are the profiles
and weights of Algorithm~\ref{alg:wtqa}.
For this pair $(t,\delta)$, Assumption~\ref{ass:spatial_profiles} gives
part~(i) of Assumption~\ref{ass:spatial_profiles_general} with the same, now
finite, tolerance and part~(ii) with $\phi=\phi_t$, and
the oracle quantities satisfy $\omega_{k,t}^{\circ}=\omega_k^{\circ}$
and $r_{\max,t}^{\circ}=D_\mu$.
The Gaussian kernel maps $[0,\infty)$ into $(0,1]$ with $K(0)=1$, so the
proof of Theorem~\ref{thm:spatial_oracle_general} applies; only its bound
$\|\widehat g_t-g_t^{\circ}\|_\infty\le2L_{t,\delta}\varepsilon_{t,\delta}$ in the
proof of Lemma~\ref{lem:oracle-weight-approx-abstract} is sharpened. On
$\mathcal E_{t,\delta}$, Lemma~\ref{lem:distance-approx-abstract} gives
$|\widehat r_{k,t}-r_{k,t}^{\circ}|\le2\varepsilon_{t,\delta}$ and
$\widehat r_{k,t}+r_{k,t}^{\circ}\le2D_\mu+2\varepsilon_{t,\delta}$, so
\begin{equation}\label{eq:gaussian-logit-bound}
\begin{aligned}
|\log K(\widehat r_{k,t})-\log K(r_{k,t}^{\circ})|
&=\frac{|\widehat r_{k,t}-r_{k,t}^{\circ}|\,(\widehat r_{k,t}+r_{k,t}^{\circ})}{2h^2}\\
&\le\frac{2\varepsilon_{t,\delta}(2D_\mu+2\varepsilon_{t,\delta})}{2h^2}
=\frac{2(D_\mu\varepsilon_{t,\delta}+\varepsilon_{t,\delta}^2)}{h^2}.
\end{aligned}
\end{equation}
By~\eqref{eq:gaussian-logit-bound} and
Lemma~\ref{lem:gibbs-lipschitz-abstract},
$\tfrac12\|W^{(t)}-W^{\circ(t)}\|_1\le(D_\mu\varepsilon_{t,\delta}+\varepsilon_{t,\delta}^2)/h^2$.
Using this bound in place of $L_{t,\delta}\varepsilon_{t,\delta}$ in the last
step of the proof of Theorem~\ref{thm:spatial_oracle_general}, which holds on
$\mathcal E_{t,\delta}$ and so with probability at least $1-\delta$, gives the
claimed bound.
\end{proof}

For a scale $h>0$, Theorem~\ref{thm:spatial_oracle_general} applies with the
following constants.
\[
\begin{array}{c|c|c}
K(r)&L_{t,\delta}&\text{estimation remainder}\\[2pt]\hline
\exp(-r^2/(2h^2))
&(r_{\max,t}^{\circ}+2\varepsilon_{t,\delta})/h^2
&(r_{\max,t}^{\circ}\varepsilon_{t,\delta}+\varepsilon_{t,\delta}^2)/h^2\\[3pt]
\exp(-r/h)&1/h&\varepsilon_{t,\delta}/h\\[3pt]
(1+r^2/h^2)^{-1/2}&1/(2h)&\varepsilon_{t,\delta}/(2h)
\end{array}
\]
The Gaussian remainder is the bound~\eqref{eq:gaussian-logit-bound} divided by
two, which is smaller than $L_{t,\delta}\varepsilon_{t,\delta}$ by
$\varepsilon_{t,\delta}^2/h^2$. For the Laplace kernel, $(\log K)'(r)=-1/h$.
For the inverse-multiquadric kernel,
\[
|({\log K})'(r)|=\frac{r}{h^2+r^2}\le\frac1{2h}.
\]

\subsubsection{The tolerance \texorpdfstring{$\varepsilon_{t,\delta}$}{eps} for running means}\label{app:eps-concrete}

This subsection verifies the profile-learnability part of
Assumption~\ref{ass:spatial_profiles} for the running means
$\widehat\mu_{k,t-1}=\frac{1}{t-1}\sum_{s=1}^{t-1}X_{k,s}$, $t\ge2$, under
explicit conditions on the features and on the spread of the latent profiles.
Write $n:=t-1$. Suppose that for each unit
$k\le N+1$ there is $\mu_k\in\mathbb R^d$ and a constant $\sigma_X>0$ such that,
for every $s\ge1$, every $u\in\mathbb S^{d-1}$ and every $\lambda\in\mathbb R$,
\begin{equation}\label{eq:cond-subgaussian}
\mathbb E\bigl[\exp\{\lambda u^\top(X_{k,s}-\mu_k)\}\,\big|\,\mathcal F_{s-1}\bigr]
\le\exp\{\lambda^2\sigma_X^2/2\}
\qquad\text{a.s.}
\end{equation}
Condition~\eqref{eq:cond-subgaussian} holds when each centered feature vector
$X_{k,s}-\mu_k$ is independent of $\mathcal F_{s-1}$ and vector sub-Gaussian
with proxy variance $\sigma_X^2$, for instance when the panel features are
independent across time; it also allows dependence across units and, through
the conditional spread, across time, since it constrains only the conditional
law of $X_{k,s}$ given the past, whose mean must equal $\mu_k$.

Fix $k$ and $u\in\mathbb S^{d-1}$. Since $X_{k,s}$ is $\mathcal F_s$-measurable,
conditioning successively on $\mathcal F_{n-1},\ldots,\mathcal F_0$
in~\eqref{eq:cond-subgaussian} gives, for every $\lambda\in\mathbb R$,
\[
\mathbb E\exp\{\lambda u^\top(\widehat\mu_{k,t-1}-\mu_k)\}
=
\mathbb E\exp\Bigl\{\frac{\lambda}{n}\sum_{s=1}^{n}u^\top(X_{k,s}-\mu_k)\Bigr\}
\le
\exp\Bigl\{\frac{\lambda^2\sigma_X^2}{2n}\Bigr\},
\]
so $u^\top(\widehat\mu_{k,t-1}-\mu_k)$ is sub-Gaussian with proxy variance
$\sigma_X^2/n$ and, for every $r>0$,
\begin{equation}\label{eq:scalar-tail}
\mathbb P\!\left(
\left|u^\top(\widehat\mu_{k,t-1}-\mu_k)\right|\ge r
\right)
\le
2\exp\!\left(-\frac{nr^2}{2\sigma_X^2}\right).
\end{equation}
Applying~\eqref{eq:scalar-tail} with $u=e_j$ to every unit $k\le N+1$ and
every coordinate $j\le d$, and taking a union bound, give
\begin{equation}\label{eq:coordinate-union}
\mathbb P\!\left(
\max_{k\le N+1,\,j\le d}
|\widehat\mu_{k,t-1,j}-\mu_{k,j}|
\ge r
\right)
\le
2(N+1)d
\exp\!\left(-\frac{nr^2}{2\sigma_X^2}\right).
\end{equation}
The right-hand side of~\eqref{eq:coordinate-union} equals $\delta$ at
\begin{equation}\label{eq:eps-coordinate}
\varepsilon^{\infty}_{t,\delta}
:=
\sigma_X\sqrt{\frac{2\log(2(N+1)d/\delta)}{t-1}},
\end{equation}
so the event
\[
\mathcal E^{\infty}_{t,\delta}
:=
\Bigl\{\max_{k\le N+1,\,j\le d}|\widehat\mu_{k,t-1,j}-\mu_{k,j}|
<\varepsilon^{\infty}_{t,\delta}\Bigr\}
\]
has probability at least $1-\delta$.

The metric $\rho_{t-1}$ divides the $j$-th coordinate by $s_{t-1,j}$, so
converting~\eqref{eq:eps-coordinate} into a tolerance in $\rho_{t-1}$ requires
$s_{t-1,j}$ to stay away from zero. For $a\in\mathbb R^{N+1}$ let
$\operatorname{sd}(a):=\|Pa\|_2/\sqrt{N+1}$, where
$P:=I-\mathbf 1\mathbf 1^\top/(N+1)$ is the centering projection; then
$s_{t-1,j}=\operatorname{sd}(\widehat\mu_{1,t-1,j},\ldots,\widehat\mu_{N+1,t-1,j})$
whenever the right-hand side is positive. Since $\|P\|_{\mathrm{op}}=1$, for
all $a,b\in\mathbb R^{N+1}$,
\begin{equation}\label{eq:sd-lipschitz}
|\operatorname{sd}(a)-\operatorname{sd}(b)|
\le\frac{\|P(a-b)\|_2}{\sqrt{N+1}}
\le\frac{\|a-b\|_2}{\sqrt{N+1}}
\le\|a-b\|_\infty .
\end{equation}
Let $s_j^\circ:=\operatorname{sd}(\mu_{1,j},\ldots,\mu_{N+1,j})$ be the
cross-sectional standard deviation of the latent profiles in coordinate $j$,
and let $J_+:=\{j\le d:s_j^\circ>0\}$. Suppose that every coordinate
$j\notin J_+$ is a feature common to all units, that is,
$X_{1,s,j}=\cdots=X_{N+1,s,j}$ for every $s$, and write $d_0:=d-|J_+|$ and
$s^\circ_{\min}:=\min_{j\in J_+}s_j^\circ$. Suppose also that
\[
t-1\ge n_0:=\frac{8\sigma_X^2\log(2(N+1)d/\delta)}{(s^\circ_{\min})^2},
\]
so that $\varepsilon^{\infty}_{t,\delta}\le s^\circ_{\min}/2$
by~\eqref{eq:eps-coordinate}. On $\mathcal E^{\infty}_{t,\delta}$,
\eqref{eq:sd-lipschitz} gives, for $j\in J_+$,
\[
\operatorname{sd}(\widehat\mu_{1,t-1,j},\ldots,\widehat\mu_{N+1,t-1,j})
\ge s_j^\circ-\varepsilon^{\infty}_{t,\delta}\ge s_j^\circ/2,
\]
hence $s_{t-1,j}\ge s_j^\circ/2$; for $j\notin J_+$ the running means coincide
across units, so $s_{t-1,j}=1$. Therefore, for every $k\le N+1$,
\begin{equation}\label{eq:rho-tolerance}
\begin{aligned}
\rho_{t-1}(\widehat\mu_{k,t-1},\mu_k)^2
&=\frac1d\sum_{j\le d}\frac{(\widehat\mu_{k,t-1,j}-\mu_{k,j})^2}{s_{t-1,j}^2}\\
&\le\frac{(\varepsilon^{\infty}_{t,\delta})^2}{d}
\Bigl(\sum_{j\in J_+}\frac{4}{(s_j^\circ)^2}+d_0\Bigr)
\le4H^2(\varepsilon^{\infty}_{t,\delta})^2,
\end{aligned}
\end{equation}
where $H^2:=d^{-1}\bigl(\sum_{j\in J_+}(s_j^\circ)^{-2}+d_0\bigr)$. The
learnability condition of Assumption~\ref{ass:spatial_profiles} therefore holds
with the deterministic tolerance
\begin{equation}\label{eq:eps-rho}
\varepsilon_{t,\delta}
=2H\varepsilon^{\infty}_{t,\delta}
=2H\sigma_X\sqrt{\frac{2\log(2(N+1)d/\delta)}{t-1}}
\qquad(t-1\ge n_0),
\end{equation}
and with $\varepsilon_{t,\delta}=+\infty$ for $t-1<n_0$. The factor $H$ is the
root-mean-square inverse spread of the latent profiles, so $H\sigma_X$ compares
the noise in the features with how much the latent profiles vary across units.

If~\eqref{eq:cond-subgaussian} is replaced by the condition that, for some
integer $m\ge0$, the vectors $X_{k,s}-\mu_k$ are sub-Gaussian with proxy
variance $\sigma_X^2$ and the process $(X_{k,s})_{s\ge1}$ is $m$-dependent,
so that $\sigma(X_{k,s}:s\le a)$ and $\sigma(X_{k,s}:s\ge a+m+1)$ are
independent for every $a$, then, for $t-1\ge m+1$,
\eqref{eq:eps-coordinate} holds with $t-1$ replaced by
$\lfloor(t-1)/(m+1)\rfloor$ and $2(N+1)d$ by $2(m+1)(N+1)d$: the running mean
is a convex combination of the $m+1$ averages over the residue classes of
$\{1,\ldots,n\}$ modulo $m+1$, each of these averages is an average of at least
$\lfloor n/(m+1)\rfloor$ independent sub-Gaussian vectors, and the absolute
value of each coordinate of a convex combination is at most the largest
absolute value of that coordinate among the combined vectors, so a union bound
over the $m+1$ classes replaces~\eqref{eq:coordinate-union}. The argument
from~\eqref{eq:sd-lipschitz} to~\eqref{eq:eps-rho} then applies with the same
replacements in $n_0$.

\subsubsection{The TV-profile condition under a factor model}\label{app:homosc_factor}

This subsection verifies the TV-profile part of
Assumption~\ref{ass:spatial_profiles} in a model with an observed common
factor. Write $Z_{j,t}\in\mathbb R^{d_1}$ for the unit-specific features and
$F_t\in\mathbb R^{d_2}$ for the common factor, both observed before the
round-$t$ prediction set is formed. The covariate supplied to the score and to
Algorithm~\ref{alg:wtqa} is the augmented vector
\[
\bar X_{j,t}:=(Z_{j,t},F_t)\in\mathbb R^{d},\qquad d=d_1+d_2,
\]
which plays the role of $X_{j,t}$ in the main text.

\begin{assumption}[Observed-factor homoscedastic model]\label{ass:factor_model_features}
Let
$g_0,\tilde g_0:\mathbb R^{d_1}\to\mathbb R$ and
$g_F,\tilde g_F:\mathbb R^{d_1}\to\mathbb R^{d_2}$.
For each unit $j\le N+1$ and round $t$, the response and the fitted predictor are
\[
Y_{j,t}
=
g_0(Z_{j,t})
+
g_F(Z_{j,t})^\top F_t
+
\varepsilon_{j,t},
\qquad
\hat f(\bar X_{j,t})
=
\tilde g_0(Z_{j,t})
+
\tilde g_F(Z_{j,t})^\top F_t,
\]
and the nonconformity score is the absolute residual
$\hat s_{j,t}=|Y_{j,t}-\hat f(\bar X_{j,t})|$. Moreover:
\begin{enumerate}
    \item there is a finite constant $M_F$ such that
    $\mathbb E[\|F_t\|_2^2\mid\mathcal F_{t-1}^{+}]\le M_F^2$
    almost surely for all $t$;
    \item conditional on $(F_t,\mathcal F_{t-1}^{+})$, the unit-specific features
    $Z_{1,t},\dots,Z_{N+1,t}$ are independent across units with
    $Z_{j,t}\sim\mathcal N(\mu_j,\Sigma_Z)$,
    where $\mu_j\in\mathbb R^{d_1}$ and
    $\Sigma_Z\in\mathbb R^{d_1\times d_1}$ is a common positive semidefinite
    matrix;
    \item conditional on $(\mathcal F_{t-1}^{+},F_t,Z_{1,t},\dots,Z_{N+1,t})$, the
    noises $\varepsilon_{1,t},\dots,\varepsilon_{N+1,t}$ are independent with
    $\varepsilon_{j,t}\sim\mathcal N(0,\sigma_\varepsilon^2)$, where
    $\sigma_\varepsilon>0$;
    \item $\tilde g_0=g_0$, and there exists a matrix
    $B\in\mathbb R^{d_2\times d_1}$ such that
    $g_F(z)-\tilde g_F(z)=Bz$ for every $z\in\mathbb R^{d_1}$.
\end{enumerate}
\end{assumption}

The latent profile of unit $j$ is $\bar\mu_j:=(\mu_j,\nu)\in\mathbb R^d$ with
one vector $\nu\in\mathbb R^{d_2}$ common to all units, so that
$\|\bar\mu_k-\bar\mu_{N+1}\|_2=\|\mu_k-\mu_{N+1}\|_2$.
Under Assumption~\ref{ass:factor_model_features}, the signed residual is
\begin{equation}\label{eq:factor-residual}
U_{j,t}:=Y_{j,t}-\hat f(\bar X_{j,t})
=
F_t^\top BZ_{j,t}+\varepsilon_{j,t},
\qquad
\hat s_{j,t}=|U_{j,t}|.
\end{equation}
The common factor can induce cross-sectional dependence in the score vector,
but conditional on $(F_t,\mathcal F_{t-1}^{+})$ the scores are independent across
units. Conditioning on the factor therefore reduces the swap discrepancy to a comparison of two one-dimensional score laws.
\begin{lemma}\label{lem:conditional_score_tv}
Under Assumption~\ref{ass:factor_model_features}, for every $k\in\{1,\dots,N\}$,
\[
d_{\mathrm{TV}}\!\left(P_t,P_{t,k}^{\mathrm{swap}}\right)
\le
2\,\mathbb E\!\left[
d_{\mathrm{TV}}\!\left(
P_{\hat s_{k,t}\mid F_t,\mathcal F_{t-1}^{+}},
P_{\hat s_{N+1,t}\mid F_t,\mathcal F_{t-1}^{+}}
\right)
\,\middle|\,\mathcal F_{t-1}^{+}
\right].
\]
\end{lemma}

\begin{proof}
Fix $t$ and $k$, and condition on $\mathcal F_{t-1}^{+}$. For a realization $f$ of
$F_t$, let
\[
Q_t^{f}
:=
\mathcal L\!\left((\hat s_{1,t},\dots,\hat s_{N+1,t})\,\middle|\,
F_t=f,\mathcal F_{t-1}^{+}\right),
\]
and let $Q_{t,k}^{\mathrm{swap},f}$ be the corresponding law after swapping the
$k$-th and target score coordinates. By the tower property,
\[
P_t(A)=\mathbb E[Q_t^{F_t}(A)\mid\mathcal F_{t-1}^{+}],
\qquad
P_{t,k}^{\mathrm{swap}}(A)
=
\mathbb E[Q_{t,k}^{\mathrm{swap},F_t}(A)\mid\mathcal F_{t-1}^{+}]
\]
for every measurable $A\subseteq\mathbb R^{N+1}$. Since total variation
distance is jointly convex,
\begin{equation}\label{eq:tv-tower}
d_{\mathrm{TV}}(P_t,P_{t,k}^{\mathrm{swap}})
\le
\mathbb E\!\left[
d_{\mathrm{TV}}(Q_t^{F_t},Q_{t,k}^{\mathrm{swap},F_t})
\mid\mathcal F_{t-1}^{+}
\right].
\end{equation}
Conditional on $(F_t=f,\mathcal F_{t-1}^{+})$, the score coordinates are independent.
Writing
\[
\nu_{j,f}:=\mathcal L(\hat s_{j,t}\mid F_t=f,\mathcal F_{t-1}^{+}),
\qquad
R_f:=\bigotimes_{j\neq k,N+1}\nu_{j,f},
\]
we have
\[
Q_t^f=R_f\otimes\nu_{k,f}\otimes\nu_{N+1,f},
\qquad
Q_{t,k}^{\mathrm{swap},f}=R_f\otimes\nu_{N+1,f}\otimes\nu_{k,f}.
\]
Tensoring with a common probability law preserves total variation distance,
so the triangle inequality through $R_f\otimes\nu_{N+1,f}\otimes\nu_{N+1,f}$ gives
\[
d_{\mathrm{TV}}(Q_t^f,Q_{t,k}^{\mathrm{swap},f})
\le
2\,d_{\mathrm{TV}}(\nu_{k,f},\nu_{N+1,f}).
\]
Substituting this into~\eqref{eq:tv-tower} proves the claim.
\end{proof}

That one-dimensional comparison is in turn controlled by the distance between the two units' latent profiles.
\begin{lemma}\label{lem:homosc_score_tv}
Under Assumption~\ref{ass:factor_model_features}, for every $k\in\{1,\dots,N\}$,
\[
\mathbb E\!\left[
d_{\mathrm{TV}}\!\left(
P_{\hat s_{k,t}\mid F_t,\mathcal F_{t-1}^{+}},
P_{\hat s_{N+1,t}\mid F_t,\mathcal F_{t-1}^{+}}
\right)
\,\middle|\,\mathcal F_{t-1}^{+}
\right]
\le
C_{\mathrm{mean}}\|\mu_k-\mu_{N+1}\|_2,
\]
where $C_{\mathrm{mean}}:=\|B\|_{\mathrm{op}}M_F/(2\sigma_\varepsilon)$.
\end{lemma}

\begin{proof}
Fix $k$ and condition on $(F_t=f,\mathcal F_{t-1}^{+})$. By~\eqref{eq:factor-residual},
\[
U_{j,t}\mid(F_t=f,\mathcal F_{t-1}^{+})
\sim
\mathcal N\!\left(f^\top B\mu_j,\,
f^\top B\Sigma_ZB^\top f+\sigma_\varepsilon^2\right),
\qquad j\in\{k,N+1\},
\]
with the same variance for both units. For any
$\mu_1,\mu_2\in\mathbb R$ and $\sigma_1,\sigma_2>0$,
\citet[Theorem~1.3]{devroye2018total} give
\begin{equation}\label{eq:gaussian-tv}
d_{\mathrm{TV}}\!\left(
\mathcal N(\mu_1,\sigma_1^2),
\mathcal N(\mu_2,\sigma_2^2)
\right)
\le
\frac{3|\sigma_1^2-\sigma_2^2|}{2\sigma_1^2}
+
\frac{|\mu_1-\mu_2|}{2\sigma_1}.
\end{equation}
The absolute-value map cannot increase total variation distance, so
\eqref{eq:gaussian-tv}, whose first term vanishes because the two conditional
variances are equal, gives
\begin{align*}
&d_{\mathrm{TV}}\!\left(
P_{\hat s_{k,t}\mid F_t=f,\mathcal F_{t-1}^{+}},
P_{\hat s_{N+1,t}\mid F_t=f,\mathcal F_{t-1}^{+}}
\right) \\
&\le
\frac{|f^\top B(\mu_k-\mu_{N+1})|}
{2\sqrt{f^\top B\Sigma_ZB^\top f+\sigma_\varepsilon^2}}
\le
\frac{\|B\|_{\mathrm{op}}\|f\|_2}{2\sigma_\varepsilon}
\|\mu_k-\mu_{N+1}\|_2 .
\end{align*}
Taking conditional expectation and using
\[
\mathbb E[\|F_t\|_2\mid\mathcal F_{t-1}^{+}]
\le
\bigl(\mathbb E[\|F_t\|_2^2\mid\mathcal F_{t-1}^{+}]\bigr)^{1/2}
\le M_F
\]
gives the result.
\end{proof}

Lemmas~\ref{lem:conditional_score_tv} and~\ref{lem:homosc_score_tv} combine into the
TV-profile condition with an explicit $\phi$.
\begin{proposition}\label{prop:homosc-factor-implies-a44}
Under Assumption~\ref{ass:factor_model_features}, for every $k\in\{1,\dots,N\}$,
\[
d_{\mathrm{TV}}(P_t,P_{t,k}^{\mathrm{swap}})\le\phi(\|\mu_k-\mu_{N+1}\|_2),
\qquad
\phi(r)=\min\{1,\,2C_{\mathrm{mean}}r\}.
\]
\end{proposition}

\begin{proof}
Combining Lemmas~\ref{lem:conditional_score_tv} and~\ref{lem:homosc_score_tv} gives
$d_{\mathrm{TV}}(P_t,P_{t,k}^{\mathrm{swap}})\le2C_{\mathrm{mean}}\|\mu_k-\mu_{N+1}\|_2$,
and total variation distance is at most $1$.
\end{proof}

The factor coordinates of $\bar\mu_k-\bar\mu_{N+1}$ vanish, so
$\|\mu_k-\mu_{N+1}\|_2=\|\bar\mu_k-\bar\mu_{N+1}\|_2\le\sqrt d\,s_{\max,t-1}\,\rho_{t-1}(\bar\mu_k,\bar\mu_{N+1})$ with
$s_{\max,t-1}:=\max_{j\le d_1}s_{t-1,j}$. Hence the TV-profile condition of
Assumption~\ref{ass:spatial_profiles} holds for the profiles $\bar\mu_j$ with
respect to $\rho_{t-1}$ with
$\phi_t(r)=\min\{1,\,2C_{\mathrm{mean}}\sqrt d\,s_{\max,t-1}\,r\}$.
Under the conditions of Appendix~\ref{app:eps-concrete} with
$J_+=\{1,\ldots,d_1\}$ and $t-1\ge n_0$, on $\mathcal E^{\infty}_{t,\delta}$
the bound~\eqref{eq:sd-lipschitz} gives
$s_{t-1,j}\le s_j^\circ+\varepsilon^{\infty}_{t,\delta}\le\tfrac32s_j^\circ$ for
$j\le d_1$, so
$\phi_t(r)\le\min\{1,\,3C_{\mathrm{mean}}\sqrt d\,s^\circ_{\max}\,r\}$ with
$s^\circ_{\max}:=\max_{j\le d_1}s_j^\circ$, a deterministic function.

\subsubsection{Second-order extension beyond homoscedastic features}\label{app:second-order-extension}

If condition~2 of Assumption~\ref{ass:factor_model_features} is replaced by
$Z_{j,t}\mid(F_t,\mathcal F_{t-1}^{+})\sim\mathcal N(\mu_j,\Sigma_{Z,j})$ with
unit-specific covariances, the $Z_{j,t}$ remaining independent across units given
$(F_t,\mathcal F_{t-1}^{+})$, the conditional variances in the proof of
Lemma~\ref{lem:homosc_score_tv} become
$f^\top B\Sigma_{Z,j}B^\top f+\sigma_\varepsilon^2$, and the first term
of~\eqref{eq:gaussian-tv} no longer vanishes. Since
$|f^\top B(\Sigma_{Z,k}-\Sigma_{Z,N+1})B^\top f|\le\|f\|_2^2\|B\|_{\mathrm{op}}^2\|\Sigma_{Z,k}-\Sigma_{Z,N+1}\|_{\mathrm{op}}$
and both variances are at least $\sigma_\varepsilon^2$, the same argument gives
\begin{equation}\label{eq:second-order-score-tv}
\mathbb E\!\left[
d_{\mathrm{TV}}\!\left(
P_{\hat s_{k,t}\mid F_t,\mathcal F_{t-1}^{+}},
P_{\hat s_{N+1,t}\mid F_t,\mathcal F_{t-1}^{+}}
\right)
\,\middle|\,\mathcal F_{t-1}^{+}
\right]
\le
C_{\mathrm{cov}}\|\Sigma_{Z,k}-\Sigma_{Z,N+1}\|_{\mathrm{op}}
+C_{\mathrm{mean}}\|\mu_k-\mu_{N+1}\|_2 ,
\end{equation}
with $C_{\mathrm{cov}}:=3\|B\|_{\mathrm{op}}^2 M_F^2/(2\sigma_\varepsilon^2)$.
The bound~\eqref{eq:second-order-score-tv} yields part~(ii) of
Assumption~\ref{ass:spatial_profiles_general} for profiles that include the
covariance. Fix $c>0$, let
$\mathcal U:=\mathbb R^{d_1}\times\mathbb R^{d_1\times d_1}$ carry the metric
$\rho((\mu,\Sigma),(\mu',\Sigma')):=\|\mu-\mu'\|_2+c\|\Sigma-\Sigma'\|_{\mathrm{op}}$,
let $z^\circ_{j,t-1}:=(\mu_j,\Sigma_{Z,j})$ for $j\le N+1$, and let
$\phi(r):=\min\{1,\,2\max\{C_{\mathrm{mean}},C_{\mathrm{cov}}/c\}\,r\}$.
The proof of Lemma~\ref{lem:conditional_score_tv} uses only the independence
of the scores across units given $(F_t,\mathcal F_{t-1}^{+})$, which still
holds. That lemma, \eqref{eq:second-order-score-tv},
$C_{\mathrm{cov}}\|\Sigma-\Sigma'\|_{\mathrm{op}}=(C_{\mathrm{cov}}/c)\,c\|\Sigma-\Sigma'\|_{\mathrm{op}}$
and $d_{\mathrm{TV}}\le1$ give, for every $k\le N$,
\begin{equation}\label{eq:second-order-swap-tv}
\begin{aligned}
d_{\mathrm{TV}}(P_t,P_{t,k}^{\mathrm{swap}})
&\le\min\bigl\{1,\,2C_{\mathrm{cov}}\|\Sigma_{Z,k}-\Sigma_{Z,N+1}\|_{\mathrm{op}}
+2C_{\mathrm{mean}}\|\mu_k-\mu_{N+1}\|_2\bigr\}\\
&\le\phi\bigl(\rho(z^\circ_{k,t-1},z^\circ_{N+1,t-1})\bigr).
\end{aligned}
\end{equation}
Theorem~\ref{thm:spatial_oracle_general} therefore applies to a
covariance-aware version of W-TQA that profiles each unit by the running mean
and covariance of its features $Z_{j,s}$, whenever these profiles satisfy
part~(i) of Assumption~\ref{ass:spatial_profiles_general};
Algorithm~\ref{alg:wtqa} uses the first-moment profile only.

\subsection{Long-run average coverage}\label{app:temporal-section}

\subsubsection{Proof of Proposition~\ref{prop:temporal-observed}}\label{app:temporal-proofs}

\begin{proof}[Proof of Proposition~\ref{prop:temporal-observed}]
\emph{The range of the level.} We show $\alpha_s\le1+\gamma\alpha$ for every
$s\ge0$ by induction; the base case is $\alpha_0=\alpha_1=\alpha<1$. Suppose
$\alpha_s\le1+\gamma\alpha$ for some $s\ge1$. If $R_s=0$, then
$\alpha_{s+1}=\alpha_s$. If $R_s=1$ and $\alpha_s<1$,
then~\eqref{eq:level-recursion} gives
$\alpha_{s+1}\le\alpha_s+\gamma\alpha<1+\gamma\alpha$.
If $R_s=1$ and $\alpha_s\ge1$, then $\tilde\alpha_s\ge\alpha_s\ge1$, so
$\hat q_s=-\infty$ by Lemma~\ref{lem:effective-level}, $C_s=\varnothing$ and
$\ell_s=1$; therefore
$\alpha_{s+1}=\alpha_s-\gamma(1-\alpha)<\alpha_s\le1+\gamma\alpha$.
For $s=T+1$,
\begin{equation}\label{eq:terminal-level-upper}
\alpha-\alpha_{T+1}\ge-(1-\alpha+\gamma\alpha).
\end{equation}

\emph{The identity.} The argument is the telescoping of Proposition~4.1 of
\citet{gibbs2021adaptive}, applied only on rounds where target feedback is
observed. Summing~\eqref{eq:level-recursion} over $t=1,\dots,T$ and using
$\alpha_1=\alpha$,
\[
\alpha_{T+1}
=
\alpha+\gamma\sum_{t\in\mathcal R_T}(\alpha-\ell_t),
\]
which rearranges to
\begin{equation}\label{eq:revealed-identity}
\sum_{t\in\mathcal R_T}(\ell_t-\alpha)
=
\frac{\alpha-\alpha_{T+1}}{\gamma}.
\end{equation}
Dividing by $|\mathcal R_T|$
on $\{\mathcal R_T\ne\varnothing\}$ gives~\eqref{eq:revealed-frequency-identity},
and dividing by the horizon $T$ gives
\begin{equation}
\label{eq:t-normalized}
\frac1T\sum_{t=1}^T R_t(\ell_t-\alpha)
=
\frac{\alpha-\alpha_{T+1}}{T\gamma}.
\end{equation}

\emph{The lower bound.} Inserting~\eqref{eq:terminal-level-upper}
into~\eqref{eq:revealed-frequency-identity} gives~\eqref{eq:revealed-lower-bound}.
\end{proof}

\subsubsection{Proof of Proposition~\ref{prop:temporal_impossibility}}\label{app:temporal-impossibility-proof}

Algorithm~\ref{alg:wtqa} with the absolute residual score satisfies the
hypothesis of Proposition~\ref{prop:temporal_impossibility} at every round
$t<(1-\alpha)/(\gamma\alpha)$. Each revealed round raises the level by at most
$\gamma\alpha$ by~\eqref{eq:level-recursion}, so
$\alpha_t\le\alpha+\gamma\alpha t<1$. Then
$\tilde\alpha_t=\max\{\alpha_t,w_{N+1,N+1}^{(t)}\}<1$, so by
Lemma~\ref{lem:effective-level} the threshold $\hat q_t$ is finite and at least
$\min_{k\le N}\hat s_{k,t}\ge0$, and
$C_t=[\hat f(X_{N+1,t})-\hat q_t,\,\hat f(X_{N+1,t})+\hat q_t]$ is a bounded
interval. The points $\hat f(X_{N+1,t})\in C_t$ and
$\hat f(X_{N+1,t})+\hat q_t+1\notin C_t$ are $\mathcal F_t$-measurable.

\begin{proof}[Proof of Proposition~\ref{prop:temporal_impossibility}]
For any joint law, total probability gives
\[
\mathbb P(\ell_t=1)=\pi_t m_t^{\mathrm{obs}}
+(1-\pi_t)\,\mathbb P(\ell_t=1\mid R_t=0),
\]
and the last conditional probability lies in $[0,1]$; when $\pi_t=1$ the
second term is zero. This proves~\eqref{eq:identification-interval}, which is
the selection bound of \citet{manski1989} applied at round $t$.

For the converse, fix the procedure; if it uses internal randomization, its
seed is included in $\mathcal F_t$ throughout this proof. Let
$y_t^{\mathrm{in}}\in C_t$ and $y_t^{\mathrm{out}}\notin C_t$ be the
$\mathcal F_t$-measurable points of the hypothesis. Fix any law of
$(\mathcal O_1,\ldots,\mathcal O_t)$. For $\theta\in[0,1]$, let $\mathbb P_\theta$
be the law under which, conditionally on $\mathcal F_t$, the pair
$(R_t,Y_{N+1,t})$ is drawn from
\[
\begin{array}{c|cc}
\text{probability}&R_t&Y_{N+1,t}\\ \hline
\pi_t(1-m_t^{\mathrm{obs}})&1&y_t^{\mathrm{in}}\\
\pi_t m_t^{\mathrm{obs}}&1&y_t^{\mathrm{out}}\\
(1-\pi_t)(1-\theta)&0&y_t^{\mathrm{in}}\\
(1-\pi_t)\theta&0&y_t^{\mathrm{out}}
\end{array}
\]
and the remaining observations are drawn as follows: the calibration slice
and target covariate of $\mathcal O_{t+1}$ from one fixed conditional law given
$(\mathcal F_t,R_t,R_tY_{N+1,t})$, and the batches $\mathcal O_s$, $s\ge t+2$,
from one fixed conditional law given $(\mathcal O_1,\ldots,\mathcal O_{t+1})$. The conditional law of
$\mathcal O_{t+1}$ given $\mathcal F_t$ does not depend on $\theta$: the
conditional law of $(R_t,R_tY_{N+1,t})$ is free of $\theta$, because the
probabilities of the two revealed branches are free of $\theta$ and
$R_t=0$, $R_tY_{N+1,t}=0$ on the two hidden branches, and the rest of
$\mathcal O_{t+1}$ is drawn given $(\mathcal F_t,R_t,R_tY_{N+1,t})$ from a law
free of $\theta$. Hence
the observed data $(\mathcal O_s)_{s\ge1}$ have the same law under every
$\mathbb P_\theta$, while
\[
\mathbb P_\theta(R_t=1)=\pi_t,
\qquad
\mathbb P_\theta(\ell_t=1\mid R_t=1)=m_t^{\mathrm{obs}},
\qquad
\mathbb P_\theta(\ell_t=1)=\pi_t m_t^{\mathrm{obs}}+(1-\pi_t)\theta,
\]
which takes every value in~\eqref{eq:identification-interval} as $\theta$
ranges over $[0,1]$.
When the hypothesis holds at every round $s\le T$, as it does for
Algorithm~\ref{alg:wtqa} with the absolute residual score if
$T<(1-\alpha)/(\gamma\alpha)$, the same construction applied at every round
$s\le T$ gives a family indexed by $(\theta_s)_{s\le T}\in[0,1]^T$ with one
observed-data law; Appendix~\ref{app:selection_bias} uses this family.
\end{proof}

When $m_t^{\mathrm{obs}}\in(0,1)$, the reveal propensities of the family
$(\mathbb P_\theta)_{\theta\in[0,1]}$ are bounded away from zero uniformly in
$\theta$, so positivity of the reveal propensity alone does not identify
$\mathbb P(\ell_t=1)$. Since
$y_t^{\mathrm{in}}\ne y_t^{\mathrm{out}}$, Bayes' rule gives, for
$p_t:=\mathbb P_\theta(R_t=1\mid\mathcal H_t)$,
\[
p_t=\begin{cases}
\displaystyle\frac{\pi_t(1-m_t^{\mathrm{obs}})}
{\pi_t(1-m_t^{\mathrm{obs}})+(1-\pi_t)(1-\theta)},&\text{on }\{Y_{N+1,t}=y_t^{\mathrm{in}}\},\\[3mm]
\displaystyle\frac{\pi_t m_t^{\mathrm{obs}}}
{\pi_t m_t^{\mathrm{obs}}+(1-\pi_t)\theta},&\text{on }\{Y_{N+1,t}=y_t^{\mathrm{out}}\},
\end{cases}
\]
with positive denominators. Uniformly over $\theta\in[0,1]$,
\[
p_t\ge c_t:=\min\left\{
\frac{\pi_t(1-m_t^{\mathrm{obs}})}{1-\pi_t m_t^{\mathrm{obs}}},
\frac{\pi_t m_t^{\mathrm{obs}}}{\pi_t m_t^{\mathrm{obs}}+1-\pi_t}
\right\}>0 .
\]

\subsubsection{Feasibility of the fallback set and stability of the level}\label{app:feasibility-proof}

The first lemma bounds the left side of~\eqref{eq:fallback-feasibility} by the
swap discrepancies of Section~\ref{subsec:spatial_guarantees}.
\begin{lemma}\label{lem:feasibility-from-swap}
Fix a round $t$ and let $S\subseteq\{1,\dots,N\}$ be a nonempty
$\mathcal F_{t-1}^{+}$-measurable random set. Then, almost surely,
\[
\mathbb P\Bigl(\hat s_{N+1,t}>\max_{k\le N}\hat s_{k,t}\,\Big|\,\mathcal F_{t-1}^{+}\Bigr)
\le\frac{1+\sum_{k\in S}d_{\mathrm{TV}}(P_t,P_{t,k}^{\mathrm{swap}})}{|S|+1}.
\]
\end{lemma}

\begin{proof}
Condition on $\mathcal F_{t-1}^{+}$, so that $S$ is fixed and the round-$t$
score vector has law $P_t$. Let
$\mathrm{pr}_S:\mathbb R^{N+1}\to\mathbb R^{|S|+1}$ keep the coordinates in
$S$, in increasing order, followed by coordinate $N+1$, and write $P_t^S$ for
the image of $P_t$ under $\mathrm{pr}_S$. For $k\in S$, exchanging coordinates
$k$ and $N+1$ before or after applying $\mathrm{pr}_S$ gives the same vector, so
the image of $P_{t,k}^{\mathrm{swap}}$ under $\mathrm{pr}_S$ is the law
$P_{t,k}^{S,\mathrm{swap}}$ obtained from $P_t^S$ by exchanging the
coordinate of unit $k$ with the last one. Total variation distance does not
increase under a measurable map, so
\begin{equation}\label{eq:subpanel-tv}
d_{\mathrm{TV}}\bigl(P_t^S,P_{t,k}^{S,\mathrm{swap}}\bigr)
\le d_{\mathrm{TV}}\bigl(P_t,P_{t,k}^{\mathrm{swap}}\bigr),
\qquad k\in S.
\end{equation}
Consider the weighted split conformal procedure with calibration scores
$(\hat s_{k,t})_{k\in S}$, test score $\hat s_{N+1,t}$, unnormalized weight one
on every calibration score and on the test point, and level $1/(|S|+1)$.
Every normalized weight, including the sentinel weight, equals $1/(|S|+1)$.
The proof of Lemma~\ref{lem:effective-level} uses only that every calibration
weight is positive and that the sentinel carries weight $w$; with
$w=1/(|S|+1)$, its case $\beta=w$ shows that the threshold of this procedure
is $\max_{k\in S}\hat s_{k,t}$. Theorem~2 of \citet{barber2023conformal},
applied as in the proof of Theorem~\ref{thm:spatial_stepwise} with unnormalized
weights one, which lie in $[0,1]$, and~\eqref{eq:subpanel-tv} give
\begin{equation}\label{eq:subpanel-miss}
\begin{aligned}
\mathbb P\Bigl(\hat s_{N+1,t}>\max_{k\in S}\hat s_{k,t}\,\Big|\,\mathcal F_{t-1}^{+}\Bigr)
&\le\frac1{|S|+1}+\sum_{k\in S}\frac{d_{\mathrm{TV}}(P_t^S,P_{t,k}^{S,\mathrm{swap}})}{|S|+1}\\
&\le\frac{1+\sum_{k\in S}d_{\mathrm{TV}}(P_t,P_{t,k}^{\mathrm{swap}})}{|S|+1}.
\end{aligned}
\end{equation}
Since $\max_{k\le N}\hat s_{k,t}\ge\max_{k\in S}\hat s_{k,t}$, the event
of~\eqref{eq:fallback-feasibility} is contained in the event
of~\eqref{eq:subpanel-miss}.
\end{proof}

The second lemma shows that, when~\eqref{eq:fallback-feasibility} holds on the
rounds with a negative level, the level has an exponential lower tail and a
lower bound in expectation. It uses~\eqref{eq:fallback-feasibility} only on
those rounds.
\begin{lemma}\label{lem:level-stability}
Suppose Assumption~\ref{ass:mcar} holds and that there is $\zeta\in(0,\alpha]$
such that, almost surely for every $t$, \eqref{eq:fallback-feasibility} holds
on the event $\{\alpha_t<0\}$. Then, for every $t\ge1$ and $x\ge0$,
\begin{equation}\label{eq:level-tail}
\mathbb P\bigl(\alpha_t<-\gamma(1-\alpha)-\gamma x\bigr)\le e^{-8\zeta x},
\qquad
\mathbb E\alpha_t\ge-\gamma(1-\alpha)-\frac{\gamma}{8\zeta}.
\end{equation}
\end{lemma}

\begin{proof}
\emph{Reduction to a reflected process.} Set
$\Gamma_t:=(-\alpha_t/\gamma-(1-\alpha))_+$ and
$\psi_t:=R_t(\ell_t-\alpha)\in[-\alpha,1-\alpha]$, so that
\begin{equation}\label{eq:level-lower-by-excess}
\alpha_t\ge-\gamma(1-\alpha)-\gamma\Gamma_t ,
\end{equation}
$\Gamma_1=0$ because $\alpha_1=\alpha>0$, and~\eqref{eq:level-recursion} reads
$-\alpha_{t+1}/\gamma=-\alpha_t/\gamma+\psi_t$. If $\alpha_t\ge0$, then
$-\alpha_{t+1}/\gamma\le\psi_t\le1-\alpha$ and $\Gamma_{t+1}=0$. If
$\alpha_t<0$, then
$-\alpha_{t+1}/\gamma-(1-\alpha)=(-\alpha_t/\gamma-(1-\alpha))+\psi_t\le
\Gamma_t+\psi_t$. Hence, for every $t\ge1$,
\begin{equation}\label{eq:excess-recursion}
\Gamma_{t+1}\le\mathbf 1\{\alpha_t<0\}\,(\Gamma_t+\psi_t)_+ .
\end{equation}

\emph{The miss probability on rounds with a negative level.} On
$\{\alpha_t<0\}$ we have $\alpha_t<w_{N+1,N+1}^{(t)}$, so round $t$ is a
fallback round, $\hat q_t=\max_{k\le N}\hat s_{k,t}$ by
Lemma~\ref{lem:effective-level}, and
$\ell_t=\mathbf 1\{\hat s_{N+1,t}>\max_{k\le N}\hat s_{k,t}\}$. Write
$\varpi_t:=\mathbb P(\ell_t=1\mid\mathcal F_{t-1}^{+})$. The event
$\{\alpha_t<0\}$ is $\mathcal F_{t-1}^{+}$-measurable, so the hypothesis gives
\begin{equation}\label{eq:negative-level-miss}
\varpi_t\le\alpha-\zeta\qquad\text{on }\{\alpha_t<0\}.
\end{equation}

\emph{Coupling with a reflected random walk.} Let $\chi_1,\chi_2,\ldots$ be
independent Uniform$(0,1)$ variables, independent of all other variables; this
enlargement of the probability space does not change the law of
$(\alpha_t)_{t\ge1}$. Set
$\mathcal J_t:=\mathcal F_{t-1}^{+}\vee\sigma(\chi_s:s<t)$ and
\[
\tau_t:=\begin{cases}
\varpi_t\chi_t, & R_t=1,\ \ell_t=1,\\
\varpi_t+(1-\varpi_t)\chi_t, & R_t=1,\ \ell_t=0,\\
\chi_t, & R_t=0.
\end{cases}
\]
Under Assumption~\ref{ass:mcar}, $R_t$ is independent of $\mathcal H_t$, and
$\ell_t$ and $\mathcal F_{t-1}^{+}$ are $\mathcal H_t$-measurable, so
conditionally on $\mathcal J_t$ the pair $(R_t,R_t\ell_t)$ takes the values
$(1,1)$, $(1,0)$ and $(0,0)$ with probabilities $\pi\varpi_t$,
$\pi(1-\varpi_t)$ and $1-\pi$. Since $\chi_t$ is independent of
$\mathcal J_t$, $R_t$ and $\ell_t$, for every $u\in[0,1]$
\[
\mathbb P(R_t=1,\,\tau_t\le u\mid\mathcal J_t)
=\pi\bigl(\varpi_t\min\{1,u/\varpi_t\}+(u-\varpi_t)_+\bigr)=\pi u,
\qquad
\mathbb P(R_t=0,\,\tau_t\le u\mid\mathcal J_t)=(1-\pi)u,
\]
with $\varpi_t\min\{1,u/\varpi_t\}:=0$ when $\varpi_t=0$. Thus
$(R_t,\tau_t)$ is independent of $\mathcal J_t$ with law
Bernoulli$(\pi)\otimes$Uniform$(0,1)$. The pair $(R_t,\tau_t)$ is
$\mathcal J_{t+1}$-measurable, because $R_t$ and $R_t\ell_t$ are
$\mathcal F_t^{+}$-measurable, $\varpi_t$ is $\mathcal F_{t-1}^{+}$-measurable
and $\chi_t$ is $\sigma(\chi_s:s\le t)$-measurable. Hence $(R_s,\tau_s)$ is
$\mathcal J_t$-measurable for every $s<t$, and since $(R_t,\tau_t)$ is
independent of $\mathcal J_t$ with a law that does not depend on $t$, the pairs
$(R_t,\tau_t)$, $t\ge1$, are independent and identically distributed. On
$\{R_t=1\}$, $\ell_t=\mathbf 1\{\tau_t\le\varpi_t\}$, so on $\{\alpha_t<0\}$~\eqref{eq:negative-level-miss}
gives
\[
\psi_t\le\bar\psi_t:=R_t\bigl(\mathbf 1\{\tau_t\le\alpha-\zeta\}-\alpha\bigr).
\]
Define $\Xi_1:=0$ and $\Xi_{t+1}:=(\Xi_t+\bar\psi_t)_+$. We show
$\Gamma_t\le\Xi_t$ for every $t\ge1$ by induction; the base case is
$\Gamma_1=0=\Xi_1$. Suppose $\Gamma_t\le\Xi_t$. On $\{\alpha_t<0\}$,
\eqref{eq:excess-recursion}, $\psi_t\le\bar\psi_t$ and the monotonicity of
$(z,y)\mapsto(z+y)_+$ in each argument give
$\Gamma_{t+1}\le(\Gamma_t+\psi_t)_+\le(\Xi_t+\bar\psi_t)_+=\Xi_{t+1}$; on
$\{\alpha_t\ge0\}$, \eqref{eq:excess-recursion} gives
$\Gamma_{t+1}=0\le\Xi_{t+1}$.

\emph{The tail of the reflected walk.} For every $t\ge1$,
\begin{equation}\label{eq:lindley-max}
\Xi_t=\max_{1\le j\le t}\sum_{s=j}^{t-1}\bar\psi_s ,
\end{equation}
with the empty sum equal to zero: $\Xi_1=0$, and if~\eqref{eq:lindley-max}
holds for $t$, then
$\Xi_{t+1}=\max\{0,\Xi_t+\bar\psi_t\}
=\max\{0,\max_{1\le j\le t}\sum_{s=j}^{t}\bar\psi_s\}
=\max_{1\le j\le t+1}\sum_{s=j}^{t}\bar\psi_s$.
The vectors $(\bar\psi_{t-1},\ldots,\bar\psi_1)$ and
$(\bar\psi_1,\ldots,\bar\psi_{t-1})$ have the same law, so
by~\eqref{eq:lindley-max} $\Xi_t$ has the law of
$\max_{0\le k\le t-1}\Psi_k$, where $\Psi_k:=\sum_{s=1}^k\bar\psi_s$. Let
$\vartheta:=8\zeta$. The variable $\mathbf 1\{\tau_t\le\alpha-\zeta\}-\alpha$
has mean $-\zeta$ and takes values in an interval of length one, so Hoeffding's
lemma \citep[inequality~(4.16)]{hoeffding1963probability} and the independence
of $R_t$ and $\tau_t$ give
\[
\mathbb E e^{\vartheta\bar\psi_t}
=1-\pi+\pi\,\mathbb E e^{\vartheta(\mathbf 1\{\tau_t\le\alpha-\zeta\}-\alpha)}
\le1-\pi+\pi e^{-\vartheta\zeta+\vartheta^2/8}=1.
\]
The $\bar\psi_s$ are independent, so
$\mathbb E[e^{\vartheta\Psi_{k+1}}\mid\Psi_1,\ldots,\Psi_k]
=e^{\vartheta\Psi_k}\,\mathbb E e^{\vartheta\bar\psi_{k+1}}\le e^{\vartheta\Psi_k}$:
$(e^{\vartheta\Psi_k})_{k\ge0}$ is a nonnegative supermartingale with
$e^{\vartheta\Psi_0}=1$. Ville's maximal inequality \citep{ville1939etude}
then gives, for $x\ge0$,
\begin{equation}\label{eq:excess-tail}
\mathbb P(\Gamma_t\ge x)\le\mathbb P(\Xi_t\ge x)
=\mathbb P\Bigl(\max_{0\le k\le t-1}e^{\vartheta\Psi_k}\ge e^{\vartheta x}\Bigr)
\le e^{-8\zeta x}.
\end{equation}
By the definition of $\Gamma_t$,
$\{\alpha_t<-\gamma(1-\alpha)-\gamma x\}\subseteq\{\Gamma_t>x\}$,
so~\eqref{eq:excess-tail} gives the first inequality of~\eqref{eq:level-tail}.
For the second, $\mathbb E\Gamma_t=\int_0^\infty\mathbb P(\Gamma_t>x)\,dx$
because $\Gamma_t\ge0$, so~\eqref{eq:level-lower-by-excess}
and~\eqref{eq:excess-tail} give
\[
\mathbb E\alpha_t\ge-\gamma(1-\alpha)-\gamma\,\mathbb E\Gamma_t
\ge-\gamma(1-\alpha)-\gamma\int_0^\infty e^{-8\zeta x}\,dx
=-\gamma(1-\alpha)-\frac{\gamma}{8\zeta}.
\qedhere
\]
\end{proof}

\subsubsection{Proof of Theorem~\ref{thm:temporal_mcar}}\label{app:mcar-proof}

\begin{proof}
Taking expectations in~\eqref{eq:t-normalized} gives
\begin{equation}\label{eq:mcar-expect}
\frac1T\sum_{t=1}^T \mathbb E[R_t\ell_t]
-
\alpha\,\frac1T\sum_{t=1}^T \mathbb E[R_t]
=
\frac{\alpha-\mathbb E\alpha_{T+1}}{T\gamma} .
\end{equation}
The set $C_t$ is $\mathcal F_t$-measurable and
$Y_{N+1,t}\in\sigma(Y_{N+1,t})$, so $\ell_t=\mathbf 1\{Y_{N+1,t}\notin C_t\}$
is $\mathcal H_t$-measurable. Assumption~\ref{ass:mcar} gives
$\mathbb E[R_t\mid \mathcal H_t]=\pi$ almost surely, hence
\[
\mathbb E[R_t\ell_t]
=
\mathbb E\!\left[\ell_t\,\mathbb E(R_t\mid \mathcal H_t)\right]
=
\pi\,\mathbb E[\ell_t],
\qquad
\mathbb E[R_t]=\pi.
\]
Substituting into~\eqref{eq:mcar-expect}, dividing by $\pi$ and using
$\mathbb E[\ell_t]=\mathbb P(Y_{N+1,t}\notin C_t(X_{N+1,t}))$ give
\begin{equation}\label{eq:mcar-identity}
\begin{aligned}
&\frac1T\sum_{t=1}^T
\mathbb P\!\left(Y_{N+1,t}\notin C_t(X_{N+1,t})\right)
-\alpha\\
&\qquad=
\frac{\alpha-\mathbb E\alpha_{T+1}}{\pi\gamma T}.
\end{aligned}
\end{equation}
Taking expectations
in~\eqref{eq:terminal-level-upper} gives
$\alpha-\mathbb E\alpha_{T+1}\ge-(1-\alpha+\gamma\alpha)$, which
with~\eqref{eq:mcar-identity} gives the lower bound in~\eqref{eq:all-round-bound}. Under
Assumption~\ref{ass:feasibility}, \eqref{eq:fallback-feasibility} holds in
particular on $\{\alpha_t<0\}$ for every $t$, so
Lemma~\ref{lem:level-stability} with $t=T+1$ gives
$\mathbb E\alpha_{T+1}\ge-\gamma(1-\alpha)-\gamma/(8\zeta)$, that is,
$\alpha-\mathbb E\alpha_{T+1}\le\alpha+\gamma(1-\alpha)+\gamma/(8\zeta)$,
which with~\eqref{eq:mcar-identity} gives the upper bound in~\eqref{eq:all-round-bound}.
The lower bound in~\eqref{eq:revealed-two-sided} is~\eqref{eq:revealed-lower-bound}.
For the upper bound, fix $\delta\in(0,1)$. By the first inequality of~\eqref{eq:level-tail} with
$t=T+1$ and $x=\log(1/\delta)/(8\zeta)$, the event
$\{\alpha_{T+1}\ge-\gamma(1-\alpha)-\gamma\log(1/\delta)/(8\zeta)\}$ has
probability at least $1-\delta$, and on it
$\alpha-\alpha_{T+1}\le\alpha+\gamma(1-\alpha)+\gamma\log(1/\delta)/(8\zeta)$.
Inserting this bound into~\eqref{eq:revealed-frequency-identity} gives the upper
bound in~\eqref{eq:revealed-two-sided} on the same event when $\mathcal R_T\ne\varnothing$.
\end{proof}

\subsubsection{Selection bias without MCAR}\label{app:selection_bias}

When Assumption~\ref{ass:mcar} fails, the revealed rounds no longer represent
all rounds, and the identity~\eqref{eq:mcar-identity} acquires a selection-bias
term.
\begin{proposition}\label{prop:selection_bias}
Run Algorithm~\ref{alg:wtqa} with step size $\gamma>0$, let
$p_t:=\mathbb P(R_t=1\mid\mathcal H_t)$, and suppose
$\mathbb E p_t=\pi\in(0,1]$ for all $t$. Then, for every $T\ge1$,
\begin{equation}\label{eq:selection-bias-decomposition}
\begin{aligned}
&\frac1T\sum_{t=1}^T\mathbb P(Y_{N+1,t}\notin C_t)-\alpha\\
&\qquad=\frac{\alpha-\mathbb E\alpha_{T+1}}{\pi\gamma T}
-\frac{\mathrm{Bias}_T}{\pi},\\
&\mathrm{Bias}_T:=\frac1T\sum_{t=1}^T\mathrm{Cov}(p_t,\ell_t).
\end{aligned}
\end{equation}
Under Assumption~\ref{ass:mcar}, $\mathrm{Bias}_T=0$
and~\eqref{eq:selection-bias-decomposition} reduces to~\eqref{eq:mcar-identity}.
\end{proposition}

\begin{proof}
Since $\ell_t$ is $\mathcal H_t$-measurable, the tower property gives
$\mathbb E[R_t(\ell_t-\alpha)]=\mathbb E[p_t(\ell_t-\alpha)]$, so taking
expectations in~\eqref{eq:t-normalized} yields
\begin{equation}\label{eq:selection-bias-identity}
\frac1T\sum_{t=1}^T\mathbb E[p_t(\ell_t-\alpha)]
=\frac{\alpha-\mathbb E\alpha_{T+1}}{T\gamma}.
\end{equation}
Write $p_t(\ell_t-\alpha)=\pi(\ell_t-\alpha)+(p_t-\pi)(\ell_t-\alpha)$. Since
$\mathbb E p_t=\pi$,
$\mathbb E[(p_t-\pi)(\ell_t-\alpha)]=\mathbb E[(p_t-\pi)\ell_t]=\mathrm{Cov}(p_t,\ell_t)$,
so the left-hand side of~\eqref{eq:selection-bias-identity} equals
$\pi\bigl(\frac1T\sum_{t=1}^T\mathbb P(Y_{N+1,t}\notin C_t)-\alpha\bigr)+\mathrm{Bias}_T$.
Dividing by $\pi$ proves~\eqref{eq:selection-bias-decomposition}. Under Assumption~\ref{ass:mcar},
$p_t=\pi$ almost surely, so $\mathrm{Cov}(p_t,\ell_t)=0$ for every $t$ and
$\mathrm{Bias}_T=0$.
\end{proof}

The term $\mathrm{Bias}_T/\pi$ is the selection bias of the revealed-round
miscoverage rate: when hard rounds are revealed more often
($\mathrm{Cov}(p_t,\ell_t)>0$), the all-round miscoverage is smaller than the
revealed rounds suggest, and when they are revealed less often it is larger.
It can be as large as the interval of
Proposition~\ref{prop:temporal_impossibility} allows. In the family of
Appendix~\ref{app:temporal-impossibility-proof} applied at every round $t\le T$,
with $T<(1-\alpha)/(\gamma\alpha)$, $\pi_t\equiv\pi$, $m_t^{\mathrm{obs}}\equiv\alpha$ and a common
$\theta_t\equiv\theta\in[0,1]$, the revealed rounds miss at rate $\alpha$,
$\mathbb E p_t=\mathbb P(R_t=1)=\pi$, and
$q_t:=\mathbb P(\ell_t=1)=\pi\alpha+(1-\pi)\theta$. Since
$\mathbb E[p_t\ell_t]=\mathbb E[R_t\ell_t]=\pi\alpha$,
\[
\mathrm{Cov}(p_t,\ell_t)=\pi\alpha-\pi q_t=\pi(1-\pi)(\alpha-\theta),
\qquad
\frac1T\sum_{t=1}^T q_t-\alpha
=(1-\pi)(\theta-\alpha)
=-\frac{\mathrm{Bias}_T}{\pi}.
\]
Thus the all-round average deviates from $\alpha$ by exactly
$-\mathrm{Bias}_T/\pi$, and $|\mathrm{Bias}_T|/\pi$ equals $(1-\pi)(1-\alpha)$
at $\theta=1$, where every hidden round misses.

\subsubsection{Feedback delayed by up to \texorpdfstring{$D$}{D} rounds}\label{app:delayed-feedback}

Section~\ref{sec:problem_formulation} lets a revealed response arrive one
round after its covariate. This subsection lets it arrive up to $D\ge1$ rounds
later, applies each level update when its response arrives, and shows that the
guarantees of Section~\ref{sec:wtqa_guarantees} then hold, with the level
$\alpha_{T+1}$ of Proposition~\ref{prop:temporal-observed} replaced by a
virtual level and the constants $\gamma\alpha$ and $\gamma(1-\alpha)$ of
Proposition~\ref{prop:temporal-observed} and Theorem~\ref{thm:temporal_mcar}
by $D\gamma\alpha$ and $D\gamma(1-\alpha)$. Like
Appendix~\ref{app:second-order-extension}, this is a theoretical extension;
Algorithm~\ref{alg:wtqa} and the experiments use $D=1$.

\paragraph{Delay model and delayed update.}
Fix an integer $D\ge1$. The response $Y_{N+1,s}$ of round $s$ either arrives
with the batch of some round $a_s\in\{s+1,\dots,s+D\}$, in which case
$R_s:=1$, or does not arrive within $D$ rounds, in which case $R_s:=0$ and
$a_s:=\infty$; a response arriving after round $s+D$ is discarded. For $D=1$
this is the model of Section~\ref{sec:problem_formulation}, with $a_s=s+1$ on
$\{R_s=1\}$. The batch arriving at round $t$ is
\[
\mathcal O_t^{D}:=\bigl((X_{k,t},Y_{k,t})_{k\le N},\;X_{N+1,t},\;
(\mathbf 1\{a_s=t\},\,\mathbf 1\{a_s=t\}Y_{N+1,s})_{\max\{1,t-D\}\le s<t}\bigr).
\]
The delayed variant of Algorithm~\ref{alg:wtqa} replaces its level update by
one gradient step per arriving response, so that for every $t\ge1$
\begin{equation}\label{eq:delayed-level}
\alpha_t=\alpha-\gamma\sum_{s\ge1:\,a_s\le t}(\ell_s-\alpha),
\end{equation}
where $\ell_s=\mathbf 1\{Y_{N+1,s}\notin C_s\}$ is computed when $Y_{N+1,s}$
arrives, from the stored threshold $\hat q_s$ and covariate $X_{N+1,s}$; every
other line is unchanged. Since $a_s\ge s+1$, the sum in~\eqref{eq:delayed-level}
is empty at $t=1$ and $\alpha_1=\alpha$; for $D=1$, \eqref{eq:delayed-level}
is~\eqref{eq:level-recursion} solved from $\alpha_1=\alpha$.

\paragraph{Virtual level.}
For $t\ge1$ let
\begin{equation}\label{eq:virtual-level}
\alpha^{\ast}_t:=\alpha-\gamma\sum_{s=1}^{t-1}R_s(\ell_s-\alpha),
\end{equation}
the level that~\eqref{eq:level-recursion} produces from the same miss
indicators when every revealed response arrives after one round:
$\alpha^{\ast}_1=\alpha$,
$\alpha^{\ast}_{t+1}=\alpha^{\ast}_t+\gamma R_t(\alpha-\ell_t)$, and
$\alpha^{\ast}_t=\alpha_t$ when $D=1$. Let
$\mathcal P_t:=\{s<t:R_s=1,\,a_s>t\}$ be the set of revealed rounds whose
responses are pending at round $t$. Every $s\in\mathcal P_t$ satisfies
$t<a_s\le s+D$, so $\mathcal P_t\subseteq\{t-D+1,\dots,t-1\}$ and
$|\mathcal P_t|\le D-1$. Since $a_s\le t$ implies $R_s=1$ and $s<t$,
subtracting~\eqref{eq:virtual-level} from~\eqref{eq:delayed-level} gives
\begin{equation}\label{eq:virtual-gap}
\alpha_t-\alpha^{\ast}_t=\gamma\sum_{s\in\mathcal P_t}(\ell_s-\alpha)
\in\bigl[-(D-1)\gamma\alpha,\;(D-1)\gamma(1-\alpha)\bigr],
\end{equation}
because each $\ell_s-\alpha$ lies in $\{-\alpha,1-\alpha\}$.

\paragraph{Information.}
Let $\mathcal F_{t-1}^{+}$ be the $\sigma$-field of
Section~\ref{sec:problem_formulation} generated by the calibration slices, the
target covariates and the pairs $(R_s,R_sY_{N+1,s})$ of rounds $s\le t-1$,
with $R_s$ as defined here, and let
\[
\mathcal G_{t-1}:=\mathcal F_{t-1}^{+}\vee
\sigma\bigl(\mathbf 1\{a_s\le r\}:1\le s<r\le t\bigr),
\qquad
\mathcal H_t^{D}:=\mathcal G_{t-1}\vee
\sigma\bigl((X_{k,t},Y_{k,t})_{k\le N},\,X_{N+1,t},\,Y_{N+1,t}\bigr).
\]
Thus $\mathcal G_{t-1}$ adds to the information of a run without delay the
arrival status, through round $t$, of every earlier response. For $D=1$,
$\mathbf 1\{a_s\le r\}=R_s$ for $r>s$, so $\mathcal G_{t-1}=\mathcal F_{t-1}^{+}$
and $\mathcal H_t^{D}=\mathcal F_t\vee\sigma(Y_{N+1,t})=\mathcal H_t$.
The weights $W^{(t)}$ are $\mathcal F_{t-1}$-measurable as before. By
induction on $t$, the level $\alpha_t$ of~\eqref{eq:delayed-level} is
$\mathcal G_{t-1}$-measurable: for $s<t$, $\mathbf 1\{a_s\le t\}$ is
$\mathcal G_{t-1}$-measurable, and on $\{a_s\le t\}$ the indicator $\ell_s$ is
a function of $R_sY_{N+1,s}$, $X_{N+1,s}$ and $\hat q_s$, where $\hat q_s$ is
determined by the round-$s$ calibration slice, $W^{(s)}$ and $\alpha_s$, all
$\mathcal G_{t-1}$-measurable. Hence $\tilde\alpha_t$ and the fallback event
$E_t$ are $\mathcal G_{t-1}$-measurable, $C_t$ and $\ell_t$ are
$\mathcal H_t^{D}$-measurable, and $R_s\ell_s$ and $\alpha^{\ast}_t$ are
$\mathcal G_{t-1}$-measurable for $s<t$. The $\sigma$-field generated by
$\mathcal O_1^{D},\dots,\mathcal O_t^{D}$, the information of the delayed
variant at round $t$, is contained in
$\mathcal G_{t-1}\vee\sigma((X_{k,t},Y_{k,t})_{k\le N},X_{N+1,t})$; the
larger $\sigma$-field $\mathcal G_{t-1}$ also contains the values of the
pending responses and serves only the analysis.

\paragraph{Assumptions.}
Assumption~\ref{ass:feasibility} is read with $\mathcal G_{t-1}$ in place of
$\mathcal F_{t-1}^{+}$ and Assumption~\ref{ass:mcar} with $\mathcal H_t^{D}$
in place of $\mathcal H_t$: there are $\zeta\in(0,\alpha]$ and $\pi\in(0,1]$
such that, almost surely for every $t$,
\begin{equation}\label{eq:delayed-feasibility}
\mathbb P\Bigl(\hat s_{N+1,t}>\max_{k\le N}\hat s_{k,t}\,\Big|\,\mathcal G_{t-1}\Bigr)
\le\alpha-\zeta
\end{equation}
and
\begin{equation}\label{eq:delayed-mcar}
\mathbb P\bigl(R_t=1\mid\mathcal H_t^{D}\bigr)=\pi .
\end{equation}
Condition~\eqref{eq:delayed-mcar} restricts only the event that the response
arrives within $D$ rounds; on that event, the arrival round $a_t$ may depend
on $Y_{N+1,t}$, for instance when misses are reported sooner than covers.
Condition~\eqref{eq:delayed-feasibility} also conditions on the arrival status
of earlier responses; when, given $\mathcal F_{t-1}^{+}$, the indicators
$(\mathbf 1\{a_s\le r\})_{s<r\le t}$ are conditionally independent of the
round-$t$ score vector, the conditional probabilities of events of that
vector given $\mathcal G_{t-1}$ and given $\mathcal F_{t-1}^{+}$ coincide, and
\eqref{eq:delayed-feasibility} is Assumption~\ref{ass:feasibility} itself.

\paragraph{Past-conditional coverage.}
Theorems~\ref{thm:spatial_stepwise} and~\ref{thm:spatial_oracle} hold for the
delayed variant with $\mathcal G_{t-1}$ in place of $\mathcal F_{t-1}^{+}$ in
the conditional miscoverage probability, in the laws $P_t$ and
$P_{t,k}^{\mathrm{swap}}$, and as the $\sigma$-field on which $\phi_t$ in
Assumption~\ref{ass:spatial_profiles} may depend. Their proofs use
$\mathcal F_{t-1}^{+}$ only through the measurability of $W^{(t)}$,
$\tilde\alpha_t$, the profiles, $\rho_{t-1}$ and $\phi_t$, and through
Theorem~2 of \citet{barber2023conformal} applied to the conditional laws
given that $\sigma$-field; $W^{(t)}$ and $\tilde\alpha_t$ are
$\mathcal G_{t-1}$-measurable and $\mathcal F_{t-1}^{+}\subseteq\mathcal G_{t-1}$.

\paragraph{Long-run average coverage.}
The next result is Proposition~\ref{prop:temporal-observed} and
Theorem~\ref{thm:temporal_mcar} for the delayed variant; at $D=1$ it reduces
to them.
\begin{proposition}\label{prop:delayed-feedback}
Run the delayed variant of Algorithm~\ref{alg:wtqa} with step size $\gamma>0$
and maximal delay $D\ge1$. Then $\alpha^{\ast}_t\le1+D\gamma\alpha$ for every
$t\ge1$, and for every $T\ge1$
\begin{equation}\label{eq:delayed-identity}
\sum_{t\in\mathcal R_T}(\ell_t-\alpha)=\frac{\alpha-\alpha^{\ast}_{T+1}}{\gamma}.
\end{equation}
If~\eqref{eq:delayed-feasibility} and~\eqref{eq:delayed-mcar} hold, then for
every $t\ge1$ and $x\ge0$
\begin{equation}\label{eq:delayed-level-tail}
\mathbb P\bigl(\alpha^{\ast}_t<-D\gamma(1-\alpha)-\gamma x\bigr)\le e^{-8\zeta x},
\qquad
\mathbb E\alpha^{\ast}_t\ge-D\gamma(1-\alpha)-\frac{\gamma}{8\zeta};
\end{equation}
for every $T\ge1$ and $\delta\in(0,1)$, with probability at least $1-\delta$,
if $\mathcal R_T\ne\varnothing$,
\begin{equation}\label{eq:delayed-revealed-two-sided}
-\frac{1-\alpha+D\gamma\alpha}{\gamma|\mathcal R_T|}
\le\frac1{|\mathcal R_T|}\sum_{t\in\mathcal R_T}\ell_t-\alpha
\le\frac{\alpha+D\gamma(1-\alpha)+\gamma\log(1/\delta)/(8\zeta)}{\gamma|\mathcal R_T|};
\end{equation}
and for every $T\ge1$
\begin{equation}\label{eq:delayed-all-round-bound}
-\frac{1-\alpha+D\gamma\alpha}{\pi\gamma T}
\le\frac1T\sum_{t=1}^T\mathbb P\bigl(Y_{N+1,t}\notin C_t(X_{N+1,t})\bigr)-\alpha
\le\frac{\alpha+D\gamma(1-\alpha)+\gamma/(8\zeta)}{\pi\gamma T}.
\end{equation}
\end{proposition}

The rates are those of Theorem~\ref{thm:temporal_mcar} and only the constants
grow, linearly in $D$: for $\gamma=0.01$, $\alpha=0.1$ and $\zeta$ near
$\alpha$, the numerator of the upper bound in~\eqref{eq:delayed-all-round-bound}
is about $0.12$ at $D=1$ and $0.16$ at $D=5$. The realized revealed-round
miscoverage in~\eqref{eq:delayed-revealed-two-sided} remains observable:
$\alpha^{\ast}_{T+1}$ is determined by the responses of rounds $s\le T$, each
of which has arrived or been discarded by round $T+D$.

\begin{proof}
\emph{The identity and the upper bound.} Equation~\eqref{eq:delayed-identity}
is~\eqref{eq:virtual-level} at $t=T+1$, since $R_s=1$ exactly for
$s\in\mathcal R_T$. For the upper bound, $\alpha^{\ast}_1=\alpha$; fix
$t\ge2$ and let $u:=\max\{s\le t-1:\alpha_s<1\}$, which exists because
$\alpha_1=\alpha<1$. For $u<s\le t-1$, $\alpha_s\ge1$, so $\tilde\alpha_s\ge1$,
$\hat q_s=-\infty$ by Lemma~\ref{lem:effective-level}, $C_s=\varnothing$,
$\ell_s=1$ and $R_s(\ell_s-\alpha)\ge0$. By~\eqref{eq:virtual-level},
\[
\alpha^{\ast}_t
=\alpha^{\ast}_u-\gamma R_u(\ell_u-\alpha)-\gamma\sum_{s=u+1}^{t-1}R_s(\ell_s-\alpha)
\le\alpha^{\ast}_u+\gamma\alpha,
\]
and~\eqref{eq:virtual-gap} gives
$\alpha^{\ast}_u\le\alpha_u+(D-1)\gamma\alpha<1+(D-1)\gamma\alpha$. Hence
$\alpha^{\ast}_t<1+D\gamma\alpha$.

\emph{The lower bound on the virtual level.} The proof of
Lemma~\ref{lem:level-stability} applies to $\alpha^{\ast}_t$ with the
following changes. Set
$\Gamma^{\ast}_t:=(-\alpha^{\ast}_t/\gamma-D(1-\alpha))_+$ and
$\psi_t:=R_t(\ell_t-\alpha)\in[-\alpha,1-\alpha]$, so that
$\alpha^{\ast}_t\ge-D\gamma(1-\alpha)-\gamma\Gamma^{\ast}_t$,
$\Gamma^{\ast}_1=0$ and
$-\alpha^{\ast}_{t+1}/\gamma=-\alpha^{\ast}_t/\gamma+\psi_t$. Let
$B_t:=\{\alpha^{\ast}_t<-(D-1)\gamma(1-\alpha)\}$. On $B_t^{c}$,
$-\alpha^{\ast}_{t+1}/\gamma\le(D-1)(1-\alpha)+\psi_t\le D(1-\alpha)$, so
$\Gamma^{\ast}_{t+1}=0$; on $B_t$,
$-\alpha^{\ast}_{t+1}/\gamma-D(1-\alpha)\le\Gamma^{\ast}_t+\psi_t$. Hence
\begin{equation}\label{eq:delayed-excess-recursion}
\Gamma^{\ast}_{t+1}\le\mathbf 1_{B_t}\,(\Gamma^{\ast}_t+\psi_t)_+ ,
\end{equation}
which is~\eqref{eq:excess-recursion} with $B_t$ in place of $\{\alpha_t<0\}$.
On $B_t$, \eqref{eq:virtual-gap} gives
$\alpha_t\le\alpha^{\ast}_t+(D-1)\gamma(1-\alpha)<0$, so round $t$ is a
fallback round, $\hat q_t=\max_{k\le N}\hat s_{k,t}$ by
Lemma~\ref{lem:effective-level} and
$\ell_t=\mathbf 1\{\hat s_{N+1,t}>\max_{k\le N}\hat s_{k,t}\}$. The event $B_t$
is $\mathcal G_{t-1}$-measurable, so with
$\varpi_t:=\mathbb P(\ell_t=1\mid\mathcal G_{t-1})$,
\eqref{eq:delayed-feasibility} gives $\varpi_t\le\alpha-\zeta$ on $B_t$, which
is~\eqref{eq:negative-level-miss} with $B_t$ and $\mathcal G_{t-1}$ in place
of $\{\alpha_t<0\}$ and $\mathcal F_{t-1}^{+}$. The coupling step of that
proof, with $\mathcal J_t:=\mathcal G_{t-1}\vee\sigma(\chi_s:s<t)$, uses
$\mathcal F_{t-1}^{+}$ only through three facts, each of which holds for
$\mathcal G_{t-1}$: $\ell_t$, $\varpi_t$ and $\mathcal G_{t-1}$ are
$\mathcal H_t^{D}$-measurable, so that~\eqref{eq:delayed-mcar} gives the
conditional law of $(R_t,R_t\ell_t)$ given $\mathcal J_t$ used there; $R_s$
and $R_s\ell_s$ are $\mathcal G_{t-1}$-measurable for $s<t$, so that
$(R_s,\tau_s)$ is $\mathcal J_t$-measurable; and $\mathcal J_t$ increases in
$t$. Hence the pairs $(R_t,\tau_t)$, $t\ge1$, are independent with law
Bernoulli$(\pi)\otimes$Uniform$(0,1)$, $\psi_t\le\bar\psi_t$ on $B_t$, and the
induction that gave $\Gamma_t\le\Xi_t$ gives $\Gamma^{\ast}_t\le\Xi_t$,
with~\eqref{eq:delayed-excess-recursion} in place
of~\eqref{eq:excess-recursion}. The tail bound~\eqref{eq:excess-tail} for
$\Xi_t$ does not involve the level and is unchanged. Since
$\{\alpha^{\ast}_t<-D\gamma(1-\alpha)-\gamma x\}\subseteq\{\Gamma^{\ast}_t>x\}$,
it gives the first inequality of~\eqref{eq:delayed-level-tail}, and
integrating the tail as in the proof of Lemma~\ref{lem:level-stability} gives
the second.

\emph{The bounds.} Dividing~\eqref{eq:delayed-identity} by $|\mathcal R_T|$
on $\{\mathcal R_T\ne\varnothing\}$ and using
$\alpha^{\ast}_{T+1}\le1+D\gamma\alpha$ gives the lower bound
in~\eqref{eq:delayed-revealed-two-sided}; on the event
$\{\alpha^{\ast}_{T+1}\ge-D\gamma(1-\alpha)-\gamma\log(1/\delta)/(8\zeta)\}$,
which has probability at least $1-\delta$ by the first inequality
of~\eqref{eq:delayed-level-tail}, it gives the upper bound.
Dividing~\eqref{eq:delayed-identity} by $T$ and taking expectations
gives~\eqref{eq:mcar-expect} with $\alpha^{\ast}_{T+1}$ in place of
$\alpha_{T+1}$. Since $\ell_t$ is $\mathcal H_t^{D}$-measurable,
\eqref{eq:delayed-mcar} gives $\mathbb E[R_t\ell_t]=\pi\,\mathbb E[\ell_t]$
and $\mathbb E[R_t]=\pi$ as in the proof of Theorem~\ref{thm:temporal_mcar}, so
\[
\frac1T\sum_{t=1}^T\mathbb P\bigl(Y_{N+1,t}\notin C_t(X_{N+1,t})\bigr)-\alpha
=\frac{\alpha-\mathbb E\alpha^{\ast}_{T+1}}{\pi\gamma T},
\]
and
$-(1-\alpha+D\gamma\alpha)\le\alpha-\mathbb E\alpha^{\ast}_{T+1}\le\alpha+D\gamma(1-\alpha)+\gamma/(8\zeta)$
by $\alpha^{\ast}_{T+1}\le1+D\gamma\alpha$ and the second inequality
of~\eqref{eq:delayed-level-tail}, which proves~\eqref{eq:delayed-all-round-bound}.
\end{proof}

\section{Experiment Details and Additional Results}\label{app:experiments}

Appendix~\ref{app:main-supplement} supplements Table~\ref{tab:main-results};
Appendices~\ref{app:simulation_experiments}, \ref{app:real-data-panels}
and~\ref{app:cross-market-panel} give the synthetic, real-data and
cross-market protocols, baselines, and complete result tables; Appendix~\ref{app:robustness} tests robustness to
outcome-informative feedback, to the hyperparameters and to the base predictor;
Appendices~\ref{app:clip-accounting} and~\ref{app:profile-distance-diagnostic}
measure the fallback and the profile geometry that enter the
guarantees of Section~\ref{sec:wtqa_guarantees}; and
Appendix~\ref{app:compute-resources} lists the compute used.

\subsection{Supplementary tables for the main results}\label{app:main-supplement}

Table~\ref{tab:main-results-sd} repeats Table~\ref{tab:main-results} with
standard deviations over the $30$ replications.
Table~\ref{tab:paired-margins} gives the paired tail-coverage margins of W-TQA
over its two ablations at every feedback probability, with the numbers of
splits in which W-TQA is higher, tied and lower; the synthetic rows at $p<1$
use one Bernoulli$(p)$ reveal draw per round, shared by all targets.
Table~\ref{tab:width-quantiles} gives the median and $90$th-percentile interval
widths of TQA-only and W-TQA and their ratios, with paired split-bootstrap
$95\%$ intervals. On the real and cross-market panels, the median ratio lies between $0.93$ and
$1.05$ at every $p$, while the $90$th-percentile ratio lies between $1.03$ and
$1.26$, with every interval above one: W-TQA keeps the typical interval close
to that of TQA-only and concentrates its additional width in the upper tail of
the width distribution.

\begin{table*}[htbp]
  \centering
  \footnotesize
  \captionsetup{skip=5pt}
  \setlength{\tabcolsep}{2pt}
  \renewcommand{\arraystretch}{0.92}
  \caption{Table~\ref{tab:main-results} with standard deviations: mean~$\pm$~sd
  over $30$ replications, synthetic panels at $p=1$, HF, M5 and SGSC at $p=0.2$,
  and the cross-market panel at $p=1$.
  Bold as in Table~\ref{tab:main-results}.}\label{tab:main-results-sd}
  \resizebox{\ifdim\width>\linewidth\linewidth\else\width\fi}{!}{%
  \begin{tabular}{@{}llccccccc@{}}
  \toprule
  & & \multicolumn{3}{c}{Synthetic, $p=1$} & \multicolumn{3}{c}{Real, $p=0.2$} & $p=1$\\
  \cmidrule(lr){3-5}\cmidrule(lr){6-8}\cmidrule(l){9-9}
  & Method & Easy & Med. & Hard & HF & M5 & SGSC & Cross\\
  \midrule
  \multirow{6}{*}{\rotatebox{90}{Avg.\ cov.}} & Split CP & $0.874 \pm 0.011$ & $0.848 \pm 0.030$ & $0.746 \pm 0.044$ & $0.904 \pm 0.015$ & $\mathbf{0.904 \pm 0.011}$ & $\mathbf{0.902 \pm 0.014}$ & $\mathbf{0.904 \pm 0.007}$ \\
   & TQA-B & $\mathbf{0.883 \pm 0.009}$ & $0.867 \pm 0.019$ & $0.819 \pm 0.021$ & $0.914 \pm 0.011$ & $0.920 \pm 0.008$ & $0.916 \pm 0.010$ & $0.905 \pm 0.006$ \\
   & LPCI & $0.788 \pm 0.012$ & $0.730 \pm 0.021$ & $0.709 \pm 0.031$ & $\mathbf{0.898 \pm 0.013}$ & $0.878 \pm 0.007$ & $0.864 \pm 0.018$ & $\mathbf{0.904 \pm 0.007}$ \\
   & TQA-only & $0.882 \pm 0.008$ & $0.864 \pm 0.021$ & $0.815 \pm 0.021$ & $0.922 \pm 0.007$ & $0.931 \pm 0.004$ & $0.924 \pm 0.006$ & $0.905 \pm 0.004$ \\
   & W-only & $0.936 \pm 0.009$ & $0.937 \pm 0.010$ & $0.929 \pm 0.021$ & $0.927 \pm 0.009$ & $0.913 \pm 0.007$ & $0.933 \pm 0.010$ & $0.915 \pm 0.004$ \\
   & W-TQA & $0.928 \pm 0.008$ & $\mathbf{0.929 \pm 0.008}$ & $\mathbf{0.922 \pm 0.015}$ & $0.926 \pm 0.004$ & $0.925 \pm 0.003$ & $0.935 \pm 0.005$ & $0.907 \pm 0.003$ \\
  \midrule
  \multirow{6}{*}{\rotatebox{90}{Tail cov.}} & Split CP & $0.800 \pm 0.025$ & $0.752 \pm 0.053$ & $0.502 \pm 0.105$ & $0.636 \pm 0.040$ & $0.579 \pm 0.030$ & $0.646 \pm 0.048$ & $0.727 \pm 0.024$ \\
   & TQA-B & $0.825 \pm 0.016$ & $0.803 \pm 0.026$ & $0.713 \pm 0.048$ & $0.713 \pm 0.022$ & $0.690 \pm 0.024$ & $0.751 \pm 0.021$ & $0.763 \pm 0.017$ \\
   & LPCI & $0.711 \pm 0.021$ & $0.648 \pm 0.030$ & $0.616 \pm 0.037$ & $0.665 \pm 0.043$ & $0.645 \pm 0.027$ & $0.575 \pm 0.041$ & $0.721 \pm 0.022$ \\
   & TQA-only & $0.832 \pm 0.011$ & $0.804 \pm 0.029$ & $0.729 \pm 0.033$ & $0.813 \pm 0.015$ & $0.827 \pm 0.011$ & $0.836 \pm 0.010$ & $0.824 \pm 0.005$ \\
   & W-only & $0.880 \pm 0.018$ & $0.872 \pm 0.026$ & $0.781 \pm 0.089$ & $0.772 \pm 0.024$ & $0.697 \pm 0.028$ & $0.763 \pm 0.038$ & $0.826 \pm 0.012$ \\
   & W-TQA & $\mathbf{0.884 \pm 0.014}$ & $\mathbf{0.880 \pm 0.017}$ & $\mathbf{0.828 \pm 0.042}$ & $\mathbf{0.848 \pm 0.011}$ & $\mathbf{0.847 \pm 0.008}$ & $\mathbf{0.858 \pm 0.009}$ & $\mathbf{0.858 \pm 0.006}$ \\
  \midrule
  \multirow{5}{*}{\rotatebox{90}{\shortstack{Rel.\\width}}} & TQA-B & $1.042 \pm 0.012$ & $1.058 \pm 0.028$ & $1.163 \pm 0.047$ & $1.030 \pm 0.022$ & $1.006 \pm 0.005$ & $1.033 \pm 0.014$ & $1.001 \pm 0.005$ \\
   & LPCI & $0.870 \pm 0.019$ & $0.800 \pm 0.042$ & $0.881 \pm 0.054$ & $0.982 \pm 0.003$ & $0.786 \pm 0.007$ & $0.830 \pm 0.020$ & $0.978 \pm 0.005$ \\
   & TQA-only & $1.024 \pm 0.011$ & $1.041 \pm 0.021$ & $1.154 \pm 0.048$ & $1.048 \pm 0.047$ & $0.986 \pm 0.016$ & $1.024 \pm 0.041$ & $0.995 \pm 0.014$ \\
   & W-only & $1.251 \pm 0.046$ & $1.293 \pm 0.081$ & $1.553 \pm 0.100$ & $1.113 \pm 0.067$ & $0.991 \pm 0.025$ & $1.297 \pm 0.052$ & $1.083 \pm 0.036$ \\
   & W-TQA & $1.217 \pm 0.037$ & $1.256 \pm 0.075$ & $1.488 \pm 0.086$ & $1.084 \pm 0.070$ & $0.963 \pm 0.027$ & $1.172 \pm 0.061$ & $1.056 \pm 0.039$ \\
  \midrule
  \multirow{6}{*}{\rotatebox{90}{Width CoV}} & Split CP & $0.108 \pm 0.035$ & $0.228 \pm 0.044$ & $0.217 \pm 0.053$ & $0.568 \pm 0.006$ & $0.076 \pm 0.003$ & $0.346 \pm 0.012$ & $0.310 \pm 0.007$ \\
   & TQA-B & $0.172 \pm 0.026$ & $0.263 \pm 0.031$ & $0.279 \pm 0.038$ & $0.688 \pm 0.053$ & $0.118 \pm 0.008$ & $0.366 \pm 0.016$ & $0.340 \pm 0.008$ \\
   & LPCI & $0.174 \pm 0.011$ & $0.239 \pm 0.021$ & $0.269 \pm 0.026$ & $0.087 \pm 0.009$ & $0.034 \pm 0.002$ & $0.209 \pm 0.021$ & $0.099 \pm 0.005$ \\
   & TQA-only & $0.118 \pm 0.031$ & $0.243 \pm 0.038$ & $\mathbf{0.288 \pm 0.045}$ & $0.901 \pm 0.109$ & $0.233 \pm 0.025$ & $0.489 \pm 0.031$ & $0.370 \pm 0.014$ \\
   & W-only & $0.214 \pm 0.029$ & $\mathbf{0.275 \pm 0.032}$ & $0.263 \pm 0.047$ & $1.002 \pm 0.183$ & $0.332 \pm 0.024$ & $0.566 \pm 0.035$ & $0.566 \pm 0.052$ \\
   & W-TQA & $\mathbf{0.225 \pm 0.031}$ & $\mathbf{0.275 \pm 0.031}$ & $0.273 \pm 0.044$ & $\mathbf{1.068 \pm 0.131}$ & $\mathbf{0.358 \pm 0.023}$ & $\mathbf{0.621 \pm 0.034}$ & $\mathbf{0.607 \pm 0.047}$ \\
  \bottomrule
  \end{tabular}%
  }
\end{table*}

\begin{table}[htbp]
  \centering
  \footnotesize
  \captionsetup{skip=5pt}
  \caption{Paired tail-coverage margins of W-TQA: mean~$\pm$~standard error over
  the $30$ splits (real panels) or replications (synthetic panels), and the
  numbers of splits in which W-TQA is higher, tied or lower (W/T/L). At $p=0$,
  W-TQA coincides with W-only.}\label{tab:paired-margins}
  \begin{tabular}{@{}llcccc@{}}
  \toprule
  Panel & $p$ & \multicolumn{2}{c}{W-TQA $-$ TQA-only} & \multicolumn{2}{c}{W-TQA $-$ W-only} \\
  \cmidrule(lr){3-4}\cmidrule(l){5-6}
  & & margin & W/T/L & margin & W/T/L \\
  \midrule
  HF hourly & 0.0 & $0.1360 \pm 0.0056$ & 30/0/0 & $0.0000 \pm 0.0000$ & 0/30/0 \\
  HF hourly & 0.2 & $0.0343 \pm 0.0019$ & 30/0/0 & $0.0757 \pm 0.0034$ & 30/0/0 \\
  HF hourly & 0.4 & $0.0147 \pm 0.0010$ & 30/0/0 & $0.0963 \pm 0.0040$ & 30/0/0 \\
  HF hourly & 0.6 & $0.0099 \pm 0.0007$ & 30/0/0 & $0.1081 \pm 0.0040$ & 30/0/0 \\
  HF hourly & 0.8 & $0.0070 \pm 0.0003$ & 30/0/0 & $0.1138 \pm 0.0042$ & 30/0/0 \\
  HF hourly & 1.0 & $0.0057 \pm 0.0002$ & 30/0/0 & $0.1180 \pm 0.0042$ & 30/0/0 \\
  \midrule
  M5 & 0.0 & $0.1184 \pm 0.0054$ & 30/0/0 & $0.0000 \pm 0.0000$ & 0/30/0 \\
  M5 & 0.2 & $0.0201 \pm 0.0018$ & 30/0/0 & $0.1494 \pm 0.0051$ & 30/0/0 \\
  M5 & 0.4 & $0.0087 \pm 0.0008$ & 30/0/0 & $0.1711 \pm 0.0052$ & 30/0/0 \\
  M5 & 0.6 & $0.0054 \pm 0.0006$ & 30/0/0 & $0.1794 \pm 0.0052$ & 30/0/0 \\
  M5 & 0.8 & $0.0035 \pm 0.0004$ & 29/0/1 & $0.1849 \pm 0.0052$ & 30/0/0 \\
  M5 & 1.0 & $0.0030 \pm 0.0003$ & 29/1/0 & $0.1892 \pm 0.0052$ & 30/0/0 \\
  \midrule
  SGSC & 0.0 & $0.1171 \pm 0.0054$ & 30/0/0 & $0.0000 \pm 0.0000$ & 0/30/0 \\
  SGSC & 0.2 & $0.0215 \pm 0.0011$ & 30/0/0 & $0.0942 \pm 0.0064$ & 30/0/0 \\
  SGSC & 0.4 & $0.0121 \pm 0.0007$ & 30/0/0 & $0.1096 \pm 0.0066$ & 30/0/0 \\
  SGSC & 0.6 & $0.0077 \pm 0.0005$ & 30/0/0 & $0.1187 \pm 0.0068$ & 30/0/0 \\
  SGSC & 0.8 & $0.0054 \pm 0.0003$ & 30/0/0 & $0.1237 \pm 0.0068$ & 30/0/0 \\
  SGSC & 1.0 & $0.0042 \pm 0.0001$ & 30/0/0 & $0.1285 \pm 0.0069$ & 30/0/0 \\
  \midrule
  Synthetic Easy & 0.0 & $0.0804 \pm 0.0045$ & 30/0/0 & $0.0000 \pm 0.0000$ & 0/30/0 \\
  Synthetic Easy & 0.2 & $0.0737 \pm 0.0041$ & 30/0/0 & $0.0019 \pm 0.0013$ & 10/14/6 \\
  Synthetic Easy & 0.4 & $0.0674 \pm 0.0034$ & 30/0/0 & $-0.0006 \pm 0.0014$ & 9/9/12 \\
  Synthetic Easy & 0.6 & $0.0643 \pm 0.0036$ & 30/0/0 & $0.0013 \pm 0.0013$ & 12/13/5 \\
  Synthetic Easy & 0.8 & $0.0561 \pm 0.0030$ & 30/0/0 & $0.0015 \pm 0.0021$ & 12/8/10 \\
  Synthetic Easy & 1.0 & $0.0520 \pm 0.0027$ & 30/0/0 & $0.0041 \pm 0.0016$ & 13/12/5 \\
  \midrule
  Synthetic Medium & 0.0 & $0.1198 \pm 0.0076$ & 30/0/0 & $0.0000 \pm 0.0000$ & 0/30/0 \\
  Synthetic Medium & 0.2 & $0.1083 \pm 0.0065$ & 30/0/0 & $0.0000 \pm 0.0010$ & 8/13/9 \\
  Synthetic Medium & 0.4 & $0.1026 \pm 0.0065$ & 30/0/0 & $0.0002 \pm 0.0014$ & 8/13/9 \\
  Synthetic Medium & 0.6 & $0.0954 \pm 0.0060$ & 30/0/0 & $0.0033 \pm 0.0020$ & 15/8/7 \\
  Synthetic Medium & 0.8 & $0.0833 \pm 0.0043$ & 30/0/0 & $0.0048 \pm 0.0021$ & 16/9/5 \\
  Synthetic Medium & 1.0 & $0.0763 \pm 0.0044$ & 30/0/0 & $0.0081 \pm 0.0023$ & 23/5/2 \\
  \midrule
  Synthetic Hard & 0.0 & $0.2787 \pm 0.0095$ & 30/0/0 & $0.0000 \pm 0.0000$ & 0/30/0 \\
  Synthetic Hard & 0.2 & $0.2378 \pm 0.0090$ & 30/0/0 & $0.0065 \pm 0.0048$ & 16/6/8 \\
  Synthetic Hard & 0.4 & $0.2091 \pm 0.0075$ & 30/0/0 & $0.0163 \pm 0.0066$ & 20/0/10 \\
  Synthetic Hard & 0.6 & $0.1691 \pm 0.0071$ & 30/0/0 & $0.0270 \pm 0.0072$ & 21/3/6 \\
  Synthetic Hard & 0.8 & $0.1272 \pm 0.0047$ & 30/0/0 & $0.0391 \pm 0.0085$ & 25/1/4 \\
  Synthetic Hard & 1.0 & $0.0991 \pm 0.0033$ & 30/0/0 & $0.0476 \pm 0.0094$ & 27/0/3 \\
  \bottomrule
  \end{tabular}
\end{table}

\begin{table*}[htbp]
  \centering
  \footnotesize
  \captionsetup{skip=5pt}
  \caption{Median (p50) and $90$th-percentile (p90) interval widths of TQA-only
  and W-TQA, computed over the unit-rounds of each split and averaged over the
  $30$ splits or replications, and the mean over splits of the ratio W-TQA /
  TQA-only with a paired split-bootstrap $95\%$ interval. Widths are on the
  native scale of each panel.}\label{tab:width-quantiles}
  \begin{tabular}{@{}llcccccc@{}}
  \toprule
  Panel & $p$ & \multicolumn{2}{c}{TQA-only} & \multicolumn{2}{c}{W-TQA} & \multicolumn{2}{c}{W-TQA / TQA-only} \\
  \cmidrule(lr){3-4}\cmidrule(lr){5-6}\cmidrule(l){7-8}
  & & p50 & p90 & p50 & p90 & p50 & p90 \\
  \midrule
  HF hourly & 0.2 & $0.01496$ & $0.03513$ & $0.01494$ & $0.03673$ & $0.998\,[0.994,1.003]$ & $1.045\,[1.033,1.057]$ \\
  HF hourly & 0.4 & $0.01468$ & $0.03495$ & $0.01455$ & $0.03619$ & $0.991\,[0.988,0.994]$ & $1.035\,[1.026,1.043]$ \\
  HF hourly & 0.6 & $0.01442$ & $0.03477$ & $0.01426$ & $0.03587$ & $0.989\,[0.986,0.991]$ & $1.031\,[1.025,1.037]$ \\
  HF hourly & 0.8 & $0.01426$ & $0.03474$ & $0.01409$ & $0.03571$ & $0.988\,[0.986,0.990]$ & $1.027\,[1.022,1.033]$ \\
  HF hourly & 1.0 & $0.01411$ & $0.03466$ & $0.01395$ & $0.03556$ & $0.989\,[0.987,0.990]$ & $1.025\,[1.021,1.030]$ \\
  \midrule
  M5 & 0.2 & $2.207$ & $2.918$ & $2.058$ & $3.188$ & $0.933\,[0.929,0.937]$ & $1.092\,[1.081,1.107]$ \\
  M5 & 0.4 & $2.123$ & $2.953$ & $2.006$ & $3.133$ & $0.945\,[0.942,0.949]$ & $1.061\,[1.054,1.069]$ \\
  M5 & 0.6 & $2.068$ & $2.954$ & $1.972$ & $3.099$ & $0.954\,[0.951,0.957]$ & $1.049\,[1.044,1.054]$ \\
  M5 & 0.8 & $2.032$ & $2.949$ & $1.951$ & $3.073$ & $0.960\,[0.957,0.963]$ & $1.042\,[1.038,1.047]$ \\
  M5 & 1.0 & $2.006$ & $2.944$ & $1.935$ & $3.052$ & $0.965\,[0.963,0.967]$ & $1.036\,[1.033,1.041]$ \\
  \midrule
  SGSC & 0.2 & $0.461$ & $0.751$ & $0.482$ & $0.945$ & $1.047\,[1.040,1.054]$ & $1.258\,[1.232,1.284]$ \\
  SGSC & 0.4 & $0.432$ & $0.791$ & $0.444$ & $0.945$ & $1.027\,[1.019,1.036]$ & $1.196\,[1.172,1.221]$ \\
  SGSC & 0.6 & $0.411$ & $0.808$ & $0.419$ & $0.934$ & $1.019\,[1.012,1.026]$ & $1.157\,[1.139,1.177]$ \\
  SGSC & 0.8 & $0.397$ & $0.815$ & $0.404$ & $0.924$ & $1.017\,[1.010,1.024]$ & $1.133\,[1.116,1.151]$ \\
  SGSC & 1.0 & $0.388$ & $0.820$ & $0.392$ & $0.914$ & $1.012\,[1.006,1.019]$ & $1.115\,[1.101,1.130]$ \\
  \midrule
  Synthetic Easy & 1.0 & $1.828$ & $2.151$ & $2.159$ & $2.870$ & $1.181\,[1.170,1.193]$ & $1.335\,[1.313,1.358]$ \\
  \midrule
  Synthetic Medium & 1.0 & $2.371$ & $3.359$ & $2.869$ & $4.134$ & $1.212\,[1.192,1.230]$ & $1.237\,[1.200,1.274]$ \\
  \midrule
  Synthetic Hard & 1.0 & $2.904$ & $4.290$ & $3.837$ & $5.482$ & $1.323\,[1.296,1.350]$ & $1.280\,[1.253,1.308]$ \\
  \midrule
  Cross-market & 1.0 & $0.04044$ & $0.06243$ & $0.03965$ & $0.06897$ & $0.981\,[0.977,0.984]$ & $1.104\,[1.089,1.120]$ \\
  \bottomrule
  \end{tabular}
\end{table*}

\FloatBarrier

\subsection{Synthetic-panel experiments}\label{app:simulation_experiments}

\paragraph{Setup.}\label{app:simulation_scenario_details}

We evaluate W-TQA and five competing methods on a synthetic panel with
\(500\) units over \(100\) time periods. Each observation has a
high-dimensional covariate vector \(X_{i,t}\in\mathbb R^{100}\) and an observed
common factor \(F_t\in\mathbb R^3\), both of which are available to the
prediction and conformal procedures. The units are partitioned into two
clusters with proportions \(88\%\) and \(12\%\), that is, 440
majority-cluster units and 60 minority-cluster units. In each repetition, the
calibration panel contains all 440 majority units and 30 minority units, while
the remaining 30 minority units are held out as test targets. The first 40 time
periods are used as burn-in data for fitting the point predictor, and the final
60 periods form the online conformal evaluation window. This split makes tail
coverage a stress test on the difficult minority part of the cross-section.

Let \(c_i\in\{A,B\}\) denote the fixed cluster label of unit \(i\), where
\(A\) is the majority cluster and \(B\) is the minority cluster. The covariate
process is designed to create both cross-sectional dependence and
cluster-specific feature distributions. Specifically, with
\(U_t\in\mathbb R^5\) denoting a global latent factor shared by all units,
\(U_{c,t}\in\mathbb R^4\) denoting a latent factor shared only by units in
cluster \(c\), and \(V_{i,t}\in\mathbb R^8\) denoting unit-level idiosyncratic
latent noise, we generate
\[
X_{i,t}
=
\mu_{c_i}
+
B U_t
+
A_{c_i} U_{c_i,t}
+
L_{c_i} V_{i,t}
+
\xi_{i,t},
\]
where \(U_t\) and \(U_{c,t}\) are Gaussian AR(1) processes started at zero, with
autoregressive coefficient \(0.65\) and innovation standard deviation \(0.85\)
for \(U_t\), and coefficients \(0.55\) for \(c=A\) and \(0.60\) for \(c=B\) with
innovation standard deviation \(0.70\) for \(U_{c,t}\),
\(V_{i,t}\sim N(0,I_8)\), and
\(\xi_{i,t}\sim N(0,\operatorname{diag}(\sigma_{c_i}^2))\). Here
\(\mu_c,\sigma_c\in\mathbb R^{100}\) are the cluster-specific mean and
coordinatewise noise scale, \(B\in\mathbb R^{100\times 5}\) loads the global
shock, \(A_c\in\mathbb R^{100\times 4}\) loads the cluster shock, and
\(L_c\in\mathbb R^{100\times 8}\) induces within-vector covariance for
idiosyncratic features. Equivalently, conditional on the global and
cluster-level shocks, a unit in cluster \(c\) has
\[
X_{i,t}\mid U_t,U_{c,t},c_i=c
\sim
N\!\left(
\mu_c+B U_t+A_c U_{c,t},
L_cL_c^\top+\operatorname{diag}(\sigma_c^2)
\right).
\]
The two clusters differ through their templates
\((\mu_c,A_c,L_c,\sigma_c)\), so they have different means, latent loadings,
and covariance structure. For each scenario, these template parameters are
deterministic functions of a fixed structure seed and are therefore identical
across repetitions; repetitions redraw the latent paths, idiosyncratic feature
noise, factor paths, minority frailty terms in the Medium and Hard scenarios,
and outcome errors. The cluster-separation parameter
controls the magnitude of the between-cluster mean shift: Easy and Medium both
use separation \(0.80\), so they share the same covariate templates, whereas
Hard uses separation \(2.0\), making the majority and minority covariate
distributions more separated.

Outcomes follow the nonlinear conditional-factor model
\[
Y_{i,t}
=
g_0(X_{i,t})
+
g_F(X_{i,t})^\top F_t
+
\eta_i F_{t,1}
+
\varepsilon_{i,t},
\]
where \(g_0:\mathbb R^{100}\to\mathbb R\) and
\(g_F:\mathbb R^{100}\to\mathbb R^3\) are the same nonlinear structural
functions for every unit and time point. Their one-hidden-layer neural-network
parameters are generated once from a fixed structure seed and reused in every
repetition, \(F_{t,1}\) is the first factor coordinate, and \(\eta_i\) is a
hidden frailty term that is nonzero only for the minority cluster in the harder
scenarios. The outcome noise is independent
\(\varepsilon_{i,t}\sim N(0,0.50^2)\). All conformal methods use the same
misspecified ridge-trained linear conditional-factor predictor,
\[
\hat f(X,F)
=
\hat g_0(X) + \hat g_F(X)^\top F,
\]
with
\begin{equation}\label{eq:sim-fitted-predictors}
\hat g_0(X)=X^\top \hat w_0,
\qquad
\hat g_F(X)=\hat W_F^\top X,
\end{equation}
where \(\hat w_0\in\mathbb R^{100}\) and
\(\hat W_F\in\mathbb R^{100\times 3}\) are estimated from the burn-in
periods using only the calibration units. Both \(X\) and \(F\) are scaled by
their calibration-burn-in standard deviations, without centering, before
fitting and prediction; \eqref{eq:sim-fitted-predictors} omits this fixed
scaling. This fitted model is misspecified relative to the nonlinear
data generating process above.

The main synthetic comparison is run in the full-feedback regime: after each
prediction, the realized target outcome is revealed immediately. This
protocol gives every method the same feedback, in particular TQA-B
and LPCI, whose updates depend on target-side feedback.
All methods use the target level \(\alpha=0.10\). For W-TQA and its ablations,
Gaussian-kernel weights use the profile metric \(\rho_t\) of
Section~\ref{sec:proposed_method} with \(h=0.6\); the synthetic profiles include the
burn-in rounds. The step size is \(\gamma=0.01\), and the level of TQA-only and
W-TQA is not clipped.

The three scenarios differ in the temporal factor process and in the
minority-only hidden frailty. On the simulation time grid \(r=0,\ldots,99\),
whose first \(40\) points are the burn-in, so that conformal round \(t\) is the
point \(r=39+t\),
the common factor satisfies \(F_0=0\) and, for \(r\ge1\),
\[
F_r
=
\phi\odot F_{r-1}
+
(1-\phi)\odot m_r
+
0.55\,s_r\,C Z_r,
\qquad
Z_r\sim N(0,I_3),
\]
where \(\phi=(0.45,0.60,0.75)\), \(C C^\top=R\), and
\(R_{jk}=0.45^{|j-k|}\). Let \(d=(1,0,-1)/\sqrt{2}\) and
\(\operatorname{expit}(u)=\{1+\exp(-u)\}^{-1}\), and define
\(\operatorname{clip}(x,0,1)=\min\{\max\{x,0\},1\}\). The scenarios are:
\emph{Easy:} \(m_r=0\), \(s_r=1\), and \(\eta_i=0\) for all units.
\emph{Medium:}
\[
m_r
=
3.60\{\operatorname{expit}(0.17(r-48))-\operatorname{expit}(-0.17\cdot48)\}d,
\qquad
s_r
=
1+0.30\,\operatorname{clip}\!\left(\frac{r-45}{55},0,1\right),
\]
with \(\eta_i=0\) for majority-cluster units and
\(\eta_i\sim N(0,0.10^2)\) for minority-cluster units. \emph{Hard:}
\[
m_r
=
5.50\{\operatorname{expit}(0.20(r-35))-\operatorname{expit}(-0.20\cdot35)\}d,
\qquad
s_r
=
1+0.80\,\operatorname{clip}\!\left(\frac{r-40}{60},0,1\right),
\]
with \(\eta_i=0\) for majority-cluster units and
\(\eta_i\sim N(0,0.15^2)\) for minority-cluster units.

\paragraph{TQA-B baseline in the simulation study.}
We implement the budgeted TQA baseline of \citet{lin2022conformal} using the
same ridge point predictor, calibration split, and symmetric interval format as
the other conformal methods. TQA-B maintains exponentially decayed mean absolute
residuals with decay \(0.8\), predicts the target's residual rank from this
history, and applies the budgeting map to choose the queried miscoverage level
\(\alpha_t^{\mathrm B}\). To keep the comparison
within the same cross-sectional panel protocol, TQA-B then runs split conformal
with this level and the current calibration residuals. Target
residual histories are updated only when the target outcome is revealed; in the
full-feedback synthetic study this happens immediately after every prediction.

\paragraph{LPCI baseline in the simulation study.}
We implement LPCI following \citet{batra2023conformal}, with the same ridge
point predictor used by the other conformal methods, fit once on
calibration-unit burn-in rows. The LPCI calibration layer is then built from
signed residual histories. Its feature vector contains a one-hot unit identifier
over all \(500\) units and six lagged-residual features, each formed by taking
a lag-\(k\) signed residual sequence and applying an exponentially weighted
moving average with smoothing parameter \(\alpha_{\mathrm{ewm}}=0.2\), equivalently
a per-step residual decay of \(0.8\). The residual quantile model is a
multi-quantile random forest with 100 trees and minimum leaf size one. The initial residual design is formed from calibration units
only. During the conformal period, calibration residual rows are appended at
every time point, while target residual rows enter the residual training buffer
only after the corresponding target outcome is revealed and can affect
subsequent rolling refits; in the full-feedback synthetic study, every target
residual is therefore appended immediately after prediction. The quantile layer
is refit every 10 test steps using a rolling window of 30 time points. Following
the LPCI adaptive asymmetry rule, six values
of \(\beta\) are evaluated on an equally spaced grid between \(0\) and
\(\alpha\). For each \(\beta\), the lower quantile is fit at \(\beta\) and the
upper quantile at \(1-\alpha+\beta\); LPCI uses the \(\beta\) that gives the
smallest predicted interval width for each sample.

\paragraph{Results.}
Table~\ref{tab:main-results-sd} gives the full-feedback results with standard
deviations, and Table~\ref{tab:simulation-sweep} the complete sweep over the
feedback probability, with one Bernoulli$(p)$ reveal draw per conformal round
shared by all targets; LPCI, whose full implementation is the most expensive
baseline, is omitted from the sweep. Under full feedback, W-TQA attains the
highest tail coverage in all three scenarios, and the gap to the next method is
largest in the Hard scenario (\(0.828\) against \(0.781\) for W-only). Its
average widths are \(2.229\), \(3.011\), and \(4.018\) from Easy to Hard,
against \(1.832\), \(2.403\), and \(2.703\) for Split CP, whose tail
coverage falls from \(0.800\) in Easy to \(0.502\) in Hard.

With $60$ rounds and $p\le0.4$, at most about $24$ target outcomes
are revealed, and W-TQA stays close to W-only: it attains the highest tail
coverage at every $p$, tied with W-only at $p=0$, where the two coincide, on
Synthetic Easy at $p=0.4$ and on Synthetic Medium at $p\in\{0.2,0.4\}$.

\begin{table}[H]
\centering
\scriptsize
\renewcommand{\arraystretch}{0.9}
\caption{Synthetic scenarios over the feedback probability $p$, mean $\pm$ sd
over 30 replications; LPCI is omitted and TQA-B is unavailable at $p=0$. Bold
as in Table~\ref{tab:main-results}.}\label{tab:simulation-sweep}
\begin{tabular}{@{}lllccccc@{}}
\toprule
Scenario & Metric & $p$ & Split CP & TQA-B & TQA-only & W-only & W-TQA \\
\midrule
Easy & Avg Cov & 0.0 & $\mathbf{0.874 \pm 0.011}$ & -- & $\mathbf{0.874 \pm 0.011}$ & $0.936 \pm 0.009$ & $0.936 \pm 0.009$ \\
 &  & 0.2 & $0.874 \pm 0.011$ & $\mathbf{0.884 \pm 0.010}$ & $0.877 \pm 0.011$ & $0.936 \pm 0.009$ & $0.935 \pm 0.009$ \\
 &  & 0.4 & $0.874 \pm 0.011$ & $\mathbf{0.885 \pm 0.009}$ & $0.877 \pm 0.009$ & $0.936 \pm 0.009$ & $0.933 \pm 0.008$ \\
 &  & 0.6 & $0.874 \pm 0.011$ & $\mathbf{0.883 \pm 0.010}$ & $0.879 \pm 0.009$ & $0.936 \pm 0.009$ & $0.931 \pm 0.008$ \\
 &  & 0.8 & $0.874 \pm 0.011$ & $\mathbf{0.884 \pm 0.009}$ & $0.881 \pm 0.008$ & $0.936 \pm 0.009$ & $0.929 \pm 0.007$ \\
 &  & 1.0 & $0.874 \pm 0.011$ & $\mathbf{0.883 \pm 0.009}$ & $0.882 \pm 0.008$ & $0.936 \pm 0.009$ & $0.928 \pm 0.008$ \\
\cmidrule(l){2-8}
 & Tail Cov & 0.0 & $0.800 \pm 0.025$ & -- & $0.800 \pm 0.025$ & $\mathbf{0.880 \pm 0.018}$ & $\mathbf{0.880 \pm 0.018}$ \\
 &  & 0.2 & $0.800 \pm 0.025$ & $0.811 \pm 0.018$ & $0.808 \pm 0.021$ & $0.880 \pm 0.018$ & $\mathbf{0.882 \pm 0.016}$ \\
 &  & 0.4 & $0.800 \pm 0.025$ & $0.820 \pm 0.018$ & $0.812 \pm 0.019$ & $\mathbf{0.880 \pm 0.018}$ & $\mathbf{0.880 \pm 0.016}$ \\
 &  & 0.6 & $0.800 \pm 0.025$ & $0.819 \pm 0.019$ & $0.817 \pm 0.020$ & $0.880 \pm 0.018$ & $\mathbf{0.881 \pm 0.017}$ \\
 &  & 0.8 & $0.800 \pm 0.025$ & $0.823 \pm 0.016$ & $0.826 \pm 0.015$ & $0.880 \pm 0.018$ & $\mathbf{0.882 \pm 0.015}$ \\
 &  & 1.0 & $0.800 \pm 0.025$ & $0.825 \pm 0.016$ & $0.832 \pm 0.011$ & $0.880 \pm 0.018$ & $\mathbf{0.884 \pm 0.014}$ \\
\cmidrule(l){2-8}
 & Avg Width & 0.0 & $1.832 \pm 0.047$ & -- & $1.832 \pm 0.047$ & $2.293 \pm 0.114$ & $2.293 \pm 0.114$ \\
 &  & 0.2 & $1.832 \pm 0.047$ & $1.934 \pm 0.067$ & $1.845 \pm 0.052$ & $2.293 \pm 0.114$ & $2.286 \pm 0.111$ \\
 &  & 0.4 & $1.832 \pm 0.047$ & $1.926 \pm 0.074$ & $1.849 \pm 0.052$ & $2.293 \pm 0.114$ & $2.268 \pm 0.103$ \\
 &  & 0.6 & $1.832 \pm 0.047$ & $1.914 \pm 0.057$ & $1.860 \pm 0.055$ & $2.293 \pm 0.114$ & $2.260 \pm 0.103$ \\
 &  & 0.8 & $1.832 \pm 0.047$ & $1.912 \pm 0.059$ & $1.869 \pm 0.056$ & $2.293 \pm 0.114$ & $2.242 \pm 0.098$ \\
 &  & 1.0 & $1.832 \pm 0.047$ & $1.909 \pm 0.055$ & $1.875 \pm 0.056$ & $2.293 \pm 0.114$ & $2.229 \pm 0.094$ \\
\cmidrule(l){2-8}
 & Width CoV & 0.0 & $0.108 \pm 0.035$ & -- & $0.108 \pm 0.035$ & $\mathbf{0.214 \pm 0.029}$ & $\mathbf{0.214 \pm 0.029}$ \\
 &  & 0.2 & $0.108 \pm 0.035$ & $0.189 \pm 0.027$ & $0.111 \pm 0.034$ & $0.214 \pm 0.029$ & $\mathbf{0.217 \pm 0.029}$ \\
 &  & 0.4 & $0.108 \pm 0.035$ & $0.180 \pm 0.031$ & $0.113 \pm 0.033$ & $0.214 \pm 0.029$ & $\mathbf{0.219 \pm 0.030}$ \\
 &  & 0.6 & $0.108 \pm 0.035$ & $0.175 \pm 0.027$ & $0.115 \pm 0.033$ & $0.214 \pm 0.029$ & $\mathbf{0.223 \pm 0.030}$ \\
 &  & 0.8 & $0.108 \pm 0.035$ & $0.173 \pm 0.026$ & $0.117 \pm 0.031$ & $0.214 \pm 0.029$ & $\mathbf{0.223 \pm 0.030}$ \\
 &  & 1.0 & $0.108 \pm 0.035$ & $0.172 \pm 0.026$ & $0.118 \pm 0.031$ & $0.214 \pm 0.029$ & $\mathbf{0.225 \pm 0.031}$ \\
\midrule
Medium & Avg Cov & 0.0 & $0.848 \pm 0.030$ & -- & $0.848 \pm 0.030$ & $\mathbf{0.937 \pm 0.010}$ & $\mathbf{0.937 \pm 0.010}$ \\
 &  & 0.2 & $0.848 \pm 0.030$ & $\mathbf{0.864 \pm 0.021}$ & $0.854 \pm 0.027$ & $0.937 \pm 0.010$ & $\mathbf{0.936 \pm 0.009}$ \\
 &  & 0.4 & $0.848 \pm 0.030$ & $0.864 \pm 0.020$ & $0.855 \pm 0.028$ & $0.937 \pm 0.010$ & $\mathbf{0.934 \pm 0.009}$ \\
 &  & 0.6 & $0.848 \pm 0.030$ & $0.865 \pm 0.021$ & $0.858 \pm 0.025$ & $0.937 \pm 0.010$ & $\mathbf{0.932 \pm 0.008}$ \\
 &  & 0.8 & $0.848 \pm 0.030$ & $0.867 \pm 0.020$ & $0.861 \pm 0.022$ & $0.937 \pm 0.010$ & $\mathbf{0.930 \pm 0.007}$ \\
 &  & 1.0 & $0.848 \pm 0.030$ & $0.867 \pm 0.019$ & $0.864 \pm 0.021$ & $0.937 \pm 0.010$ & $\mathbf{0.929 \pm 0.008}$ \\
\cmidrule(l){2-8}
 & Tail Cov & 0.0 & $0.752 \pm 0.053$ & -- & $0.752 \pm 0.053$ & $\mathbf{0.872 \pm 0.026}$ & $\mathbf{0.872 \pm 0.026}$ \\
 &  & 0.2 & $0.752 \pm 0.053$ & $0.779 \pm 0.037$ & $0.763 \pm 0.048$ & $\mathbf{0.872 \pm 0.026}$ & $\mathbf{0.872 \pm 0.025}$ \\
 &  & 0.4 & $0.752 \pm 0.053$ & $0.783 \pm 0.032$ & $0.769 \pm 0.047$ & $\mathbf{0.872 \pm 0.026}$ & $\mathbf{0.872 \pm 0.024}$ \\
 &  & 0.6 & $0.752 \pm 0.053$ & $0.791 \pm 0.031$ & $0.780 \pm 0.040$ & $0.872 \pm 0.026$ & $\mathbf{0.875 \pm 0.022}$ \\
 &  & 0.8 & $0.752 \pm 0.053$ & $0.800 \pm 0.028$ & $0.793 \pm 0.031$ & $0.872 \pm 0.026$ & $\mathbf{0.876 \pm 0.020}$ \\
 &  & 1.0 & $0.752 \pm 0.053$ & $0.803 \pm 0.026$ & $0.804 \pm 0.029$ & $0.872 \pm 0.026$ & $\mathbf{0.880 \pm 0.017}$ \\
\cmidrule(l){2-8}
 & Avg Width & 0.0 & $2.403 \pm 0.173$ & -- & $2.403 \pm 0.173$ & $3.100 \pm 0.195$ & $3.100 \pm 0.195$ \\
 &  & 0.2 & $2.403 \pm 0.173$ & $2.542 \pm 0.167$ & $2.434 \pm 0.170$ & $3.100 \pm 0.195$ & $3.092 \pm 0.188$ \\
 &  & 0.4 & $2.403 \pm 0.173$ & $2.537 \pm 0.178$ & $2.443 \pm 0.169$ & $3.100 \pm 0.195$ & $3.067 \pm 0.189$ \\
 &  & 0.6 & $2.403 \pm 0.173$ & $2.527 \pm 0.155$ & $2.465 \pm 0.165$ & $3.100 \pm 0.195$ & $3.050 \pm 0.191$ \\
 &  & 0.8 & $2.403 \pm 0.173$ & $2.540 \pm 0.158$ & $2.483 \pm 0.164$ & $3.100 \pm 0.195$ & $3.026 \pm 0.185$ \\
 &  & 1.0 & $2.403 \pm 0.173$ & $2.539 \pm 0.161$ & $2.501 \pm 0.166$ & $3.100 \pm 0.195$ & $3.011 \pm 0.184$ \\
\cmidrule(l){2-8}
 & Width CoV & 0.0 & $0.228 \pm 0.044$ & -- & $0.228 \pm 0.044$ & $\mathbf{0.275 \pm 0.032}$ & $\mathbf{0.275 \pm 0.032}$ \\
 &  & 0.2 & $0.228 \pm 0.044$ & $0.264 \pm 0.033$ & $0.232 \pm 0.043$ & $0.275 \pm 0.032$ & $\mathbf{0.276 \pm 0.031}$ \\
 &  & 0.4 & $0.228 \pm 0.044$ & $0.259 \pm 0.033$ & $0.233 \pm 0.042$ & $\mathbf{0.275 \pm 0.032}$ & $\mathbf{0.275 \pm 0.031}$ \\
 &  & 0.6 & $0.228 \pm 0.044$ & $0.258 \pm 0.031$ & $0.237 \pm 0.040$ & $0.275 \pm 0.032$ & $\mathbf{0.276 \pm 0.030}$ \\
 &  & 0.8 & $0.228 \pm 0.044$ & $0.262 \pm 0.032$ & $0.240 \pm 0.039$ & $\mathbf{0.275 \pm 0.032}$ & $\mathbf{0.275 \pm 0.030}$ \\
 &  & 1.0 & $0.228 \pm 0.044$ & $0.263 \pm 0.031$ & $0.243 \pm 0.038$ & $\mathbf{0.275 \pm 0.032}$ & $\mathbf{0.275 \pm 0.031}$ \\
\midrule
Hard & Avg Cov & 0.0 & $0.746 \pm 0.044$ & -- & $0.746 \pm 0.044$ & $\mathbf{0.929 \pm 0.021}$ & $\mathbf{0.929 \pm 0.021}$ \\
 &  & 0.2 & $0.746 \pm 0.044$ & $0.809 \pm 0.025$ & $0.763 \pm 0.039$ & $0.929 \pm 0.021$ & $\mathbf{0.923 \pm 0.019}$ \\
 &  & 0.4 & $0.746 \pm 0.044$ & $0.812 \pm 0.024$ & $0.774 \pm 0.034$ & $0.929 \pm 0.021$ & $\mathbf{0.923 \pm 0.018}$ \\
 &  & 0.6 & $0.746 \pm 0.044$ & $0.813 \pm 0.026$ & $0.788 \pm 0.033$ & $0.929 \pm 0.021$ & $\mathbf{0.922 \pm 0.017}$ \\
 &  & 0.8 & $0.746 \pm 0.044$ & $0.818 \pm 0.022$ & $0.804 \pm 0.024$ & $0.929 \pm 0.021$ & $\mathbf{0.923 \pm 0.016}$ \\
 &  & 1.0 & $0.746 \pm 0.044$ & $0.819 \pm 0.021$ & $0.815 \pm 0.021$ & $0.929 \pm 0.021$ & $\mathbf{0.922 \pm 0.015}$ \\
\cmidrule(l){2-8}
 & Tail Cov & 0.0 & $0.502 \pm 0.105$ & -- & $0.502 \pm 0.105$ & $\mathbf{0.781 \pm 0.089}$ & $\mathbf{0.781 \pm 0.089}$ \\
 &  & 0.2 & $0.502 \pm 0.105$ & $0.664 \pm 0.055$ & $0.549 \pm 0.090$ & $0.781 \pm 0.089$ & $\mathbf{0.787 \pm 0.073}$ \\
 &  & 0.4 & $0.502 \pm 0.105$ & $0.686 \pm 0.056$ & $0.588 \pm 0.069$ & $0.781 \pm 0.089$ & $\mathbf{0.797 \pm 0.061}$ \\
 &  & 0.6 & $0.502 \pm 0.105$ & $0.693 \pm 0.054$ & $0.639 \pm 0.066$ & $0.781 \pm 0.089$ & $\mathbf{0.808 \pm 0.054}$ \\
 &  & 0.8 & $0.502 \pm 0.105$ & $0.709 \pm 0.050$ & $0.692 \pm 0.041$ & $0.781 \pm 0.089$ & $\mathbf{0.820 \pm 0.047}$ \\
 &  & 1.0 & $0.502 \pm 0.105$ & $0.713 \pm 0.048$ & $0.729 \pm 0.033$ & $0.781 \pm 0.089$ & $\mathbf{0.828 \pm 0.042}$ \\
\cmidrule(l){2-8}
 & Avg Width & 0.0 & $2.703 \pm 0.188$ & -- & $2.703 \pm 0.188$ & $4.194 \pm 0.339$ & $4.194 \pm 0.339$ \\
 &  & 0.2 & $2.703 \pm 0.188$ & $3.112 \pm 0.249$ & $2.800 \pm 0.199$ & $4.194 \pm 0.339$ & $4.075 \pm 0.330$ \\
 &  & 0.4 & $2.703 \pm 0.188$ & $3.116 \pm 0.253$ & $2.857 \pm 0.201$ & $4.194 \pm 0.339$ & $4.062 \pm 0.323$ \\
 &  & 0.6 & $2.703 \pm 0.188$ & $3.114 \pm 0.229$ & $2.947 \pm 0.219$ & $4.194 \pm 0.339$ & $4.044 \pm 0.319$ \\
 &  & 0.8 & $2.703 \pm 0.188$ & $3.142 \pm 0.250$ & $3.052 \pm 0.236$ & $4.194 \pm 0.339$ & $4.033 \pm 0.314$ \\
 &  & 1.0 & $2.703 \pm 0.188$ & $3.142 \pm 0.249$ & $3.120 \pm 0.250$ & $4.194 \pm 0.339$ & $4.018 \pm 0.310$ \\
\cmidrule(l){2-8}
 & Width CoV & 0.0 & $0.217 \pm 0.053$ & -- & $0.217 \pm 0.053$ & $\mathbf{0.263 \pm 0.047}$ & $\mathbf{0.263 \pm 0.047}$ \\
 &  & 0.2 & $0.217 \pm 0.053$ & $\mathbf{0.286 \pm 0.035}$ & $0.226 \pm 0.052$ & $0.263 \pm 0.047$ & $0.265 \pm 0.046$ \\
 &  & 0.4 & $0.217 \pm 0.053$ & $\mathbf{0.275 \pm 0.032}$ & $0.238 \pm 0.047$ & $0.263 \pm 0.047$ & $0.266 \pm 0.045$ \\
 &  & 0.6 & $0.217 \pm 0.053$ & $\mathbf{0.276 \pm 0.035}$ & $0.255 \pm 0.049$ & $0.263 \pm 0.047$ & $0.268 \pm 0.045$ \\
 &  & 0.8 & $0.217 \pm 0.053$ & $\mathbf{0.278 \pm 0.037}$ & $0.277 \pm 0.046$ & $0.263 \pm 0.047$ & $0.271 \pm 0.044$ \\
 &  & 1.0 & $0.217 \pm 0.053$ & $0.279 \pm 0.038$ & $\mathbf{0.288 \pm 0.045}$ & $0.263 \pm 0.047$ & $0.273 \pm 0.044$ \\
\bottomrule
\end{tabular}
\end{table}

\FloatBarrier

\subsection{Real-data panels}\label{app:real-data-panels}\label{app:real}

\paragraph{Data and protocol.}
This subsection covers HF, M5 and SGSC, which differ in sampling
frequency, response scale, and cross-sectional heterogeneity.  The HF hourly panel is a
U.S. equity panel from the second half of 2024; the target is the next-hour
return, the processed panel runs from 13:00 on 2 July 2024 to 15:00 on
31 December 2024, burn-in ends at 13:00 on 7 August 2024, and each replication
splits the 495 tickers into 395 calibration tickers and 100 test
tickers, followed over a conformal period of $500$ hourly rounds.  The M5 panel uses the
\texttt{CA\_1}/\texttt{FOODS\_3} retail-sales slice, with target
$Y_{i,t+1}=\log(1+\mathrm{sales}_{i,t+1})$, where $\log$ is the natural
logarithm throughout; after preprocessing it contains 648
items, split into 454 calibration items and 194 test items, with burn-in ending on
31 December 2013 and a conformal period of $600$ days.  The SGSC panel uses residential
electricity consumption from the Smart-Grid Smart-City project, restricted to
plain-load households with no solar generation or controlled load; the target is
$Y_{i,t+1}=\log(1+\mathrm{GENERAL\_SUPPLY\_KWH}_{i,t+1})$, and each replication
splits 400 eligible households into 280 calibration and 120 test households, followed over
a conformal period of $600$ half-hour rounds after a burn-in window of $3984$ half-hours ending on
17 March 2013.

\paragraph{Covariates.}
Table~\ref{tab:real-features} lists the covariates, which are the same for
calibration and test units. Each is observed at round $t$ whether or not the
target's outcome is revealed: exogenous variables of round $t$, statistics of
the unit's burn-in history, during which every unit is fully observed, and
means of the round-$t$ calibration outcomes, taken over the calibration units
other than the unit itself. No covariate contains an outcome of the target
after the burn-in, so a round with $R_t=0$ affects only $\alpha_t$ and the
target-side states of TQA-B and LPCI. Here $\hat\beta_i$ is the burn-in
least-squares slope of unit $i$'s outcome on its peer mean. The profiles use
only covariates that describe the unit itself; the calendar variables and the
peer means are essentially common to all units and so do not distinguish one
unit from another.

\begin{table}[H]
\centering
\footnotesize
\setlength{\tabcolsep}{3pt}
\renewcommand{\arraystretch}{1.2}
\caption{Covariates of the real panels. Burn-in statistics are computed on the
burn-in rows of each unit; peer means average the round-$t$ outcomes of the
calibration units other than the unit itself. Profile marks the covariates that
enter the profile.}\label{tab:real-features}
\begin{tabular*}{\linewidth}{@{\extracolsep{\fill}}llc@{}}
\toprule
Panel & Covariate & Profile \\
\midrule
HF & hour of day (sine, cosine) & \\
\cdashline{2-3}[0.6pt/2pt]
 & peer mean; peer mean within the Yahoo Finance sector; $\hat\beta_i\times$ peer mean & \\
\cdashline{2-3}[0.6pt/2pt]
 & sector indicators & $\checkmark$ \\
\cdashline{2-3}[0.6pt/2pt]
 & burn-in s.d.\ of $Y$, mean $|Y|$, mean $\log(1+\text{volume})$, mean (high$-$low)/close, $\hat\beta_i$ & $\checkmark$ \\
\midrule
M5 & day of week, month (sine, cosine); SNAP; event day; 7- and 28-day price change & \\
\cdashline{2-3}[0.6pt/2pt]
 & log price; price relative to the item's burn-in median; item age & $\checkmark$ \\
\cdashline{2-3}[0.6pt/2pt]
 & peer mean; $\hat\beta_i\times$ peer mean & \\
\cdashline{2-3}[0.6pt/2pt]
 & burn-in mean, s.d.\ and zero share of $Y$; $\hat\beta_i$ & $\checkmark$ \\
\midrule
SGSC & time of day, day of week, month (sine, cosine); weekend & \\
\cdashline{2-3}[0.6pt/2pt]
 & 31 household attributes & $\checkmark$ \\
\cdashline{2-3}[0.6pt/2pt]
 & peer mean; peer mean within the climate zone; $\hat\beta_i\times$ peer mean & \\
\cdashline{2-3}[0.6pt/2pt]
 & burn-in mean, s.d.\ and zero share of $Y$; $\hat\beta_i$ & $\checkmark$ \\
\cdashline{2-3}[0.6pt/2pt]
 & burn-in mean of $Y$ at the round's half-hour slot & \\
\bottomrule
\end{tabular*}
\end{table}

\paragraph{Data access and licenses.}
HF, M5 and SGSC are built from publicly downloadable source data.  For the
HF hourly panel, we use the HF Data Library clean U.S. equity data, selecting the
\texttt{Clean} version and \texttt{Hourly} timeframe from the download/API portal
\url{https://hfdatalibrary.com/pages/download}.  Because the HF Data Library may
revise historical files, the HF hourly panel used in our experiments is
available from the authors.  The HF Data Library states that
the data are licensed under Creative Commons Attribution 4.0 (CC BY 4.0), and we
cite the data set as \citet{elkassabgi2026hfdata} for attribution.

For M5, the raw files \texttt{sales\_train\_evaluation.csv},
\texttt{calendar.csv}, and \texttt{sell\_prices.csv} are from the public M5
Forecasting--Accuracy competition data page,
\url{https://www.kaggle.com/competitions/m5-forecasting-accuracy/data}.  The
official M5 methods repository documents the competition data set contents,
including unit sales, calendar, promotion, price, and scoring files,
\url{https://github.com/Mcompetitions/M5-methods}.  Kaggle's competition
metadata lists the M5 data license as \emph{Subject to Competition Rules}; users
should access the data through Kaggle and follow the competition rules at
\url{https://www.kaggle.com/competitions/m5-forecasting-accuracy/rules}.

For SGSC, we use the Australian Government Smart-Grid Smart-City
Customer Trial Data, \url{https://www.data.gov.au/data/dataset/smart-grid-smart-city-customer-trial-data}.
The specific resources used are \emph{Electricity Use Interval Reading},
\emph{Customer Household Data}, and \emph{Customer Trial Data Dictionary}.
Data.gov.au lists the
data set and these resources under Creative Commons Attribution 3.0 Australia
(CC BY 3.0 AU).

All real-data runs use a global standardized ridge predictor with ridge penalty
10 on every panel, set without tuning,
fit only on burn-in rows from calibration units and then held fixed during the
conformal period.  The target level is \(\alpha=0.10\) and the
step size is \(\gamma=0.01\) for all three panels, and the level of TQA-only and
W-TQA is not clipped. W-only and W-TQA use Gaussian-kernel
weights in the profile metric \(\rho_t\) of Section~\ref{sec:proposed_method}
with \(h=0.6\), with profiles started at the first conformal round.  For the
missing-feedback sweeps, a single Bernoulli reveal indicator with success probability $p$ is
drawn in each round and shared by all test units; when it equals one, all held-out test outcomes of that
round are revealed to the adaptive methods after prediction, and otherwise
they are withheld.
W-TQA and its two ablations run separately for each test unit, so a unit's
results depend on the reveal draws only through its own reveal sequence, whose
law is the same as under independent draws for each unit.
Calibration-unit outcomes remain fully observed throughout,
so the calibration panel refreshes in every round.

\paragraph{Real-data LPCI variant.}
The method labeled LPCI in the real-data and MCAR feedback-sweep tables is
a lighter LPCI variant used on the longer panels.  It differs from the
random-forest LPCI baseline of the synthetic study in two ways.  First, it replaces
the random-forest residual quantile estimator with independent linear pinball
regressions, which makes repeated runs feasible on conformal periods of \(500\)--\(600\)
rounds.  Second, rather than selecting among six \(\beta\)-indexed asymmetric
intervals, it fixes the LPCI asymmetry parameter at \(\beta=\alpha/2\), i.e.,
the lower and upper residual quantiles are \(0.05\) and \(0.95\).

Within each replication, this variant shares the same fitted point predictor as the
conformal methods.  Its residual quantile layer keeps the LPCI residual-history
structure: six EWMA-smoothed lagged signed-residual features, one for each lag
\(k=1,\ldots,6\), constructed by applying shift-\(k\) followed by
\(\operatorname{ewm}(\alpha_{\mathrm{ewm}}=0.2)\), and a lightweight unit fixed
effect obtained by partialing out each unit's rolling mean signed residual and
adding that mean back to the fitted quantiles.  The linear quantile regressions
use 3000 AdaGrad iterations with learning rate \(0.05\) and \(\ell_2\) penalty
\(10^{-4}\), enough for the fitted quantiles to converge on all three panels,
with at most 25,000 fitting rows.

The fixed tails replace the per-sample minimum-width selection over the
\(\beta\)-grid, which, with independently fitted lower and upper quantiles,
selects intervals that are too narrow wherever the two fits are noisy or
cross; the random-forest implementation of the synthetic study fits the
quantiles jointly within one forest.

Calibration residual rows are added in every conformal round, while a target
unit's residual history and unit-mean state are updated only when its feedback is
revealed.  The LPCI residual layer is refit every 10 conformal rounds on HF,
M5 and SGSC, using the latest 240 hours on HF hourly, the latest 280
days on M5, and the latest 500 half-hours on SGSC.  Since this variant requires
target-side feedback to update target states, it is not available at \(p=0\),
which the tables mark with a dash.

With full feedback, this variant attains average coverage close to \(1-\alpha\)
on all three panels
(Tables~\ref{tab:hf-hourly-feedback-metrics}--\ref{tab:sgsc-feedback-metrics}),
so the lighter quantile layer does not by itself reduce its average coverage.

\paragraph{Additional mechanism plots on M5 and SGSC.}
Figure~\ref{fig:m5-sgsc-mechanisms-app} repeats the two mechanism views from
Figure~\ref{fig:hf_mechanisms} on M5 and SGSC; all curves are means over the
$30$ splits with $\pm1$ sd bands.

\paragraph{Full missing-feedback sweeps.}
Tables~\ref{tab:hf-hourly-feedback-metrics}, \ref{tab:m5-feedback-metrics}, and
\ref{tab:sgsc-feedback-metrics} give the full repeated sweeps over feedback
probability.  The sweeps show the same pattern as the summary tables. With target
feedback, the level update lifts W-TQA above W-only;
when feedback is scarce, the spatial weights alone keep the tail coverage of W-TQA
above that of the unweighted methods on all three panels. At \(p=0\), W-TQA
coincides with W-only and TQA-only coincides with Split CP, while TQA-B and
LPCI are unavailable because they require target feedback.
In the real-data tables below, bolding is applied separately within each panel,
feedback probability, and metric row: average-coverage and
tail-coverage entries closest to \(0.90\) are bolded, while Width CoV entries are
bolded when largest. Tied reported means are bolded together.

\begin{figure}[H]
  \centering
  \textbf{M5 panel}\\[-0.25em]
  \begin{minipage}[t]{0.48\linewidth}
    \centering
    \includegraphics[width=\linewidth]{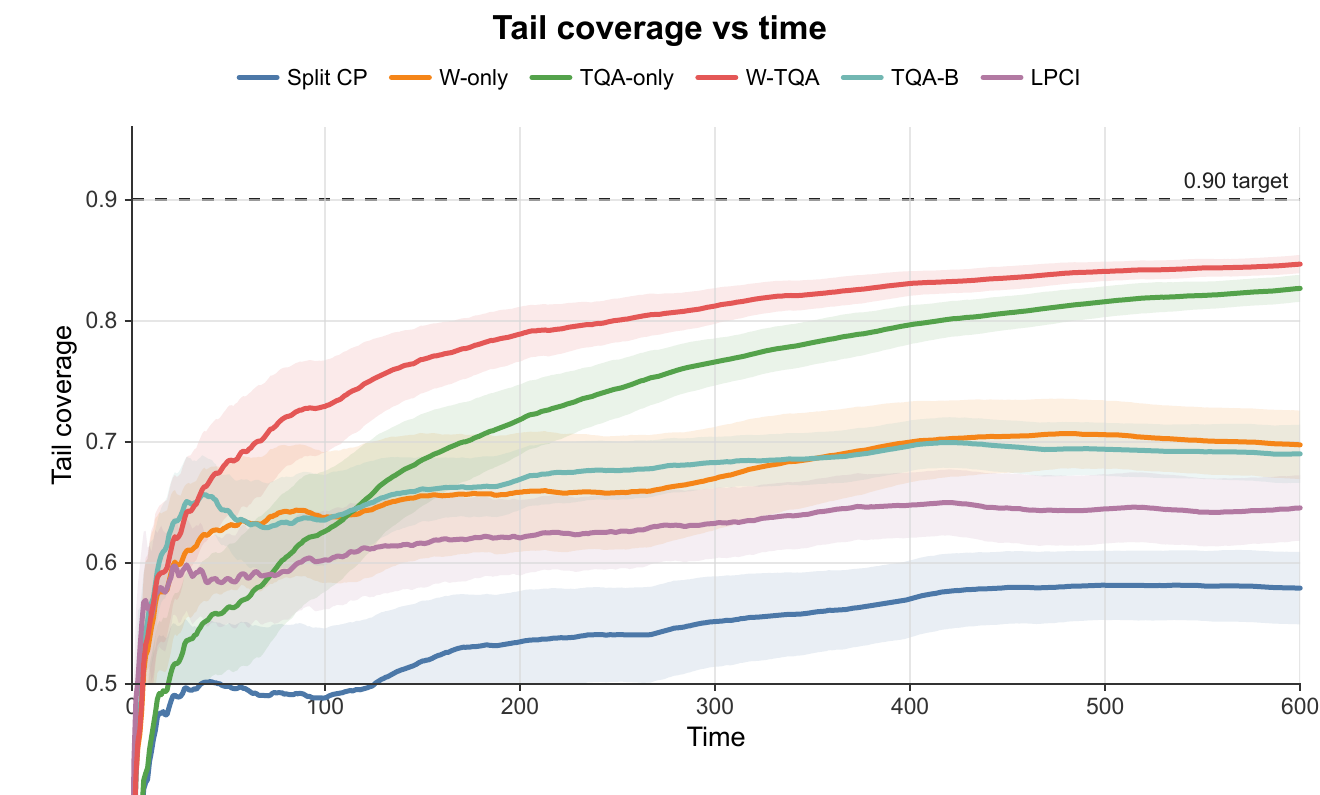}
  \end{minipage}\hfill
  \begin{minipage}[t]{0.48\linewidth}
    \centering
    \includegraphics[width=\linewidth]{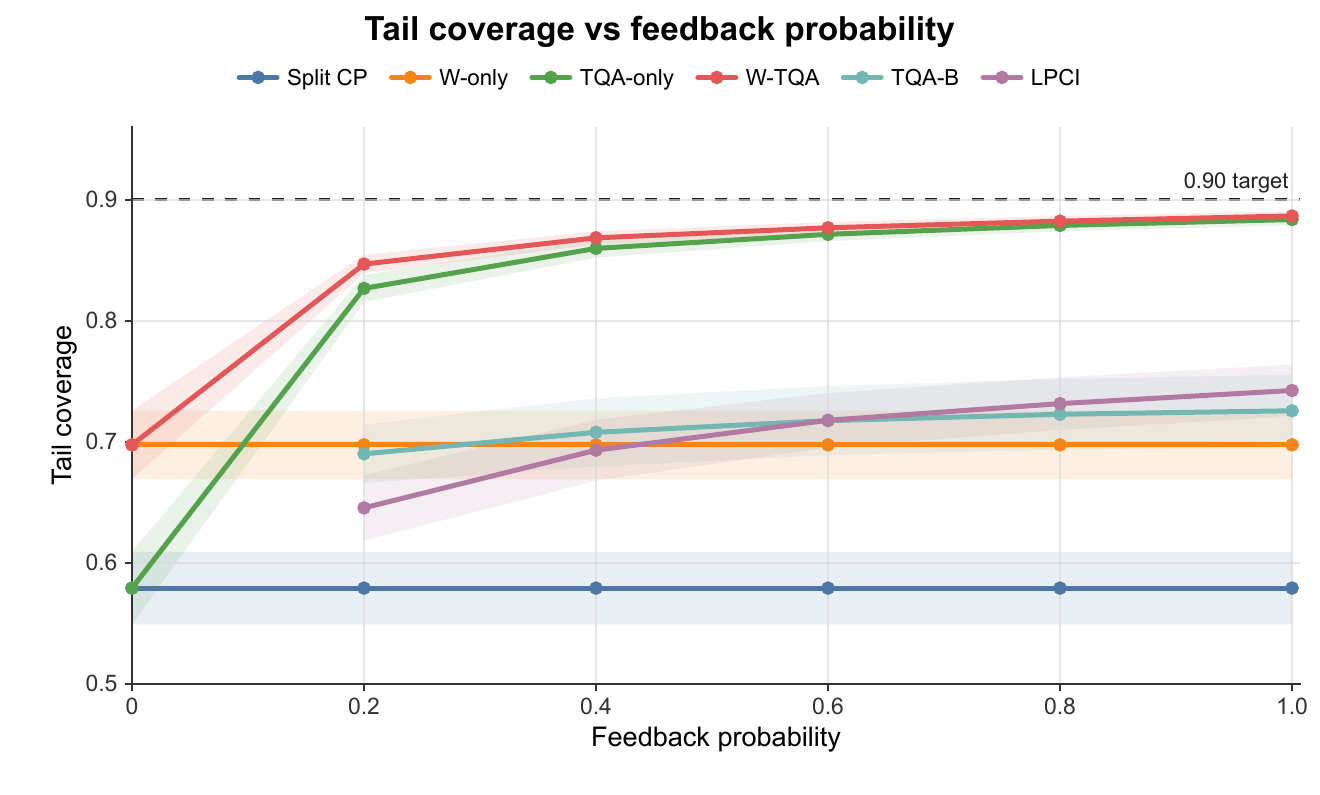}
  \end{minipage}

  \vspace{0.8em}
  \textbf{SGSC panel}\\[-0.25em]
  \begin{minipage}[t]{0.48\linewidth}
    \centering
    \includegraphics[width=\linewidth]{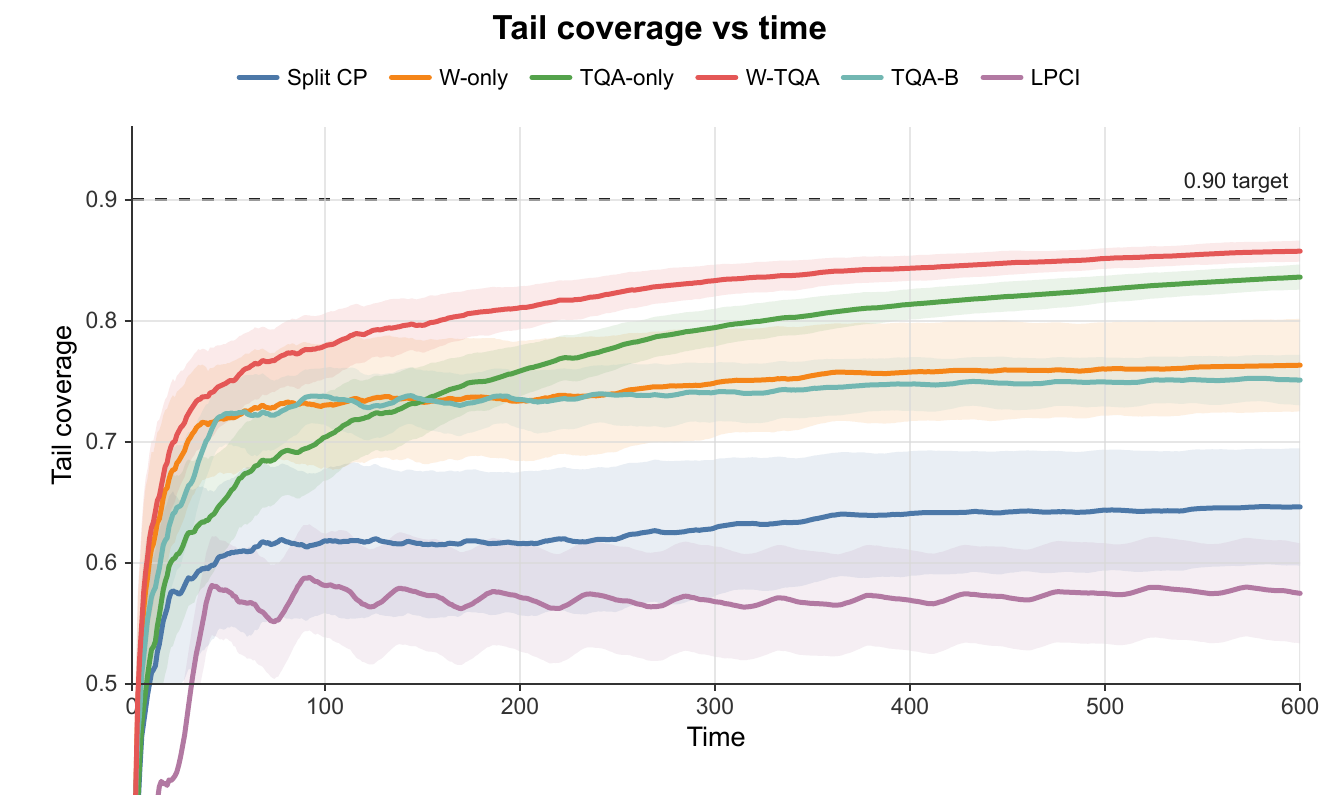}
  \end{minipage}\hfill
  \begin{minipage}[t]{0.48\linewidth}
    \centering
    \includegraphics[width=\linewidth]{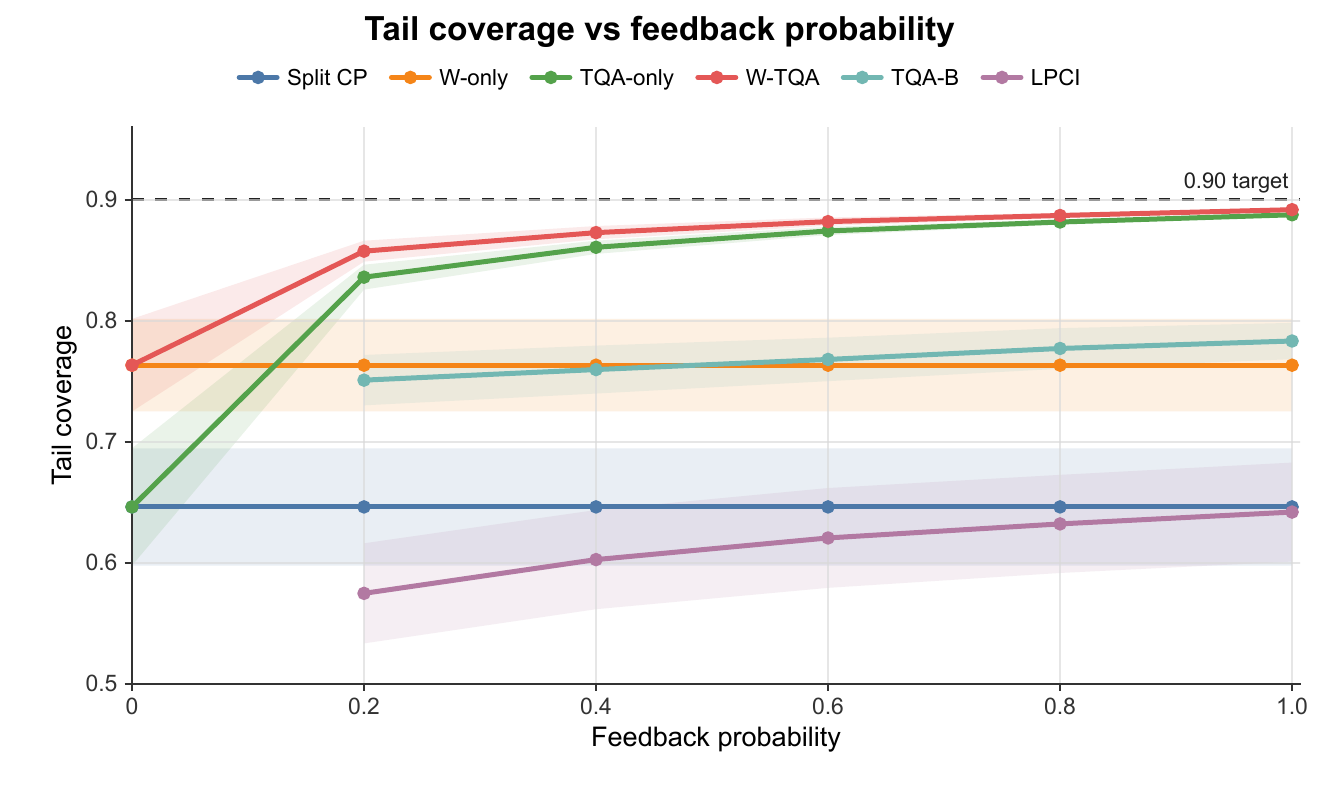}
  \end{minipage}
  \caption{Mechanism plots on M5 and SGSC, matching the layout of the HF hourly
  Figure~\ref{fig:hf_mechanisms}. Left panels show cumulative tail coverage at
  $p=0.2$; right panels show tail coverage versus target-feedback probability
  \(p\). Means over the $30$ splits with $\pm1$ sd bands.}\label{fig:m5-sgsc-mechanisms-app}
\end{figure}

\begin{table}[H]
  \centering
  \footnotesize
  \setlength{\tabcolsep}{2.5pt}
  \caption{Complete HF hourly missing-feedback sweep across average coverage,
  tail coverage, average width, and Width CoV. Entries are mean $\pm$ sd over
  30 random replications.}\label{tab:hf-hourly-feedback-metrics}
  \resizebox{\ifdim\width>\linewidth\linewidth\else\width\fi}{!}{%
  \begin{tabular}{llcccccc}
\toprule
Metric & $p$ & Split CP & TQA-B & LPCI & TQA-only & W-only & W-TQA \\
\midrule
    \multirow{6}{*}{Avg Cov} & 0.0 & $\mathbf{0.904 \pm 0.015}$ & -- & -- & $\mathbf{0.904 \pm 0.015}$ & $0.927 \pm 0.009$ & $0.927 \pm 0.009$ \\
      & 0.2 & $0.904 \pm 0.015$ & $0.914 \pm 0.011$ & $\mathbf{0.898 \pm 0.013}$ & $0.922 \pm 0.007$ & $0.927 \pm 0.009$ & $0.926 \pm 0.004$ \\
      & 0.4 & $0.904 \pm 0.015$ & $0.915 \pm 0.011$ & $\mathbf{0.898 \pm 0.013}$ & $0.921 \pm 0.004$ & $0.927 \pm 0.009$ & $0.921 \pm 0.003$ \\
      & 0.6 & $0.904 \pm 0.015$ & $0.915 \pm 0.010$ & $\mathbf{0.898 \pm 0.013}$ & $0.917 \pm 0.003$ & $0.927 \pm 0.009$ & $0.916 \pm 0.002$ \\
      & 0.8 & $0.904 \pm 0.015$ & $0.916 \pm 0.011$ & $\mathbf{0.898 \pm 0.013}$ & $0.914 \pm 0.003$ & $0.927 \pm 0.009$ & $0.913 \pm 0.002$ \\
      & 1.0 & $0.904 \pm 0.015$ & $0.916 \pm 0.011$ & $\mathbf{0.899 \pm 0.013}$ & $0.912 \pm 0.002$ & $0.927 \pm 0.009$ & $0.911 \pm 0.002$ \\
    \midrule
    \multirow{6}{*}{Tail Cov} & 0.0 & $0.636 \pm 0.040$ & -- & -- & $0.636 \pm 0.040$ & $\mathbf{0.772 \pm 0.024}$ & $\mathbf{0.772 \pm 0.024}$ \\
      & 0.2 & $0.636 \pm 0.040$ & $0.713 \pm 0.022$ & $0.665 \pm 0.043$ & $0.813 \pm 0.015$ & $0.772 \pm 0.024$ & $\mathbf{0.848 \pm 0.011}$ \\
      & 0.4 & $0.636 \pm 0.040$ & $0.720 \pm 0.021$ & $0.666 \pm 0.042$ & $0.854 \pm 0.008$ & $0.772 \pm 0.024$ & $\mathbf{0.868 \pm 0.006}$ \\
      & 0.6 & $0.636 \pm 0.040$ & $0.724 \pm 0.017$ & $0.668 \pm 0.041$ & $0.870 \pm 0.006$ & $0.772 \pm 0.024$ & $\mathbf{0.880 \pm 0.005}$ \\
      & 0.8 & $0.636 \pm 0.040$ & $0.723 \pm 0.018$ & $0.669 \pm 0.041$ & $0.879 \pm 0.003$ & $0.772 \pm 0.024$ & $\mathbf{0.886 \pm 0.002}$ \\
      & 1.0 & $0.636 \pm 0.040$ & $0.723 \pm 0.018$ & $0.670 \pm 0.041$ & $0.884 \pm 0.001$ & $0.772 \pm 0.024$ & $\mathbf{0.890 \pm 0.001}$ \\
    \midrule
    \multirow{6}{*}{Avg Width} & 0.0 & $0.01925 \pm 0.00032$ & -- & -- & $0.01925 \pm 0.00032$ & $0.02141 \pm 0.00106$ & $0.02141 \pm 0.00106$ \\
      & 0.2 & $0.01925 \pm 0.00032$ & $0.01983 \pm 0.00024$ & $0.01890 \pm 0.00030$ & $0.02017 \pm 0.00060$ & $0.02141 \pm 0.00106$ & $0.02084 \pm 0.00103$ \\
      & 0.4 & $0.01925 \pm 0.00032$ & $0.01981 \pm 0.00015$ & $0.01891 \pm 0.00030$ & $0.02016 \pm 0.00087$ & $0.02141 \pm 0.00106$ & $0.02042 \pm 0.00113$ \\
      & 0.6 & $0.01925 \pm 0.00032$ & $0.01983 \pm 0.00015$ & $0.01891 \pm 0.00029$ & $0.01992 \pm 0.00096$ & $0.02141 \pm 0.00106$ & $0.02013 \pm 0.00115$ \\
      & 0.8 & $0.01925 \pm 0.00032$ & $0.01979 \pm 0.00011$ & $0.01890 \pm 0.00030$ & $0.01979 \pm 0.00104$ & $0.02141 \pm 0.00106$ & $0.01995 \pm 0.00114$ \\
      & 1.0 & $0.01925 \pm 0.00032$ & $0.01975 \pm 0.00010$ & $0.01890 \pm 0.00030$ & $0.01966 \pm 0.00108$ & $0.02141 \pm 0.00106$ & $0.01981 \pm 0.00119$ \\
    \midrule
    \multirow{6}{*}{Width CoV} & 0.0 & $0.568 \pm 0.006$ & -- & -- & $0.568 \pm 0.006$ & $\mathbf{1.002 \pm 0.183}$ & $\mathbf{1.002 \pm 0.183}$ \\
      & 0.2 & $0.568 \pm 0.006$ & $0.688 \pm 0.053$ & $0.087 \pm 0.009$ & $0.901 \pm 0.109$ & $1.002 \pm 0.183$ & $\mathbf{1.068 \pm 0.131}$ \\
      & 0.4 & $0.568 \pm 0.006$ & $0.681 \pm 0.038$ & $0.086 \pm 0.008$ & $1.003 \pm 0.111$ & $1.002 \pm 0.183$ & $\mathbf{1.101 \pm 0.128}$ \\
      & 0.6 & $0.568 \pm 0.006$ & $0.685 \pm 0.038$ & $0.086 \pm 0.010$ & $1.052 \pm 0.103$ & $1.002 \pm 0.183$ & $\mathbf{1.132 \pm 0.115}$ \\
      & 0.8 & $0.568 \pm 0.006$ & $0.673 \pm 0.031$ & $0.085 \pm 0.009$ & $1.080 \pm 0.106$ & $1.002 \pm 0.183$ & $\mathbf{1.150 \pm 0.110}$ \\
      & 1.0 & $0.568 \pm 0.006$ & $0.662 \pm 0.027$ & $0.085 \pm 0.008$ & $1.107 \pm 0.108$ & $1.002 \pm 0.183$ & $\mathbf{1.163 \pm 0.113}$ \\
    \bottomrule
\end{tabular}%
  }
\end{table}

\begin{table}[H]
  \centering
  \footnotesize
  \setlength{\tabcolsep}{2.5pt}
  \caption{Complete M5 missing-feedback sweep across average coverage, tail coverage,
  average width, and Width CoV. Entries are mean $\pm$ sd over 30 random replications.}\label{tab:m5-feedback-metrics}
  \resizebox{\ifdim\width>\linewidth\linewidth\else\width\fi}{!}{%
  \begin{tabular}{llcccccc}
\toprule
Metric & $p$ & Split CP & TQA-B & LPCI & TQA-only & W-only & W-TQA \\
\midrule
    \multirow{6}{*}{Avg Cov} & 0.0 & $\mathbf{0.904 \pm 0.011}$ & -- & -- & $\mathbf{0.904 \pm 0.011}$ & $0.913 \pm 0.007$ & $0.913 \pm 0.007$ \\
      & 0.2 & $\mathbf{0.904 \pm 0.011}$ & $0.920 \pm 0.008$ & $0.878 \pm 0.007$ & $0.931 \pm 0.004$ & $0.913 \pm 0.007$ & $0.925 \pm 0.003$ \\
      & 0.4 & $\mathbf{0.904 \pm 0.011}$ & $0.924 \pm 0.008$ & $0.890 \pm 0.007$ & $0.928 \pm 0.003$ & $0.913 \pm 0.007$ & $0.922 \pm 0.002$ \\
      & 0.6 & $0.904 \pm 0.011$ & $0.926 \pm 0.008$ & $\mathbf{0.897 \pm 0.006}$ & $0.923 \pm 0.003$ & $0.913 \pm 0.007$ & $0.918 \pm 0.002$ \\
      & 0.8 & $0.904 \pm 0.011$ & $0.928 \pm 0.008$ & $\mathbf{0.900 \pm 0.006}$ & $0.919 \pm 0.002$ & $0.913 \pm 0.007$ & $0.916 \pm 0.002$ \\
      & 1.0 & $0.904 \pm 0.011$ & $0.928 \pm 0.008$ & $\mathbf{0.903 \pm 0.006}$ & $0.917 \pm 0.002$ & $0.913 \pm 0.007$ & $0.913 \pm 0.001$ \\
    \midrule
    \multirow{6}{*}{Tail Cov} & 0.0 & $0.579 \pm 0.030$ & -- & -- & $0.579 \pm 0.030$ & $\mathbf{0.697 \pm 0.028}$ & $\mathbf{0.697 \pm 0.028}$ \\
      & 0.2 & $0.579 \pm 0.030$ & $0.690 \pm 0.024$ & $0.645 \pm 0.027$ & $0.827 \pm 0.011$ & $0.697 \pm 0.028$ & $\mathbf{0.847 \pm 0.008}$ \\
      & 0.4 & $0.579 \pm 0.030$ & $0.708 \pm 0.028$ & $0.693 \pm 0.025$ & $0.860 \pm 0.008$ & $0.697 \pm 0.028$ & $\mathbf{0.869 \pm 0.005}$ \\
      & 0.6 & $0.579 \pm 0.030$ & $0.717 \pm 0.029$ & $0.718 \pm 0.023$ & $0.871 \pm 0.006$ & $0.697 \pm 0.028$ & $\mathbf{0.877 \pm 0.005}$ \\
      & 0.8 & $0.579 \pm 0.030$ & $0.723 \pm 0.029$ & $0.732 \pm 0.022$ & $0.879 \pm 0.005$ & $0.697 \pm 0.028$ & $\mathbf{0.882 \pm 0.004}$ \\
      & 1.0 & $0.579 \pm 0.030$ & $0.726 \pm 0.030$ & $0.742 \pm 0.021$ & $0.884 \pm 0.005$ & $0.697 \pm 0.028$ & $\mathbf{0.887 \pm 0.004}$ \\
    \midrule
    \multirow{6}{*}{Avg Width} & 0.0 & $2.380 \pm 0.029$ & -- & -- & $2.380 \pm 0.029$ & $2.358 \pm 0.044$ & $2.358 \pm 0.044$ \\
      & 0.2 & $2.380 \pm 0.029$ & $2.394 \pm 0.018$ & $1.870 \pm 0.011$ & $2.346 \pm 0.018$ & $2.358 \pm 0.044$ & $2.292 \pm 0.043$ \\
      & 0.4 & $2.380 \pm 0.029$ & $2.397 \pm 0.017$ & $1.870 \pm 0.011$ & $2.262 \pm 0.024$ & $2.358 \pm 0.044$ & $2.223 \pm 0.043$ \\
      & 0.6 & $2.380 \pm 0.029$ & $2.399 \pm 0.018$ & $1.870 \pm 0.011$ & $2.202 \pm 0.028$ & $2.358 \pm 0.044$ & $2.176 \pm 0.043$ \\
      & 0.8 & $2.380 \pm 0.029$ & $2.401 \pm 0.018$ & $1.870 \pm 0.011$ & $2.164 \pm 0.031$ & $2.358 \pm 0.044$ & $2.144 \pm 0.045$ \\
      & 1.0 & $2.380 \pm 0.029$ & $2.401 \pm 0.018$ & $1.870 \pm 0.011$ & $2.135 \pm 0.032$ & $2.358 \pm 0.044$ & $2.120 \pm 0.044$ \\
    \midrule
    \multirow{6}{*}{Width CoV} & 0.0 & $0.076 \pm 0.003$ & -- & -- & $0.076 \pm 0.003$ & $\mathbf{0.332 \pm 0.024}$ & $\mathbf{0.332 \pm 0.024}$ \\
      & 0.2 & $0.076 \pm 0.003$ & $0.118 \pm 0.008$ & $0.034 \pm 0.002$ & $0.233 \pm 0.025$ & $0.332 \pm 0.024$ & $\mathbf{0.358 \pm 0.023}$ \\
      & 0.4 & $0.076 \pm 0.003$ & $0.123 \pm 0.008$ & $0.033 \pm 0.002$ & $0.279 \pm 0.018$ & $0.332 \pm 0.024$ & $\mathbf{0.372 \pm 0.021}$ \\
      & 0.6 & $0.076 \pm 0.003$ & $0.126 \pm 0.008$ & $0.032 \pm 0.002$ & $0.304 \pm 0.021$ & $0.332 \pm 0.024$ & $\mathbf{0.381 \pm 0.020}$ \\
      & 0.8 & $0.076 \pm 0.003$ & $0.127 \pm 0.008$ & $0.032 \pm 0.002$ & $0.321 \pm 0.020$ & $0.332 \pm 0.024$ & $\mathbf{0.388 \pm 0.020}$ \\
      & 1.0 & $0.076 \pm 0.003$ & $0.128 \pm 0.008$ & $0.032 \pm 0.002$ & $0.334 \pm 0.020$ & $0.332 \pm 0.024$ & $\mathbf{0.394 \pm 0.020}$ \\
    \bottomrule
\end{tabular}%
  }
\end{table}

\begin{table}[H]
  \centering
  \footnotesize
  \setlength{\tabcolsep}{2.5pt}
  \caption{Complete SGSC missing-feedback sweep across average coverage, tail coverage,
  average width, and Width CoV. Entries are mean $\pm$ sd over 30 random replications.}\label{tab:sgsc-feedback-metrics}
  \resizebox{\ifdim\width>\linewidth\linewidth\else\width\fi}{!}{%
  \begin{tabular}{llcccccc}
\toprule
Metric & $p$ & Split CP & TQA-B & LPCI & TQA-only & W-only & W-TQA \\
\midrule
    \multirow{6}{*}{Avg Cov} & 0.0 & $\mathbf{0.902 \pm 0.014}$ & -- & -- & $\mathbf{0.902 \pm 0.014}$ & $0.933 \pm 0.010$ & $0.933 \pm 0.010$ \\
      & 0.2 & $\mathbf{0.902 \pm 0.014}$ & $0.916 \pm 0.010$ & $0.864 \pm 0.018$ & $0.924 \pm 0.006$ & $0.933 \pm 0.010$ & $0.935 \pm 0.005$ \\
      & 0.4 & $\mathbf{0.902 \pm 0.014}$ & $0.918 \pm 0.010$ & $0.877 \pm 0.017$ & $0.923 \pm 0.005$ & $0.933 \pm 0.010$ & $0.929 \pm 0.005$ \\
      & 0.6 & $\mathbf{0.902 \pm 0.014}$ & $0.920 \pm 0.010$ & $0.887 \pm 0.016$ & $0.921 \pm 0.004$ & $0.933 \pm 0.010$ & $0.924 \pm 0.003$ \\
      & 0.8 & $\mathbf{0.902 \pm 0.014}$ & $0.922 \pm 0.010$ & $0.893 \pm 0.016$ & $0.918 \pm 0.003$ & $0.933 \pm 0.010$ & $0.920 \pm 0.003$ \\
      & 1.0 & $0.902 \pm 0.014$ & $0.923 \pm 0.009$ & $\mathbf{0.899 \pm 0.015}$ & $0.915 \pm 0.003$ & $0.933 \pm 0.010$ & $0.917 \pm 0.003$ \\
    \midrule
    \multirow{6}{*}{Tail Cov} & 0.0 & $0.646 \pm 0.048$ & -- & -- & $0.646 \pm 0.048$ & $\mathbf{0.763 \pm 0.038}$ & $\mathbf{0.763 \pm 0.038}$ \\
      & 0.2 & $0.646 \pm 0.048$ & $0.751 \pm 0.021$ & $0.575 \pm 0.041$ & $0.836 \pm 0.010$ & $0.763 \pm 0.038$ & $\mathbf{0.858 \pm 0.009}$ \\
      & 0.4 & $0.646 \pm 0.048$ & $0.760 \pm 0.020$ & $0.603 \pm 0.041$ & $0.861 \pm 0.005$ & $0.763 \pm 0.038$ & $\mathbf{0.873 \pm 0.005}$ \\
      & 0.6 & $0.646 \pm 0.048$ & $0.768 \pm 0.018$ & $0.620 \pm 0.041$ & $0.874 \pm 0.003$ & $0.763 \pm 0.038$ & $\mathbf{0.882 \pm 0.003}$ \\
      & 0.8 & $0.646 \pm 0.048$ & $0.777 \pm 0.017$ & $0.632 \pm 0.041$ & $0.882 \pm 0.002$ & $0.763 \pm 0.038$ & $\mathbf{0.887 \pm 0.002}$ \\
      & 1.0 & $0.646 \pm 0.048$ & $0.783 \pm 0.015$ & $0.642 \pm 0.041$ & $0.888 \pm 0.001$ & $0.763 \pm 0.038$ & $\mathbf{0.892 \pm 0.001}$ \\
    \midrule
    \multirow{6}{*}{Avg Width} & 0.0 & $0.474 \pm 0.012$ & -- & -- & $0.474 \pm 0.012$ & $0.614 \pm 0.024$ & $0.614 \pm 0.024$ \\
      & 0.2 & $0.474 \pm 0.012$ & $0.489 \pm 0.007$ & $0.393 \pm 0.014$ & $0.485 \pm 0.010$ & $0.614 \pm 0.024$ & $0.555 \pm 0.021$ \\
      & 0.4 & $0.474 \pm 0.012$ & $0.489 \pm 0.007$ & $0.393 \pm 0.014$ & $0.475 \pm 0.014$ & $0.614 \pm 0.024$ & $0.528 \pm 0.025$ \\
      & 0.6 & $0.474 \pm 0.012$ & $0.489 \pm 0.007$ & $0.393 \pm 0.014$ & $0.466 \pm 0.016$ & $0.614 \pm 0.024$ & $0.509 \pm 0.023$ \\
      & 0.8 & $0.474 \pm 0.012$ & $0.490 \pm 0.007$ & $0.393 \pm 0.014$ & $0.459 \pm 0.018$ & $0.614 \pm 0.024$ & $0.495 \pm 0.023$ \\
      & 1.0 & $0.474 \pm 0.012$ & $0.492 \pm 0.006$ & $0.393 \pm 0.014$ & $0.453 \pm 0.018$ & $0.614 \pm 0.024$ & $0.485 \pm 0.022$ \\
    \midrule
    \multirow{6}{*}{Width CoV} & 0.0 & $0.346 \pm 0.012$ & -- & -- & $0.346 \pm 0.012$ & $\mathbf{0.566 \pm 0.035}$ & $\mathbf{0.566 \pm 0.035}$ \\
      & 0.2 & $0.346 \pm 0.012$ & $0.366 \pm 0.016$ & $0.209 \pm 0.021$ & $0.489 \pm 0.031$ & $0.566 \pm 0.035$ & $\mathbf{0.621 \pm 0.034}$ \\
      & 0.4 & $0.346 \pm 0.012$ & $0.379 \pm 0.016$ & $0.208 \pm 0.020$ & $0.567 \pm 0.025$ & $0.566 \pm 0.035$ & $\mathbf{0.676 \pm 0.032}$ \\
      & 0.6 & $0.346 \pm 0.012$ & $0.390 \pm 0.017$ & $0.206 \pm 0.020$ & $0.610 \pm 0.027$ & $0.566 \pm 0.035$ & $\mathbf{0.707 \pm 0.032}$ \\
      & 0.8 & $0.346 \pm 0.012$ & $0.401 \pm 0.018$ & $0.206 \pm 0.020$ & $0.640 \pm 0.025$ & $0.566 \pm 0.035$ & $\mathbf{0.723 \pm 0.033}$ \\
      & 1.0 & $0.346 \pm 0.012$ & $0.409 \pm 0.019$ & $0.206 \pm 0.021$ & $0.660 \pm 0.025$ & $0.566 \pm 0.035$ & $\mathbf{0.735 \pm 0.032}$ \\
    \bottomrule
\end{tabular}%
  }
\end{table}

\FloatBarrier
\subsection{Cross-market panel}\label{app:cross-market-panel}

\paragraph{Data and protocol.}
The cross-market panel uses daily open, high, low, close and volume data for the
S\&P~500 constituents listed on 26 September 2026 and the TOPIX~500 (TOPIX
Core30, Large70 and Mid400) constituents of the TOPIX weight file of 31 July
2026, adjusted for splits and dividends. Each round is a Tokyo trading session
$d$. Its calibration slice is the latest U.S.\ session dated on or before the
calendar day preceding $d$; that session closes at 16:00 ET, which is 05:00 or
06:00 JST, before the Tokyo open at 09:00 JST. The target outcome of round $d$
is realized at the Tokyo close (15:00 JST, and 15:30 JST from 5 November 2024)
and enters the filtration before round $d+1$, so every round reveals its target
outcomes. A Tokyo session whose latest preceding U.S.\ session is that of the
previous Tokyo session, as after a U.S.\ holiday, is skipped. The sample runs
from 19 July 2023 to 25 September 2026, with $28$ sessions skipped, and contains
$494$ U.S.\ and $486$ Tokyo stocks.
The main window has a $250$-round burn-in, from 19 July 2023 to 5 August 2024,
and a $100$-round conformal period, from 6 August 2024 to 8 January 2025. Each
replication draws $100$ Tokyo stocks as test units
and $386$ U.S.\ stocks as calibration units.

For stock $i$ of market $m$ in session $s$, let $O_{i,s}$, $H_{i,s}$, $L_{i,s}$,
$C_{i,s}$ and $V_{i,s}$ be the open, high, low, close and volume, and let
$\sigma_{m,s}$ be the cross-sectional median, over the constituents of market
$m$, of the standard deviation of the $20$ open-to-close log returns before
session $s$. The outcome is $c\log(C_{i,s}/O_{i,s})/\sigma_{m,s}$, where the
constant $c$ is the median of these market scales $\sigma_{m,s}$ over the
stock-sessions of the main-window burn-in in both markets, so that the
outcome is in return units. Because $\sigma_{m,s}$ is known at the open of session
$s$, an interval for the outcome maps to an interval for the open-to-close
return with the same coverage. The features are computed identically in both
markets from information available at the open of session $s$, and those that
are returns are divided by $\sigma_{m,s}$: the overnight gap
$\log(O_{i,s}/C_{i,s-1})$ and its absolute value; the two previous open-to-close
returns, the absolute value of the previous one, and the mean of the five previous
ones; the momentum $\log(C_{i,s-1}/C_{i,s-21})$; the standard deviations of the
five and the $20$ previous open-to-close returns; the mean absolute gap over the
$20$ previous sessions; the previous range $\log(H_{i,s-1}/L_{i,s-1})$ and its
mean over five sessions; $\log(1+V_{i,s-1})$ minus its mean over the $20$
previous sessions; and $11$ sector indicators from the Yahoo Finance sector
field. All of these features enter the profile. The predictor, the six methods
and all constants are those of the HF
hourly panel, including the lighter LPCI variant with its HF refit settings.

\paragraph{Synchronous control.}
We also compare the asynchronous timing of this panel with synchronous
timing, in which the calibration and target outcomes are realized together and
the target outcome is withheld until after prediction. To this end, every replication is repeated
with the same $100$ Tokyo test units and, as calibration units, the $386$ Tokyo
stocks that are not test units, whose outcomes in session $d$ are realized
together with those of the targets. All constants are unchanged, and the
predictor is refit on the burn-in rows of the new calibration units.

\paragraph{Data access and terms.}
Prices and sector labels are downloaded from Yahoo Finance with the open-source
\texttt{yfinance} package (Apache-2.0 license), which is not affiliated with
Yahoo. Yahoo's terms of use govern the downloaded data, so we release only the
download and preprocessing scripts, not the data.
The constituent lists are the Wikipedia list of S\&P~500 companies (text under
CC BY-SA 4.0), \url{https://en.wikipedia.org/wiki/List_of_S%26P_500_companies},
and the TOPIX weight file published by the Japan Exchange Group,
\url{https://www.jpx.co.jp/automation/markets/indices/topix/files/topixweight_j.csv}.

\paragraph{Results.}
Table~\ref{tab:cross-market-results} reports the main window at $p=1$, which
forms the Cross column of Table~\ref{tab:main-results}, together with
the synchronous control. With either calibration panel, W-TQA has the highest
tail coverage and exceeds both ablations in all $30$ splits; its margins over
TQA-only and W-only are $0.033$ and $0.032$ with the U.S.\ calibration units and
$0.041$ and $0.032$ with the Tokyo calibration units, so the lead of W-TQA is
similar under the two timings. With the U.S.\ calibration
units, the median interval width of W-TQA is no larger than that of TQA-only in
$29$ of the $30$ splits (Table~\ref{tab:width-quantiles}), and W-TQA falls back on $3.0\%$ of unit-rounds. Table~\ref{tab:cross-market-windows} restarts the
$100$-round window at four later dates, refitting the predictor on the $250$
preceding rounds and resetting all online states; W-TQA again exceeds both
ablations in all $30$ splits of every window and configuration.

\begin{table}[H]
  \centering
  \footnotesize
  \captionsetup{skip=5pt}
  \setlength{\tabcolsep}{4pt}
  \renewcommand{\arraystretch}{0.95}
  \caption{Cross-market panel at $p=1$ in the main window,
  with U.S.\ calibration units (Async.) and with same-session Tokyo calibration
  units (Sync.) for the same targets. Mean~$\pm$~sd over $30$ replications; bold
  as in Table~\ref{tab:main-results}.}\label{tab:cross-market-results}
  \begin{tabular}{@{}llcc@{}}
  \toprule
  & Method & Async. & Sync.\\
  \midrule
  \multirow{6}{*}{\rotatebox{90}{Avg.\ cov.}} & Split CP & $\mathbf{0.904 \pm 0.007}$ & $\mathbf{0.900 \pm 0.009}$ \\
   & TQA-B & $0.905 \pm 0.006$ & $0.906 \pm 0.007$ \\
   & LPCI & $\mathbf{0.904 \pm 0.007}$ & $0.863 \pm 0.009$ \\
   & TQA-only & $0.905 \pm 0.004$ & $0.910 \pm 0.004$ \\
   & W-only & $0.915 \pm 0.004$ & $0.921 \pm 0.005$ \\
   & W-TQA & $0.907 \pm 0.003$ & $0.917 \pm 0.003$ \\
  \midrule
  \multirow{6}{*}{\rotatebox{90}{Tail cov.}} & Split CP & $0.727 \pm 0.024$ & $0.693 \pm 0.030$ \\
   & TQA-B & $0.763 \pm 0.017$ & $0.765 \pm 0.013$ \\
   & LPCI & $0.721 \pm 0.022$ & $0.670 \pm 0.022$ \\
   & TQA-only & $0.824 \pm 0.005$ & $0.826 \pm 0.008$ \\
   & W-only & $0.826 \pm 0.012$ & $0.835 \pm 0.011$ \\
   & W-TQA & $\mathbf{0.858 \pm 0.006}$ & $\mathbf{0.867 \pm 0.005}$ \\
  \midrule
  \multirow{5}{*}{\rotatebox{90}{\shortstack{Rel.\\width}}} & TQA-B & $1.001 \pm 0.005$ & $1.016 \pm 0.008$ \\
   & LPCI & $0.978 \pm 0.005$ & $0.973 \pm 0.004$ \\
   & TQA-only & $0.995 \pm 0.014$ & $1.023 \pm 0.019$ \\
   & W-only & $1.083 \pm 0.036$ & $1.115 \pm 0.042$ \\
   & W-TQA & $1.056 \pm 0.039$ & $1.090 \pm 0.042$ \\
  \midrule
  \multirow{6}{*}{\rotatebox{90}{Width CoV}} & Split CP & $0.310 \pm 0.007$ & $0.335 \pm 0.004$ \\
   & TQA-B & $0.340 \pm 0.008$ & $0.366 \pm 0.007$ \\
   & LPCI & $0.099 \pm 0.005$ & $0.126 \pm 0.004$ \\
   & TQA-only & $0.370 \pm 0.014$ & $0.420 \pm 0.024$ \\
   & W-only & $0.566 \pm 0.052$ & $0.618 \pm 0.049$ \\
   & W-TQA & $\mathbf{0.607 \pm 0.047}$ & $\mathbf{0.639 \pm 0.044}$ \\
  \bottomrule
  \end{tabular}
\end{table}

\begin{table}[H]
  \centering
  \footnotesize
  \captionsetup{skip=5pt}
  \setlength{\tabcolsep}{3pt}
  \caption{Tail coverage on the cross-market panel at $p=1$ over $100$-round
  windows starting on 6 August 2024, 9 January 2025, 16 June 2025, 14 November
  2025 and 22 April 2026, each with the predictor refit on the $250$ preceding
  rounds and all online states reset. Means over $30$
  replications.}\label{tab:cross-market-windows}
  \begin{tabular}{@{}llccccc@{}}
  \toprule
  & Method & Aug.\ 2024 & Jan.\ 2025 & Jun.\ 2025 & Nov.\ 2025 & Apr.\ 2026\\
  \midrule
  \multirow{3}{*}{Async.} & TQA-only & $0.824$ & $0.818$ & $0.820$ & $0.822$ & $0.828$ \\
   & W-only & $0.826$ & $0.776$ & $0.823$ & $0.819$ & $0.829$ \\
   & W-TQA & $0.858$ & $0.845$ & $0.860$ & $0.857$ & $0.866$ \\
  \midrule
  \multirow{3}{*}{Sync.} & TQA-only & $0.826$ & $0.827$ & $0.826$ & $0.827$ & $0.825$ \\
   & W-only & $0.835$ & $0.831$ & $0.830$ & $0.833$ & $0.817$ \\
   & W-TQA & $0.867$ & $0.866$ & $0.865$ & $0.866$ & $0.860$ \\
  \bottomrule
  \end{tabular}
\end{table}
\subsection{Robustness checks}\label{app:robustness}

This subsection varies three ingredients of Section~\ref{sec:experiments}: the
reveal mechanism, which is made outcome-informative and so violates
Assumption~\ref{ass:mcar}; the hyperparameters $(h,\gamma)$; and the base
predictor.

\subsubsection{Outcome-informative (non-MCAR) feedback}\label{app:nonmcar-feedback-exp}

The MCAR feedback sweeps of Appendix~\ref{app:real-data-panels} vary the amount
of target feedback while keeping the reveal process independent of the current
target outcome.  To stress-test the methods when this assumption fails, we run
an informative-feedback experiment
on HF, M5 and SGSC.  For each replication and each conformal round
\(t\), we first compute a target-side difficulty score using the fixed burn-in
predictor,
\[
d_t=\frac{1}{|\mathcal I_t|}
\sum_{i\in\mathcal I_t}|Y_{i,t}-\widehat f(X_{i,t})|,
\]
where \(\mathcal I_t\) is the set of held-out test units at round \(t\).  We
convert these difficulties to a centered rank score \(z_t\in[-1,1]\), ranked
over the conformal rounds of the replication, with larger
values corresponding to more difficult rounds.  Let
\(\operatorname{expit}(u)=\{1+\exp(-u)\}^{-1}\) denote the logistic sigmoid.  We reveal
target feedback with probability
\[
p_t^{\mathrm{hard}}=\operatorname{expit}(2z_t),
\qquad
p_t^{\mathrm{easy}}=\operatorname{expit}(-2z_t).
\]
The hard-visible mechanism therefore makes high-residual rounds more likely
to be revealed, while the easy-visible mechanism makes them more likely to be
hidden.  Because \(p_t\) is a function of the current target outcome through
\(d_t\), both mechanisms deliberately violate Assumption~\ref{ass:mcar}.  The
average reveal probability is \(0.5\) by symmetry of the rank scores, but the
revealed subset is informative: across 30 replications, the average correlation
between \(p_t\) and \(z_t\) is approximately \(+0.9975\) for hard-visible feedback
and \(-0.9975\) for easy-visible feedback on each of these panels.

Table~\ref{tab:nonmcar-informative-feedback} reports tail coverage under the two
informative feedback mechanisms.  W-TQA remains the strongest tail-coverage method
in all six panel--mechanism combinations.  Its average coverage
is \(0.917/0.920\) on HF hourly, \(0.915/0.923\) on M5, and
\(0.924/0.925\) on SGSC for easy-visible/hard-visible feedback, respectively.
The corresponding average widths are \(0.0201/0.0204\), \(2.178/2.214\), and
\(0.508/0.524\), on the native scales of the three panels.  The combined
spatial--temporal update thus remains stable when the feedback process is
outcome-informative.

\begin{table}[H]
  \centering
  \footnotesize
  \setlength{\tabcolsep}{3pt}
  \caption{Tail coverage under informative non-MCAR feedback.  The reveal
  propensity is a logistic function of the rank-normalized target residual
  difficulty \(z_t\), so the feedback process is outcome-informative by design.
  Entries are mean $\pm$ sd over 30 random replications. Bold as in
  Table~\ref{tab:main-results}.}\label{tab:nonmcar-informative-feedback}
  \resizebox{\ifdim\width>\linewidth\linewidth\else\width\fi}{!}{%
  \begin{tabular}{llrrrrrr}
\toprule
Panel & Mechanism & Split CP & TQA-B & LPCI & TQA-only & W-only & W-TQA \\
\midrule
    HF hourly & easy-visible & $0.636 \pm 0.040$ & $0.682 \pm 0.025$ & $0.670 \pm 0.041$ & $0.859 \pm 0.006$ & $0.772 \pm 0.024$ & $\mathbf{0.870 \pm 0.006}$ \\
    HF hourly & hard-visible & $0.636 \pm 0.040$ & $0.760 \pm 0.014$ & $0.664 \pm 0.042$ & $0.863 \pm 0.004$ & $0.772 \pm 0.024$ & $\mathbf{0.875 \pm 0.005}$ \\
    M5 & easy-visible & $0.579 \pm 0.030$ & $0.698 \pm 0.028$ & $0.696 \pm 0.026$ & $0.853 \pm 0.007$ & $0.697 \pm 0.028$ & $\mathbf{0.861 \pm 0.005}$ \\
    M5 & hard-visible & $0.579 \pm 0.030$ & $0.726 \pm 0.028$ & $0.713 \pm 0.022$ & $0.870 \pm 0.008$ & $0.697 \pm 0.028$ & $\mathbf{0.876 \pm 0.005}$ \\
    SGSC & easy-visible & $0.646 \pm 0.048$ & $0.746 \pm 0.022$ & $0.606 \pm 0.038$ & $0.851 \pm 0.008$ & $0.763 \pm 0.038$ & $\mathbf{0.865 \pm 0.006}$ \\
    SGSC & hard-visible & $0.646 \pm 0.048$ & $0.787 \pm 0.014$ & $0.610 \pm 0.042$ & $0.850 \pm 0.013$ & $0.763 \pm 0.038$ & $\mathbf{0.862 \pm 0.012}$ \\
    \bottomrule
\end{tabular}%
  }
\end{table}

\FloatBarrier
\raggedbottom

\subsubsection{Hyperparameters}\label{app:parameter-robustness}

\paragraph{Sensitivity-slice design.}
We test whether the real-data conclusions depend on a finely tuned kernel
bandwidth $h$ or step size $\gamma$.  The sweep fixes the default used in
the main experiments at $(h,\gamma)=(0.60,0.010)$ and uses the predeclared grids
$h\in\{0.30,0.45,0.60,0.90,1.20\}$ and
$\gamma\in\{0.005,0.010,0.020,0.040,0.080\}$.
Each grid setting is evaluated over 30 random replications.
We report HF hourly, M5 and SGSC under both
full feedback ($p=1$) and sparse MCAR
target feedback ($p=0.2$).  The $h$ slices fix $\gamma=0.010$ and compare
W-TQA with W-only, isolating the effect of adding temporal adaptation to the
weighted calibration pool.  The $\gamma$ slices fix $h=0.60$ and compare W-TQA
with TQA-only, isolating the effect of adding spatial weighting to temporal
adaptation.

\paragraph{Results.}
Figures~\ref{fig:param_full} and~\ref{fig:param_sparse} plot mean
tail coverage over the 30 random replications. W-TQA exceeds the corresponding
one-branch ablation at every grid point except at the smallest bandwidth
\(h=0.3\) on SGSC, where W-only exceeds it by \(0.012\) under full feedback
and \(0.006\) under sparse feedback; at the largest step size
\(\gamma=0.08\), TQA-only is within \(0.001\) of it under full feedback. All
curves vary smoothly in \(h\) and \(\gamma\).
W-TQA's tail coverage changes far less across the bandwidth grid than W-only's,
which decreases steeply as \(h\) grows, and its margin over TQA-only shrinks as
\(\gamma\) grows and the level of TQA-only adapts faster. The same pattern holds
under sparse feedback (\(p=0.2\)).

\begin{figure}[H]
  \centering
  \includegraphics[width=\linewidth]{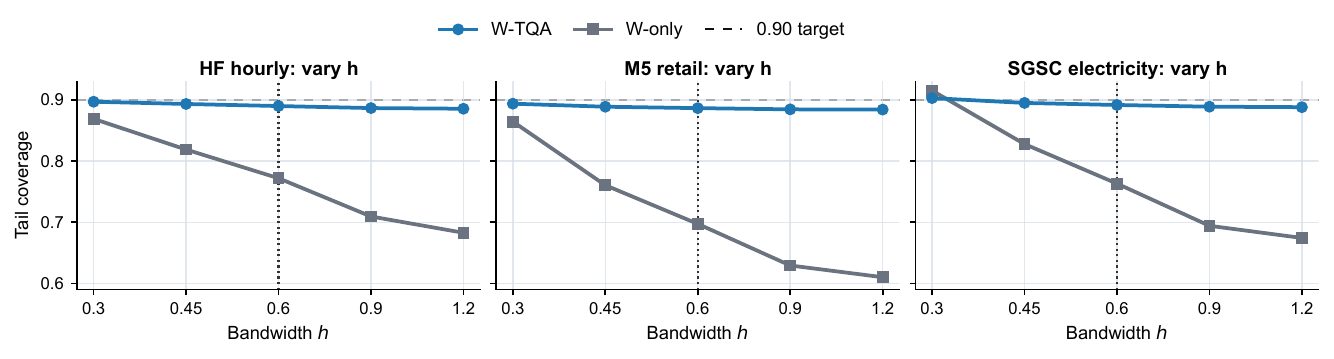}
  \vspace{-0.8em}
  \includegraphics[width=\linewidth]{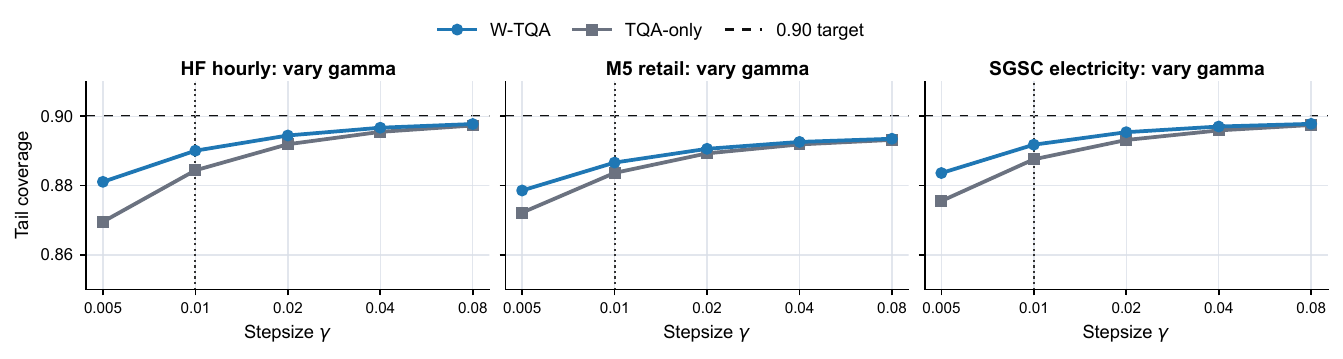}
  \caption{Parameter-sensitivity slices under full target feedback ($p=1$).
  Each panel reports mean tail coverage over 30 random replications.  The dotted
  vertical line marks the default used in the main experiments; the dashed
  horizontal line marks $1-\alpha=0.90$.}\label{fig:param_full}
\end{figure}

\begin{figure}[H]
  \centering
  \includegraphics[width=\linewidth]{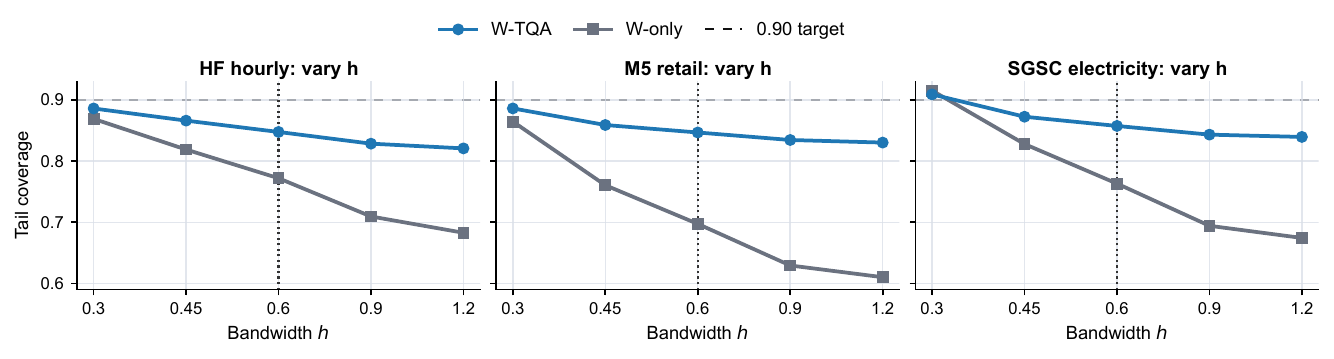}
  \vspace{-0.8em}
  \includegraphics[width=\linewidth]{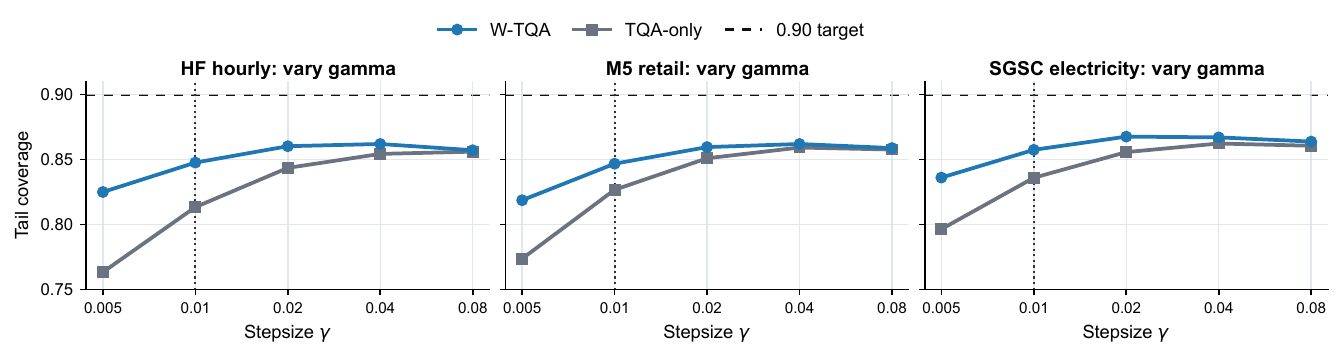}
  \caption{Parameter-sensitivity slices under sparse target feedback ($p=0.2$).
  The design is the same as in Figure~\ref{fig:param_full}, with a
  shared round-level MCAR feedback schedule in each replication.}\label{fig:param_sparse}
\end{figure}

\FloatBarrier
\subsubsection{Base predictor}\label{app:base-predictor-sensitivity}

We check whether the comparison persists when the ridge predictor is replaced.
Under full feedback, we repeat it with HistGradientBoosting (HGB) and LightGBM on
HF hourly, M5 and SGSC. On Synthetic Hard, we use two oracle predictors
built from the data-generating process: the correct neural-factor representation with a fitted output
layer, and the conditional mean given the observed inputs. All four conformal
methods share a frozen predictor within each run; learned components use only
calibration-unit burn-in data. We use 30 matched unit splits on HF hourly, M5 and SGSC
and 30 independent synthetic draws, with $\alpha=0.1$, $h=0.6$, and
$\gamma=0.01$.
Hyperparameters are fixed across splits: HGB uses
\texttt{max\_iter=100}, \texttt{max\_depth=6}, \texttt{learning\_rate=0.05}
and \texttt{min\_samples\_leaf=20} without early stopping, and LightGBM uses
\texttt{n\_estimators=200}, \texttt{num\_leaves=31},
\texttt{learning\_rate=0.05}, \texttt{min\_child\_samples=100},
\texttt{subsample=0.8}, \texttt{colsample\_bytree=0.8} and
\texttt{reg\_lambda=1}; apart from a fixed random seed, all other arguments
keep the library defaults.

Table~\ref{tab:base-predictor-sensitivity} reports prediction error and all four
methods' coverage and width diagnostics. W-TQA retains the highest mean tail
coverage in all eight settings. It exceeds both one-branch ablations in 179 of
the 180 real split--predictor comparisons, TQA-only being higher in one M5
split with HGB, and is best or tied best in all 60 synthetic
oracle runs. The ranking therefore holds whether the replacement predictor is
slightly less accurate than ridge, as the tree models are on HF hourly and M5,
or far more accurate, as the oracle predictors are on Synthetic Hard.

\begin{table}[H]
\centering
\footnotesize
\setlength{\tabcolsep}{1.4pt}
\renewcommand{\arraystretch}{0.95}
\caption{Base-predictor sensitivity under full feedback. Hard denotes Synthetic
Hard. Entries are mean
$\pm$ sd over 30 matched splits or independent synthetic draws; RMSE reductions
are averaged within-run reductions relative to ridge, negative for an
increase. Coverage SD is the
cross-unit standard deviation of coverage; tail coverage and tail width use
the lowest-covered 10\% of target units. Width CoV is the standard deviation
of all unit--time widths divided by their mean. Bold marks the highest mean
tail coverage in each setting. Real-panel error and width are on the return
(HF hourly), $\log(1+\mathrm{sales})$ (M5) or $\log(1+\mathrm{kWh})$ (SGSC) scale.
Real-panel sd measures unit-split sensitivity within one temporal panel.}
\label{tab:base-predictor-sensitivity}
\begin{tabular*}{\linewidth}{@{\extracolsep{\fill}}lrrrrr@{}}
\toprule
\multicolumn{6}{@{}l}{\textbf{(a) Prediction error}} \\
Setting / predictor & RMSE & MAE & Ridge RMSE & Ridge MAE & RMSE reduction \\
\midrule
HF hourly / HGB & $0.0085 \pm 0.0006$ & $0.0044 \pm 0.0002$ & $0.0084 \pm 0.0006$ & $0.0044 \pm 0.0002$ & $-1.24\%$ \\
HF hourly / LightGBM & $0.0086 \pm 0.0006$ & $0.0045 \pm 0.0002$ & $0.0084 \pm 0.0006$ & $0.0044 \pm 0.0002$ & $-2.13\%$ \\
M5 / HGB & $0.740 \pm 0.022$ & $0.590 \pm 0.014$ & $0.732 \pm 0.019$ & $0.587 \pm 0.013$ & $-1.00\%$ \\
M5 / LightGBM & $0.738 \pm 0.020$ & $0.588 \pm 0.014$ & $0.732 \pm 0.019$ & $0.587 \pm 0.013$ & $-0.73\%$ \\
SGSC / HGB & $0.162 \pm 0.007$ & $0.100 \pm 0.005$ & $0.163 \pm 0.007$ & $0.100 \pm 0.005$ & $0.44\%$ \\
SGSC / LightGBM & $0.162 \pm 0.007$ & $0.100 \pm 0.005$ & $0.163 \pm 0.007$ & $0.100 \pm 0.005$ & $0.46\%$ \\
Hard / neural-factor oracle & $0.818 \pm 0.082$ & $0.650 \pm 0.064$ & $1.190 \pm 0.152$ & $0.939 \pm 0.117$ & $30.68\%$ \\
Hard / conditional-mean oracle & $0.754 \pm 0.060$ & $0.599 \pm 0.047$ & $1.190 \pm 0.152$ & $0.939 \pm 0.117$ & $35.87\%$ \\
\bottomrule
\end{tabular*}

\vspace{0.5em}
\begin{tabular*}{\linewidth}{@{\extracolsep{\fill}}lrrrrrr@{}}
\toprule
\multicolumn{7}{@{}l}{\textbf{(b) Conformal calibration}} \\
Method & Avg. coverage & Tail coverage & Coverage SD & Avg. width & Tail width & Width CoV \\
\midrule
\multicolumn{7}{@{}l}{\textit{HF hourly / HGB}} \\
Split CP & $0.904 \pm 0.015$ & $0.632 \pm 0.039$ & $0.116 \pm 0.013$ & $0.01943 \pm 0.00032$ & $0.01943 \pm 0.00032$ & $0.593 \pm 0.007$ \\
W-only & $0.927 \pm 0.009$ & $0.770 \pm 0.024$ & $0.068 \pm 0.008$ & $0.02156 \pm 0.00107$ & $0.02376 \pm 0.00181$ & $1.016 \pm 0.180$ \\
TQA-only & $0.912 \pm 0.002$ & $0.884 \pm 0.001$ & $0.018 \pm 0.001$ & $0.01985 \pm 0.00107$ & $0.04349 \pm 0.00530$ & $1.122 \pm 0.105$ \\
W-TQA & $0.911 \pm 0.002$ & $\mathbf{0.890 \pm 0.002}$ & $0.014 \pm 0.001$ & $0.01995 \pm 0.00119$ & $0.03637 \pm 0.00586$ & $1.176 \pm 0.111$ \\
\midrule
\multicolumn{7}{@{}l}{\textit{HF hourly / LightGBM}} \\
Split CP & $0.903 \pm 0.014$ & $0.634 \pm 0.040$ & $0.115 \pm 0.013$ & $0.01969 \pm 0.00033$ & $0.01969 \pm 0.00033$ & $0.595 \pm 0.007$ \\
W-only & $0.927 \pm 0.009$ & $0.772 \pm 0.023$ & $0.067 \pm 0.008$ & $0.02183 \pm 0.00107$ & $0.02398 \pm 0.00177$ & $1.006 \pm 0.173$ \\
TQA-only & $0.912 \pm 0.002$ & $0.884 \pm 0.001$ & $0.018 \pm 0.001$ & $0.02012 \pm 0.00107$ & $0.04352 \pm 0.00529$ & $1.100 \pm 0.100$ \\
W-TQA & $0.911 \pm 0.002$ & $\mathbf{0.890 \pm 0.002}$ & $0.014 \pm 0.001$ & $0.02026 \pm 0.00119$ & $0.03611 \pm 0.00520$ & $1.158 \pm 0.108$ \\
\midrule
\multicolumn{7}{@{}l}{\textit{M5 / HGB}} \\
Split CP & $0.901 \pm 0.012$ & $0.557 \pm 0.038$ & $0.139 \pm 0.013$ & $2.381 \pm 0.038$ & $2.381 \pm 0.038$ & $0.087 \pm 0.004$ \\
W-only & $0.912 \pm 0.007$ & $0.691 \pm 0.028$ & $0.092 \pm 0.009$ & $2.375 \pm 0.045$ & $2.657 \pm 0.146$ & $0.346 \pm 0.027$ \\
TQA-only & $0.916 \pm 0.002$ & $0.882 \pm 0.009$ & $0.023 \pm 0.004$ & $2.143 \pm 0.031$ & $3.307 \pm 0.184$ & $0.346 \pm 0.023$ \\
W-TQA & $0.913 \pm 0.002$ & $\mathbf{0.884 \pm 0.008}$ & $0.020 \pm 0.004$ & $2.132 \pm 0.044$ & $2.774 \pm 0.223$ & $0.406 \pm 0.024$ \\
\midrule
\multicolumn{7}{@{}l}{\textit{M5 / LightGBM}} \\
Split CP & $0.900 \pm 0.012$ & $0.562 \pm 0.034$ & $0.137 \pm 0.012$ & $2.369 \pm 0.034$ & $2.369 \pm 0.034$ & $0.094 \pm 0.005$ \\
W-only & $0.911 \pm 0.008$ & $0.693 \pm 0.029$ & $0.091 \pm 0.009$ & $2.378 \pm 0.048$ & $2.615 \pm 0.110$ & $0.347 \pm 0.026$ \\
TQA-only & $0.915 \pm 0.002$ & $0.884 \pm 0.005$ & $0.020 \pm 0.003$ & $2.151 \pm 0.035$ & $3.273 \pm 0.165$ & $0.341 \pm 0.021$ \\
W-TQA & $0.912 \pm 0.001$ & $\mathbf{0.887 \pm 0.005}$ & $0.018 \pm 0.003$ & $2.146 \pm 0.049$ & $2.755 \pm 0.194$ & $0.400 \pm 0.022$ \\
\midrule
\multicolumn{7}{@{}l}{\textit{SGSC / HGB}} \\
Split CP & $0.900 \pm 0.014$ & $0.640 \pm 0.048$ & $0.112 \pm 0.014$ & $0.466 \pm 0.011$ & $0.466 \pm 0.011$ & $0.350 \pm 0.012$ \\
W-only & $0.932 \pm 0.010$ & $0.759 \pm 0.039$ & $0.076 \pm 0.011$ & $0.610 \pm 0.024$ & $0.560 \pm 0.038$ & $0.570 \pm 0.035$ \\
TQA-only & $0.915 \pm 0.003$ & $0.888 \pm 0.001$ & $0.024 \pm 0.002$ & $0.450 \pm 0.018$ & $0.876 \pm 0.063$ & $0.671 \pm 0.026$ \\
W-TQA & $0.917 \pm 0.003$ & $\mathbf{0.892 \pm 0.001}$ & $0.022 \pm 0.002$ & $0.483 \pm 0.022$ & $0.822 \pm 0.058$ & $0.741 \pm 0.032$ \\
\midrule
\multicolumn{7}{@{}l}{\textit{SGSC / LightGBM}} \\
Split CP & $0.899 \pm 0.014$ & $0.634 \pm 0.049$ & $0.113 \pm 0.014$ & $0.461 \pm 0.011$ & $0.461 \pm 0.011$ & $0.357 \pm 0.012$ \\
W-only & $0.932 \pm 0.010$ & $0.758 \pm 0.038$ & $0.076 \pm 0.011$ & $0.605 \pm 0.024$ & $0.555 \pm 0.033$ & $0.576 \pm 0.035$ \\
TQA-only & $0.915 \pm 0.003$ & $0.888 \pm 0.001$ & $0.024 \pm 0.002$ & $0.448 \pm 0.018$ & $0.873 \pm 0.060$ & $0.675 \pm 0.026$ \\
W-TQA & $0.917 \pm 0.003$ & $\mathbf{0.892 \pm 0.001}$ & $0.022 \pm 0.002$ & $0.481 \pm 0.022$ & $0.820 \pm 0.056$ & $0.744 \pm 0.032$ \\
\midrule
\multicolumn{7}{@{}l}{\textit{Synthetic Hard / neural-factor oracle}} \\
Split CP & $0.725 \pm 0.045$ & $0.336 \pm 0.112$ & $0.184 \pm 0.036$ & $1.766 \pm 0.041$ & $1.766 \pm 0.041$ & $0.055 \pm 0.018$ \\
W-only & $0.926 \pm 0.030$ & $0.706 \pm 0.133$ & $0.093 \pm 0.042$ & $3.033 \pm 0.283$ & $3.031 \pm 0.284$ & $0.203 \pm 0.020$ \\
TQA-only & $0.817 \pm 0.020$ & $0.688 \pm 0.066$ & $0.072 \pm 0.017$ & $2.117 \pm 0.123$ & $2.996 \pm 0.279$ & $0.281 \pm 0.046$ \\
W-TQA & $0.927 \pm 0.020$ & $\mathbf{0.788 \pm 0.079}$ & $0.064 \pm 0.023$ & $2.889 \pm 0.226$ & $3.404 \pm 0.278$ & $0.223 \pm 0.019$ \\
\midrule
\multicolumn{7}{@{}l}{\textit{Synthetic Hard / conditional-mean oracle}} \\
Split CP & $0.745 \pm 0.036$ & $0.339 \pm 0.101$ & $0.189 \pm 0.032$ & $1.705 \pm 0.017$ & $1.705 \pm 0.017$ & $0.042 \pm 0.004$ \\
W-only & $0.932 \pm 0.027$ & $0.706 \pm 0.129$ & $0.092 \pm 0.041$ & $2.865 \pm 0.222$ & $2.861 \pm 0.218$ & $0.196 \pm 0.018$ \\
TQA-only & $0.829 \pm 0.016$ & $0.692 \pm 0.066$ & $0.076 \pm 0.017$ & $2.002 \pm 0.073$ & $2.863 \pm 0.226$ & $0.273 \pm 0.037$ \\
W-TQA & $0.931 \pm 0.019$ & $\mathbf{0.793 \pm 0.080}$ & $0.062 \pm 0.025$ & $2.709 \pm 0.163$ & $3.196 \pm 0.236$ & $0.217 \pm 0.020$ \\
\bottomrule
\end{tabular*}
\end{table}

\subsection{Fallback and level accounting}\label{app:clip-accounting}

This subsection quantifies, for the experiments of
Section~\ref{sec:experiments}, how often the fallback of
Algorithm~\ref{alg:wtqa} is active, whether any prediction set is empty, and how
often the event of Assumption~\ref{ass:feasibility} occurs.
Table~\ref{tab:fallback-accounting-appendix} gives the fallback accounting,
Table~\ref{tab:fallback-decomposition} splits the unit-rounds at $p=1$ into
fallback and regular rounds, and Table~\ref{tab:feasibility-appendix} gives the
frequency of the event in~\eqref{eq:fallback-feasibility}. All entries are means over the same
$30$ random splits (real panels) or replications (synthetic panels) as
Table~\ref{tab:main-results}.

\paragraph{Fallback.}
W-TQA falls back on at most $3.8\%$ of unit-rounds at every $p$
(Table~\ref{tab:fallback-accounting-appendix}). On these rounds the threshold is
the largest peer score, and on HF hourly, M5 and SGSC W-TQA covers more often
there than on regular rounds (Tables~\ref{tab:fallback-accounting-appendix}
and~\ref{tab:fallback-decomposition}).

\paragraph{Level.}
By~\eqref{eq:level-recursion}, each revealed round raises the level by at
most $\gamma\alpha$, so
$\alpha_t\le\alpha+\gamma\alpha(t-1)<1$ for $t\le900$, and no set is empty in
these experiments.

\paragraph{Feasibility.}
The conditional probability in~\eqref{eq:fallback-feasibility} cannot be
observed, but Assumption~\ref{ass:feasibility} has a consequence that can be
checked. Taking expectations in~\eqref{eq:fallback-feasibility} gives
$\mathbb P(\hat s_{N+1,t}>\max_{k\le N}\hat s_{k,t})\le\alpha-\zeta<\alpha$ for
every $t$, so if Assumption~\ref{ass:feasibility} holds for every test unit
taken as the target, the expected frequency of this event over unit-rounds is
below $\alpha=10\%$. The event involves neither the weights nor the level, so
its frequency is the same for every method and feedback probability.
Table~\ref{tab:feasibility-appendix} reports it: the mean frequency is at most
$1.81\%$ on every panel, far below $\alpha$. The experimental data therefore pass this check of
Assumption~\ref{ass:feasibility}.

\begin{table*}[htbp]
  \centering
  \footnotesize
  \captionsetup{skip=5pt}
  \caption{Fallback accounting at $\gamma=0.01$, mean~$\pm$~sd over the $30$
  splits (real panels) or replications (synthetic panels, $p=1$ only). The
  fallback frequency is the percentage of unit-rounds on which $E_t$ occurs,
  that is $\alpha_t<w^{(t)}_{N+1,N+1}$ for W-TQA and $\alpha_t<1/(N+1)$ for
  TQA-only; W-only has the W-TQA frequency of the $p=0$ row at every $p$, and
  Split~CP, TQA-B and LPCI never fall back. Entries on fallback rounds average over
  the splits with at least one fallback round; a single such split gives no
  sd.}\label{tab:fallback-accounting-appendix}
  \resizebox{\ifdim\width>\linewidth\linewidth\else\width\fi}{!}{%
  \begin{tabular}{@{}llccc@{}}
  \toprule
  Panel, $p$ & Method & \begin{tabular}[c]{@{}c@{}}fallback\\freq.\ (\%)\end{tabular} & \begin{tabular}[c]{@{}c@{}}miscov.\ on\\fallback rounds\end{tabular} & \begin{tabular}[c]{@{}c@{}}miscov.\ on\\regular rounds\end{tabular} \\
  \midrule
  HF hourly, $0$ & TQA-only & $0.00 \pm 0.00$ & -- & $0.096 \pm 0.015$ \\
  & W-TQA & $1.70 \pm 1.29$ & $0.037 \pm 0.019$ & $0.073 \pm 0.009$ \\
  HF hourly, $0.2$ & TQA-only & $0.56 \pm 0.48$ & $0.053 \pm 0.024$ & $0.078 \pm 0.007$ \\
  & W-TQA & $1.19 \pm 0.79$ & $0.044 \pm 0.021$ & $0.074 \pm 0.004$ \\
  HF hourly, $0.4$ & TQA-only & $0.86 \pm 0.61$ & $0.053 \pm 0.019$ & $0.079 \pm 0.004$ \\
  & W-TQA & $1.13 \pm 0.84$ & $0.048 \pm 0.024$ & $0.079 \pm 0.003$ \\
  HF hourly, $0.6$ & TQA-only & $0.95 \pm 0.69$ & $0.058 \pm 0.021$ & $0.083 \pm 0.003$ \\
  & W-TQA & $1.16 \pm 0.89$ & $0.051 \pm 0.024$ & $0.084 \pm 0.002$ \\
  HF hourly, $0.8$ & TQA-only & $1.03 \pm 0.72$ & $0.061 \pm 0.018$ & $0.086 \pm 0.003$ \\
  & W-TQA & $1.16 \pm 0.82$ & $0.053 \pm 0.023$ & $0.087 \pm 0.002$ \\
  HF hourly, $1$ & TQA-only & $1.08 \pm 0.74$ & $0.059 \pm 0.018$ & $0.088 \pm 0.002$ \\
  & W-TQA & $1.15 \pm 0.83$ & $0.053 \pm 0.023$ & $0.089 \pm 0.002$ \\
  \midrule
  M5, $0$ & TQA-only & $0.00 \pm 0.00$ & -- & $0.096 \pm 0.011$ \\
  & W-TQA & $3.38 \pm 1.13$ & $0.032 \pm 0.023$ & $0.088 \pm 0.008$ \\
  M5, $0.2$ & TQA-only & $0.70 \pm 0.41$ & $0.101 \pm 0.043$ & $0.069 \pm 0.004$ \\
  & W-TQA & $2.90 \pm 1.00$ & $0.045 \pm 0.024$ & $0.075 \pm 0.003$ \\
  M5, $0.4$ & TQA-only & $1.06 \pm 0.51$ & $0.102 \pm 0.041$ & $0.072 \pm 0.003$ \\
  & W-TQA & $2.68 \pm 0.95$ & $0.051 \pm 0.024$ & $0.079 \pm 0.002$ \\
  M5, $0.6$ & TQA-only & $1.13 \pm 0.57$ & $0.106 \pm 0.039$ & $0.077 \pm 0.002$ \\
  & W-TQA & $2.43 \pm 0.88$ & $0.057 \pm 0.026$ & $0.082 \pm 0.002$ \\
  M5, $0.8$ & TQA-only & $1.20 \pm 0.59$ & $0.105 \pm 0.039$ & $0.080 \pm 0.002$ \\
  & W-TQA & $2.35 \pm 0.86$ & $0.060 \pm 0.026$ & $0.085 \pm 0.002$ \\
  M5, $1$ & TQA-only & $1.24 \pm 0.61$ & $0.107 \pm 0.038$ & $0.083 \pm 0.002$ \\
  & W-TQA & $2.26 \pm 0.84$ & $0.063 \pm 0.026$ & $0.087 \pm 0.001$ \\
  \midrule
  SGSC, $0$ & TQA-only & $0.00 \pm 0.00$ & -- & $0.098 \pm 0.014$ \\
  & W-TQA & $3.80 \pm 1.54$ & $0.001 \pm 0.001$ & $0.069 \pm 0.010$ \\
  SGSC, $0.2$ & TQA-only & $0.21 \pm 0.18$ & $0.047 \pm 0.039$ & $0.076 \pm 0.006$ \\
  & W-TQA & $2.93 \pm 1.36$ & $0.013 \pm 0.016$ & $0.067 \pm 0.005$ \\
  SGSC, $0.4$ & TQA-only & $0.45 \pm 0.28$ & $0.059 \pm 0.019$ & $0.077 \pm 0.005$ \\
  & W-TQA & $3.00 \pm 1.38$ & $0.018 \pm 0.015$ & $0.073 \pm 0.005$ \\
  SGSC, $0.6$ & TQA-only & $0.63 \pm 0.36$ & $0.062 \pm 0.021$ & $0.080 \pm 0.004$ \\
  & W-TQA & $2.81 \pm 1.24$ & $0.024 \pm 0.019$ & $0.077 \pm 0.003$ \\
  SGSC, $0.8$ & TQA-only & $0.78 \pm 0.45$ & $0.071 \pm 0.013$ & $0.082 \pm 0.003$ \\
  & W-TQA & $2.61 \pm 1.16$ & $0.028 \pm 0.019$ & $0.081 \pm 0.003$ \\
  SGSC, $1$ & TQA-only & $0.85 \pm 0.48$ & $0.074 \pm 0.012$ & $0.085 \pm 0.003$ \\
  & W-TQA & $2.39 \pm 1.02$ & $0.032 \pm 0.019$ & $0.084 \pm 0.003$ \\
  \midrule
  Synthetic Easy, $1$ & TQA-only & $0.00 \pm 0.00$ & -- & $0.118 \pm 0.008$ \\
   & W-TQA & $0.02 \pm 0.10$ & $0.000$ & $0.072 \pm 0.008$ \\
  \midrule
  Synthetic Medium, $1$ & TQA-only & $0.03 \pm 0.11$ & $0.000 \pm 0.000$ & $0.136 \pm 0.021$ \\
   & W-TQA & $0.05 \pm 0.16$ & $0.026 \pm 0.044$ & $0.071 \pm 0.008$ \\
  \midrule
  Synthetic Hard, $1$ & TQA-only & $1.90 \pm 1.80$ & $0.103 \pm 0.087$ & $0.186 \pm 0.022$ \\
   & W-TQA & $0.70 \pm 1.59$ & $0.194 \pm 0.156$ & $0.077 \pm 0.014$ \\
  \bottomrule
  \end{tabular}%
  }
\end{table*}

\begin{table*}[htbp]
  \centering
  \footnotesize
  \captionsetup{skip=5pt}
  \setlength{\tabcolsep}{3pt}
  \caption{Unit-rounds split by the fallback indicator $E_t$ of W-TQA,
  mean~$\pm$~sd over $30$ splits at $p=1$: coverage and average width on all,
  regular and fallback unit-rounds.}\label{tab:fallback-decomposition}
  \resizebox{\ifdim\width>\linewidth\linewidth\else\width\fi}{!}{%
  \begin{tabular}{@{}llcccccc@{}}
  \toprule
  & & \multicolumn{3}{c}{Coverage} & \multicolumn{3}{c}{Average width} \\
  \cmidrule(lr){3-5}\cmidrule(lr){6-8}
  Panel & Method & all & regular & fallback & all & regular & fallback \\
  \midrule
  HF hourly & W-TQA & $0.911 \pm 0.002$ & $0.911 \pm 0.002$ & $0.947 \pm 0.023$ & $0.0198 \pm 0.0012$ & $0.0189 \pm 0.0008$ & $0.0999 \pm 0.0064$ \\
   & TQA-only & $0.912 \pm 0.002$ & $0.912 \pm 0.002$ & $0.897 \pm 0.026$ & $0.0197 \pm 0.0011$ & $0.0189 \pm 0.0007$ & $0.0796 \pm 0.0100$ \\
  \midrule
  M5 & W-TQA & $0.913 \pm 0.001$ & $0.913 \pm 0.001$ & $0.937 \pm 0.026$ & $2.1198 \pm 0.0442$ & $2.0458 \pm 0.0299$ & $5.3495 \pm 0.0943$ \\
   & TQA-only & $0.917 \pm 0.002$ & $0.918 \pm 0.002$ & $0.875 \pm 0.016$ & $2.1346 \pm 0.0323$ & $2.0845 \pm 0.0249$ & $4.2623 \pm 0.2722$ \\
  \midrule
  SGSC & W-TQA & $0.917 \pm 0.003$ & $0.916 \pm 0.003$ & $0.968 \pm 0.019$ & $0.4847 \pm 0.0225$ & $0.4548 \pm 0.0186$ & $1.7080 \pm 0.0350$ \\
   & TQA-only & $0.915 \pm 0.003$ & $0.915 \pm 0.003$ & $0.922 \pm 0.014$ & $0.4533 \pm 0.0184$ & $0.4435 \pm 0.0166$ & $0.9089 \pm 0.2801$ \\
  \bottomrule
  \end{tabular}%
  }
\end{table*}

\begin{table}[htbp]
  \centering
  \footnotesize
  \captionsetup{skip=5pt}
  \caption{Frequency of the event $\hat s_{N+1,t}>\max_{k\le N}\hat s_{k,t}$
  in~\eqref{eq:fallback-feasibility}: percentage of unit-rounds, mean~$\pm$~sd
  over the $30$ splits or replications. The event does not depend on the
  method or on $p$.}\label{tab:feasibility-appendix}
  \resizebox{\ifdim\width>\linewidth\linewidth\else\width\fi}{!}{%
  \begin{tabular}{@{}lcccccc@{}}
  \toprule
  Panel & HF hourly & M5 & SGSC & Synthetic Easy & Synthetic Medium & Synthetic Hard \\
  \midrule
  \begin{tabular}[c]{@{}l@{}}$\hat s_{N+1,t}>\max_k\hat s_{k,t}$\\(\% unit-rounds)\end{tabular} & $0.25 \pm 0.10$ & $0.20 \pm 0.11$ & $0.36 \pm 0.11$ & $0.43 \pm 0.18$ & $0.63 \pm 0.27$ & $1.81 \pm 0.86$ \\
  \bottomrule
  \end{tabular}%
  }
\end{table}

\FloatBarrier

\subsection{Profile distance and future residual-law mismatch}\label{app:profile-distance-diagnostic}

The TV-profile condition of Assumption~\ref{ass:spatial_profiles} cannot be
tested (Proposition~\ref{prop:spatial_impossibility}), but its premise, that
units with closer profiles have more similar score laws, can be probed. We do
so with the profile metric $\rho_t$ of Section~\ref{sec:proposed_method} on HF
hourly, M5 and SGSC. In each of $30$ unit splits, the point predictor is fitted
on calibration-unit burn-in rows. Profiles use only rows preceding each
evaluation block and are then frozen; absolute-residual laws are estimated on
the disjoint future block. The block length is $100$, with four evaluated
blocks on HF hourly and five on M5 and SGSC. Three proxies measure the
mismatch between a target's and a peer's residual law: the total variation
distance between histograms on pooled calibration-residual quantile bins fixed
from burn-in (histogram-TV), the empirical two-sample Kolmogorov--Smirnov
statistic (KS), and the empirical $1$-Wasserstein distance divided by one
scale per split, the interquartile range of the pooled burn-in absolute
residuals of the calibration units, or their standard deviation if that range
is zero (standardized $W_1$). The weighted
comparison uses Gaussian weights with $h=0.6$, normalized over the peers so
that the sentinel receives no weight, against uniform peer weights. Blockwise
statistics are averaged within each split before a split-level bootstrap with
$5{,}000$ draws.

\begin{table}[htbp]
\centering
\footnotesize
\setlength{\tabcolsep}{3pt}
\renewcommand{\arraystretch}{1.08}
\caption{Lagged profile distance versus future marginal absolute-residual-law
mismatch. Entries are split means with 95\% bootstrap intervals in brackets.
Nearest and farthest peer quintiles are selected separately for each target.
The final column excludes the sentinel weight; negative values favor feature
weighting.}
\label{tab:profile-distance-diagnostic}
\begin{tabular}{@{}llccc@{}}
\toprule
Panel & Mismatch proxy &
\begin{tabular}[c]{@{}c@{}}Spearman\\distance vs.\ mismatch\end{tabular} &
\begin{tabular}[c]{@{}c@{}}Farthest minus nearest\\peer quintile\end{tabular} &
\begin{tabular}[c]{@{}c@{}}Feature minus uniform\\peer-normalized\end{tabular} \\
\midrule
HF hourly & Histogram-TV & \begin{tabular}[c]{@{}c@{}}\(0.3379\)\\\([0.3325, 0.3436]\)\end{tabular} & \begin{tabular}[c]{@{}c@{}}\(0.1127\)\\\([0.1114, 0.1140]\)\end{tabular} & \begin{tabular}[c]{@{}c@{}}\(-0.0328\)\\\([-0.0335, -0.0321]\)\end{tabular} \\
HF hourly & KS & \begin{tabular}[c]{@{}c@{}}\(0.3523\)\\\([0.3468, 0.3580]\)\end{tabular} & \begin{tabular}[c]{@{}c@{}}\(0.1326\)\\\([0.1312, 0.1340]\)\end{tabular} & \begin{tabular}[c]{@{}c@{}}\(-0.0390\)\\\([-0.0398, -0.0383]\)\end{tabular} \\
HF hourly & Standardized \(W_1\) & \begin{tabular}[c]{@{}c@{}}\(0.2906\)\\\([0.2834, 0.2982]\)\end{tabular} & \begin{tabular}[c]{@{}c@{}}\(0.4106\)\\\([0.4040, 0.4171]\)\end{tabular} & \begin{tabular}[c]{@{}c@{}}\(-0.1147\)\\\([-0.1167, -0.1127]\)\end{tabular} \\
\addlinespace[2pt]
M5 & Histogram-TV & \begin{tabular}[c]{@{}c@{}}\(0.2145\)\\\([0.2106, 0.2183]\)\end{tabular} & \begin{tabular}[c]{@{}c@{}}\(0.1291\)\\\([0.1271, 0.1312]\)\end{tabular} & \begin{tabular}[c]{@{}c@{}}\(-0.0473\)\\\([-0.0482, -0.0463]\)\end{tabular} \\
M5 & KS & \begin{tabular}[c]{@{}c@{}}\(0.2427\)\\\([0.2395, 0.2459]\)\end{tabular} & \begin{tabular}[c]{@{}c@{}}\(0.1151\)\\\([0.1136, 0.1167]\)\end{tabular} & \begin{tabular}[c]{@{}c@{}}\(-0.0374\)\\\([-0.0381, -0.0368]\)\end{tabular} \\
M5 & Standardized \(W_1\) & \begin{tabular}[c]{@{}c@{}}\(0.4102\)\\\([0.4072, 0.4133]\)\end{tabular} & \begin{tabular}[c]{@{}c@{}}\(0.3967\)\\\([0.3887, 0.4042]\)\end{tabular} & \begin{tabular}[c]{@{}c@{}}\(-0.1153\)\\\([-0.1168, -0.1137]\)\end{tabular} \\
\addlinespace[2pt]
SGSC & Histogram-TV & \begin{tabular}[c]{@{}c@{}}\(0.0960\)\\\([0.0939, 0.0981]\)\end{tabular} & \begin{tabular}[c]{@{}c@{}}\(0.0712\)\\\([0.0699, 0.0724]\)\end{tabular} & \begin{tabular}[c]{@{}c@{}}\(-0.0219\)\\\([-0.0223, -0.0215]\)\end{tabular} \\
SGSC & KS & \begin{tabular}[c]{@{}c@{}}\(0.1028\)\\\([0.1006, 0.1051]\)\end{tabular} & \begin{tabular}[c]{@{}c@{}}\(0.0801\)\\\([0.0786, 0.0815]\)\end{tabular} & \begin{tabular}[c]{@{}c@{}}\(-0.0245\)\\\([-0.0249, -0.0241]\)\end{tabular} \\
SGSC & Standardized \(W_1\) & \begin{tabular}[c]{@{}c@{}}\(0.1181\)\\\([0.1152, 0.1212]\)\end{tabular} & \begin{tabular}[c]{@{}c@{}}\(0.2234\)\\\([0.2184, 0.2285]\)\end{tabular} & \begin{tabular}[c]{@{}c@{}}\(-0.0747\)\\\([-0.0761, -0.0734]\)\end{tabular} \\
\bottomrule
\end{tabular}
\end{table}

Across HF, M5 and SGSC and the three proxies, the Spearman correlations
and the farthest-minus-nearest quintile gaps are positive and the
feature-minus-uniform differences are negative, and every bootstrap interval
excludes zero (Table~\ref{tab:profile-distance-diagnostic}). These results are
consistent with the premise of the TV-profile condition of
Assumption~\ref{ass:spatial_profiles}: peers with closer lagged profiles have
more similar future residual laws, and the deployed weights reduce the average
mismatch relative to uniform weights.

\FloatBarrier

\subsection{Compute resources}\label{app:compute-resources}

All experiments were run on CPU workers only; no GPU was used. A complete
real-data sweep over the $30$ splits, the feedback probability grid and the two
non-MCAR mechanisms takes $2.4$, $5.4$ and $3.5$ core-hours on HF hourly, M5
and SGSC, run as independent single-core jobs. The heaviest component is the random-forest LPCI baseline of the
synthetic study, which takes about $6$ hours. The cross-market panel takes about $25$ seconds
per replication and window on one CPU core.

\end{document}